\documentclass[journal,onecolumn]{IEEEtran}

\usepackage{amsmath,amsfonts,amsthm}
\usepackage{algorithm}
\usepackage{array}
\usepackage[caption=false,font=normalsize,labelfont=sf,textfont=sf]{subfig}
\usepackage{textcomp}
\usepackage{stfloats}
\usepackage{url}
\usepackage{verbatim}
\usepackage{graphicx}
\usepackage{natbib}
\usepackage{enumerate}

\usepackage{hyperref}       
\usepackage{booktabs}       
\usepackage{nicefrac}       
\usepackage{microtype}      
\usepackage{xcolor}         
\usepackage{xspace}
\usepackage{wrapfig}
\usepackage{mathrsfs}
\usepackage{subeqnarray}
\usepackage{subfloat}
\usepackage{multirow}
\usepackage{mathtools}
\usepackage{cases}
\usepackage{todonotes}
\usepackage{algorithmicx}
\usepackage{enumitem}
\usepackage[noend]{algpseudocode}
\usepackage{etoc}

\newcommand{\sectionname}{Sec.}

\usepackage{amsmath,amsfonts,bm}

\def\eqref#1{Eq.~(\ref{#1})}

\def\1{\bm{1}}

\DeclareMathAlphabet{\mathsfit}{\encodingdefault}{\sfdefault}{m}{sl}
\SetMathAlphabet{\mathsfit}{bold}{\encodingdefault}{\sfdefault}{bx}{n}

\def\gG{{\mathcal{G}}}

\newcommand{\E}{\mathbb{E}}

\theoremstyle{remark}
\newtheorem{theorem}{Theorem}[section]
\newtheorem{lemma}{Lemma}[theorem]
\newtheorem{definition}[theorem]{Definition}
\newtheorem{corollary}{Corollary}[theorem]
\newtheorem{proposition}{Proposition}[theorem]
\newtheorem{remark}{Remark}[theorem]
\newtheorem{assumption}{Assumption}[theorem]

\newcommand\zx[1]{{\color{blue}{zx: #1}}}

\begin{document}

\title{The Mathematical Theory of Double Descent as Shrinkage not Overfitting: Matching Bounds for Path-Regularized Multilayer Neural Networks}

\author{Hao~Yu
\thanks{Independent Researcher. E-mail: dustless2014@163.com}
}
\markboth{}%
{Yu: Path Regularization and the Double Descent Phenomenon}

\maketitle

\begin{abstract}
We give a nonasymptotic generalization theory for multilayer fully connected neural networks
trained with path regularization, for general Lipschitz losses and without assuming the loss
bounded, the architecture sparse, or any width or depth limit. Approximation error is modelled
explicitly, and along the way we settle an open problem of E et al.\ on approximation rates in
generalized Barron spaces. We treat the underparameterized and overparameterized regimes
separately, which is what allows the resulting curve to be non-monotone in the width. We then
prove lower bounds on the error itself, valid at every target and at every local minimizer, and
combine them with the upper bounds into a theorem asserting that the error -- not merely our
bound on it -- descends, rises and descends again as the width grows, with a peak exceeding the
final plateau by a factor diverging in the sample size. Having both sides of the bound also
identifies what drives the ascent, and it is not overfitting: once the regularization exceeds
the dual norm of the noise, which the prescribed schedule does at every width, a slow-rate
oracle inequality controls the error with no variance term at all. The rise is instead the
shrinkage demanded by the regularization that the underparameterized regime requires, and we
show it cannot be removed by any width-independent regularization under which the theory's
guarantees hold. This mechanism is a property of the estimator analyzed here; we do not claim
it accounts for the double descent reported for unregularized training.
\end{abstract}

\begin{IEEEkeywords}
Path regularization, Generalization error, Multilayer feed-forward neural networks, Lipschitz activation, Lipschitz loss function, Double descent
\end{IEEEkeywords}

\section{Introduction}\label{sec0}

\subsection{Neural Networks and Their Generalization Theory}
The advent of neural networks (NNs) has greatly enhanced and revolutionized artificial intelligence. Although research on neural networks can be traced back to the early 1980s with Hinton et al.~\cite{rumelhart1986learning,hinton1986learning,hinton1981implementing,rumelhart1986general}, their power only began to emerge in recent decades. NN models have been immersed in many areas of AI, including computer vision, natural language processing, speech, reinforcement learning, etc., and have witnessed thrilling breakthroughs in AI development. Although realistic NNs work very well in practice, they still defy rigorous mathematical understanding.

Recently, there has been a surge of research on the theoretical foundations of NNs. One line of research concerns the training dynamics of neural networks, especially for stochastic gradient descent (SGD) algorithms. For example,~\citet{jin2017escape} shows that noisy SGD can escape from saddle points.~\citet{fang2019sharp} explains that SGD can also escape from saddle points with high probability. Many works~\citet{zhou2019sgd,mertikopoulos2020almost,du2018gradient,arora2019fine,du2019gradient,rotskoff2022trainability,bao2023global} show that SGD can converge to global minima under certain conditions.


This paper focuses on another prominent problem: understanding the generalization property of a trained neural network. This can be studied either under the assumption that the network is trained by some specific optimization algorithm or by simply assuming that a global minimum has been found without knowledge of the optimization algorithm. The relationship between generalization error and empirical error is a classic topic that has been well studied historically. For neural networks, most existing literature focuses on two-layer ReLU neural networks as a starting point. For example,~\citet{bach2017breaking,parhi2021banach,parhi2022near,e2019priori,neyshabur13601359}. Some works also leverage continuous formulations~\citet{mei2018mean}, while others approach the problem via asymptotic regimes~\citet{deng2022model,liang2022precise,mei2022generalization,yang2020rethinking,seroussi2022lower}. The salient double descent phenomenon for overparameterized neural networks also poses an intrinsic challenge beyond the classical characterization for general models. On the other hand, there are many empirical studies~\citet{zhang2021understanding,nagarajan2019generalization,neyshabur2017exploring,zhao2021estimating,schiff2021predicting}.

Among others, we would like to emphasize a recent work~\citet{huiyuan2023Nonasymptotic} by Huiyuan and Wei, which studies the nonasymptotic generalization property of a general two-layer ReLU neural network. It does not require specific network structure constraints assumed in other works and derives a nonasymptotic near-optimal generalization error bound. A closely related work is~\citet{neyshabur13601359}.

However, as far as we know, rigorous studies (rather than empirical or heuristic studies) of generalization properties for multilayer neural networks are few, albeit there are many works providing generalization error bounds in various ways. It appears like a puzzle for readers to understand their complex relationships. A superficial reason is that two-layer ReLU neural networks are quite close to linear functions, making them easier to transfer to traditional and well-studied statistical problems, whereas multilayer neural networks are highly nonlinear and nonconvex.

\subsection{Path Regularization}
Path regularization was proposed in works~\citet{neyshabur2015path,neyshabur2015norm} as a new form of regularization with favorable properties for ReLU neural networks. There are many discussions or variations~\cite{luo2021pathbn,zheng2019basispath,neyshabur2016pathnorm,latorre2023pathnorm}. And the recent studies~\cite{gonon2024phd,barron2019complexity,jiang2020neurips,dziugaite2021search} emphasize that it is the most promising candidate quantity for measuring the generalization properties of neural networks. We now discuss in details why we study the learning problem~\eqref{opt_problem} with PeSV regularization (a slight variant of path regularization). From a nonparametric or functional standpoint, a neural network should be considered as a nonlinear function rather than a machine with a set of parameters. This view has been adopted by many works. Under this consideration, a Banach function space associated with ReLU neural networks with finite depth and infinite width was developed by~\citet{wojtowytsch2020banach}. There is a very natural norm on this space with a well-defined mathematical foundation, and the path norm is simply its discrete analog (strictly, the $\ell^1$ path norm, while the PeSV norm is a slight variant of the $\ell^1$ path norm). In contrast, viewing neural networks as a set of parameters with weight decay regularization seems too naive to be a correct foundation for the theory of neural networks. Furthermore, the authors show that under this norm, the unit ball in the space for $L$-layer networks has low Rademacher complexity and thus favorable generalization properties, implying that using path norm as a regularizer is a priori expected to train a network that generalizes well. Another related piece of evidence is the observation that the $\ell^p$ path norm is related to the per-unit $\ell^p$ norm (also called max $\ell^p$ norm) in that, for ReLU networks, it is the minimum of the per-unit $\ell^p$ norms over all linearly rescaled networks representing the same function. Since per-unit $\ell^1$ and $\ell^2$ norms have shown great generalization properties, this suggests that the $\ell^1$ and $\ell^2$ path norms are important regularizers. Last but not least,~\citet{neyshabur2015path} advocates that the geometric invariance of ReLU networks under scaling should be taken into account and suggests that optimization algorithms should incorporate this invariance. Path norm is scaling-invariant, and that work develops Path-SGD, a scaling-invariant optimization algorithm that numerically achieves much better generalization than SGD and Adam. The work~\citet{ergen2023path} suggests studying path norm for training neural networks and shows that it has an equivalent convex formulation that favors GD/SGD in finding a global minimum.~\citet{gonon2023path} develops a first toolkit specialized for path-norm regularization to facilitate challenging concrete promises of path-norm-based generalization bounds. With all these considerations, it is worthwhile and desirable to mathematically investigate the generalization properties of learning algorithms with path norm as a regularizer. Here we take a slight variant of the path norm, the PeSV norm (path enhanced scaled variation norm), which is an $\ell^1$ and $\ell^2$ mixed version of the path norm. If we use the $\ell^1$ norm for the first layer, we recover the $\ell^1$ path norm. The $\ell^2$ norm here is not essential, and all our conclusions hold for the pure path norm. Thus we take problem~\eqref{opt_problem} in Section~\ref{sec4} as the object of study in this work.

Another point we need to emphasize is that the path regularization can be computed efficiently by using dynamic programming, as opposed to the apparently exponential complexity~\cite{neyshabur2015path}. This is very important for practical network training. 

Although path regularization has favorable properties only for ReLU networks, we find that our results actually hold for general Lipschitz activations. Therefore, in the following sections, we present the results in this generality unless otherwise specified. 

In passing, the advantage of path regularization for general Lipschitz activations remains to be explored and will be the subject of our future research.

\subsection{Our Contributions}
Based on our review of existing generalization theories for multilayer neural networks available in the literature, in the next section, we propose a framework for what constitutes a satisfactory generalization theory. This paper attempts to fill the gap between theory and practice by providing a fairly general rigorous characterization of the generalization property of a general multilayer neural network with path regularizations for quite general problems, including regression and classification, under mild assumptions. For the squared loss and ReLU activations the overparameterized rate cannot be improved uniformly in the radius of the target ball. Most importantly, the bound demonstrates the double descent phenomenon observed in practice for such networks.

Another work most closely related to ours is perhaps~\citet{parhi2022near}. They set up a correct space of functions of second-order bounded variation in the Radon domain, which turns out to be the same as the variation space~\citet{siegel2023characterization} studied before for associating shallow ReLU neural networks with the same form of learning problem, and finally establish a similar conclusion to ours in this space. One novel difference is that they choose estimators in this function space instead of the original shallow neural networks. Their main tool for proving the $L^2$ generalization error is a metric entropy argument corresponding to the underparameterized regime in our work; however, this would not be optimal for the aforementioned reason: overparameterization and underparameterization are different.

We believe that the main meaning and contributions of this work are as follows.
\begin{enumerate}
    \item \textit{We give the first near-complete and optimal nonasymptotic generalization error bound for multilayer neural networks with path regularization for general loss functions}. It is flagged with no constraints placed on network structures (e.g., intrinsic sparsity or boundedness of certain norms), no constraints on any particular optimization algorithm used, and no boundedness assumption on the loss. We explicitly model the target function space and take the approximation error into account. The latter point is notable as it is almost \textit{ignored} in previous works.
    \item \textit{To the best of our knowledge, our work is the first to establish the double descent phenomenon for deep networks with path regularization without relying on asymptotic analysis}. We do so in both directions: the upper bound exhibits the curve, and lower bounds valid at every target show that the error itself, and not only the bound, is non-monotone in the width. It also reveals the intrinsic mechanism of why the double descent phenomenon occurs. The analysis can be extended to other machine learning models. In principle, they will follow the same pattern to exhibit double descent. This paves the way to the final complete understanding of double descent nonasymptotically.
    \item \textit{We solve an open question posed in~\citet{wojtowytsch2020banach} by Weinan E et al. regarding approximation error rates in generalized Barron spaces}. 
\end{enumerate}
Our work is one of very few works that goes beyond two-layer neural networks by overcoming several challenges. The characteristics of our work can be listed as follows:
\begin{enumerate}
    \item We explicitly distinguish between the overparameterized regime and the underparameterized regime and estimate the generalization error separately, obtaining results that are as good as possible. Previous works either give a single bound that is not optimal or focus only on the overparameterization regime.
    \item When the loss is MSE and the activation is ReLU, we establish an information-theoretic lower bound showing that the overparameterized rate cannot be improved uniformly in the radius of the target ball. We are careful to state what this does and does not give: over a ball of fixed radius and in fixed dimension the bound is not minimax optimal, as the two-layer case shows (Remark~\ref{rmk_ball_vs_space}).
    \item We establish an optimal approximation bound for multilayer neural networks with ReLU activations in a suitable continuous function space called the generalized Barron space, greatly strengthening previously known results. To the best of our knowledge, this result is new.
    \item We show that the generalization upper bound can predict the double descent phenomenon, and that lower bounds on the error confirm it. Previous works have never been able to predict the double descent phenomenon from a theoretical point of view alone, still less to certify it from below.
\end{enumerate}

Although our results require path-norm regularization, recent work relates the stochastic gradient methods used to train neural networks to path norms~\cite{lyu2020iclr,gunasekar2018icml,lebeurrier2026arxiv,wu2023implicit}. In particular, this work~\cite{wu2023implicit} theoretically proves that for two layer ReLU networks, the stationary point found by SGD has a minimal path-norm, which is strong evidence that the same may hold for deep networks. Should that be the case, the present work would partially explain the double descent phenomenon for networks trained by SGD without explicit path regularization. And if one still uses common regularizations like weight decay, the regularization effect enforced by explicit regularization will not change the global double descent landscape but only change the smoothness or dynamic behavior  of this landscape. This discussion is not rigorous, but it indicates directions in which a full understanding of the phenomenon might be sought.

The organization of this paper is as follows: Section~\ref{sec1} discusses relevant works. Section~\ref{sec2} states the notation. In Section~\ref{sec3}, we set up a reasonable function space to work on in this paper. We define the appropriate normed function space to establish a framework for studying general problems modeled by multilayer neural networks. Section~\ref{sec4} sets up the learning problem under consideration and the optimization objects. Section~\ref{sec_approx} studies the approximation properties of this function space. Sections~\ref{sec5}~\ref{sec6}~\ref{sec7} and~\ref{sec8} study the empirical and generalization error bounds in the overparameterized and underparameterized regimes, respectively, and discuss their implications for the double descent phenomenon. Section~\ref{sec_lb} establishes lower bounds on the error and turns that discussion into a statement about the risk. Section~\ref{sec9} generalizes our theory to Lipschitz loss functions under some very mild conditions. We compare our results with classical results in Section~\ref{sec10} and demonstrate our superiority. Finally, Section~\ref{sec11} concludes the paper. All technical details, including all proofs, are presented in the Supplementary Information~\ref{appendix}.

\subsection{Main Results at a Glance}\label{sec_glance}

Before the technical development we state the three results the paper is built around, so that
a reader can see where it is going without first absorbing the machinery. Throughout, $f^*$ is
the unknown regression function, assumed to lie in a ball $W^L_M$ of the generalized Barron
space; $\mathbf{m}$ is the width vector; $\hat\theta$ is the path-regularized empirical risk
minimizer~\eqref{opt_problem}; and $H(\mathbf{m})$, $\delta_n$, $\lambda_1$, $\lambda_2$,
$f(n)$, $\xi$ and $A_\infty$ are as in Table~\ref{tab:notation}.

\textbf{(i) An upper bound with two regimes} (Theorem~\ref{para_gen}). With the regularization
prescribed by each regime,
\begin{align}
\|g(\cdot;\hat{\theta}) - f^*\|_2^2 \le C\Big\{ \underbrace{H(\mathbf{m})^2\|f^*\|_{W^L}^2}_{\text{approximation}}
+ \big(\sigma_\epsilon^2+\|f^*\|^2_{W^L}\big)\min\Big(\underbrace{\Lambda_{L,d}\sqrt{\tfrac{f(n)}{n}}}_{\text{overparameterized}},\ \underbrace{\tfrac{(2dm_1+4m_1m_2+\cdots)\log n}{n}}_{\text{underparameterized}}\Big)\Big\}.
\end{align}
The first term falls with the width and the second rises until it saturates, which is why the
bound is not monotone. The approximation term rests on a sharp rate for generalized Barron
spaces (Theorem~\ref{approx1_2}), which answers a question of~\citet{wojtowytsch2020banach}.

\textbf{(ii) A lower bound on the error itself} (Section~\ref{sec_lb}). For \emph{every}
target and at \emph{every} local minimizer, with no nondegeneracy or localization hypothesis,
\begin{align}
\|g(\cdot;\hat\theta)-f^*\|_n \;\ge\; \max\Big\{A_\infty(\mathbf{m},f^*),\ \min\big\{\|f^*\|_n,\ \tfrac{(\lambda-\xi)_+}{L_\sigma^{L-1}}\big\}\Big\},
\qquad
\|g(\cdot;\hat\theta)-f^*\|_2 \;\ge\; \operatorname{dist}_2\big(f^*,\mathcal{B}_{Q_\lambda}\big).
\end{align}
The first is obtained from the optimality condition restricted to a scaling ray, which turns
it into the derivative of a univariate polynomial and avoids differentiating the non-convex
penalty; the second needs no information about the design points and is therefore a statement
about the generalization error from the outset.

\textbf{(iii) Double descent of the error} (Theorem~\ref{double_descent}). Combining (i)
and (ii) under the schedule~\eqref{schedule}, there are widths $\mathbf{m}_0$ (a valley) and
$\mathbf{m}_\sharp$ (at the interpolation threshold $\delta_n=1$) with
\begin{align}
\frac{\text{error at }\mathbf{m}_\sharp}{\text{error at }\mathbf{m}_0}\ \longrightarrow\ \infty,
\qquad
\frac{\text{error at }\mathbf{m}_\sharp}{\text{error at large widths}}\ \longrightarrow\ \infty .
\end{align}
The error therefore descends, rises and descends again. Theorem~\ref{no_fixed_lambda} shows the
ascent is not an artefact of the prescription: no width-independent regularization can both
validate the underparameterized guarantee near the threshold and reach the final plateau.
Figure~\ref{fig:dd} shows the resulting shape.

\begin{figure}[htp]
    \centering
    \includegraphics[width=.75\textwidth]{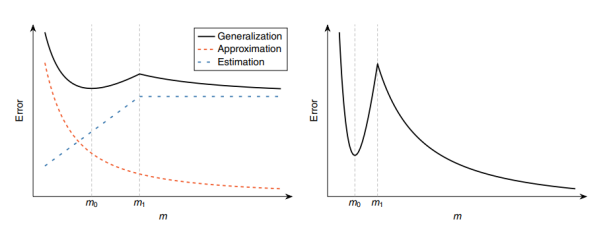}
    \caption{The shape established by Theorem~\ref{double_descent}: a $U$ over the
    underparameterized range, a peak at the interpolation threshold $\delta_n=1$, and a second
    descent to the plateau. Both the upper bound of Theorem~\ref{para_gen} and the floor
    of~\eqref{total_floor} have this shape, which is what makes the statement about the error
    rather than about a bound on it.}
    \label{fig:dd}
\end{figure}

\textbf{The mechanism.} Having both sides of the bound settles what produces the ascent, and it
is not the estimator fitting noise. Once $\lambda$ exceeds the dual norm $\xi$ of the noise --
which the prescribed schedule does at every width -- a slow-rate oracle inequality
(Theorem~\ref{slow_rate}) controls the error with no variance term at all, so there is no
overfitting to be traded against approximation. What grows with the width is instead the
regularization that the underparameterized bound requires: $\delta_n$ grows, $\lambda_2$ grows
with it, and the estimator is pushed further from the target. The peak marks the width at which
that demand is most severe, the rise is the price of shrinkage along the regularization path
(Remark~\ref{rmk_mechanism}), and by Theorem~\ref{no_fixed_lambda} no width-independent
$\lambda$ removes it while keeping the guarantees.

\textbf{What is not claimed.} The lower bound in the generalization norm carries a two-sided
condition on $\|f^*\|_{W^L}/\sigma_\epsilon$ (Remark~\ref{rmk_which_floor}); and the
overparameterized rate is sharp uniformly in $M$ but not over a ball of fixed radius in fixed
dimension (Remark~\ref{rmk_ball_vs_space}). The mechanism just described is a property of the
path-regularized estimator analyzed here: we do not claim that it accounts for the double
descent reported for networks trained without explicit regularization, for which the
corresponding question remains open.

\section{Related Work}\label{sec1}
We review relevant work on the generalization properties of neural networks, including both shallow (two-layer) and deep architectures, while abstracting away specific activation functions and regularization techniques. We subsequently discuss their relations and, especially, differences with our work. We suggest readers consult~\citet{huiyuan2023Nonasymptotic} and the references therein first for thorough discussions of the background and motivation of this problem, as well as several relations with other related theories, e.g., high-dimensional statistics and recent understanding of bias-variance tradeoff~\citet{derumigny2023lower}.

Neural Tangent Kernel (NTK)~\citet{jacot2018neural,chizat2019lazy} is a method to understand the training dynamics of NNs by formulating them as kernel gradient descent with the neural tangent kernel. When the number of neurons approaches infinity, this kernel is constant during training, facilitating easy analysis. This method requires the parameters to stay in a small neighborhood of their initial values, which does not align with realistic NNs. The generalization analysis is therefore not useful.

Mean field theory~\citet{sirignano2020mean,rotskoff2022trainability} aims to formulate the training dynamics of NNs as a stochastic partial differential equation in the continuous limit. The generalization property analysis via SPDE poses difficulties that do not seem easy to solve.

There are a number of works analyzing the generalization properties of two-layer neural networks. Most of them leverage the law of large numbers or the central limit theorem to represent the training dynamics of SGD as a stochastic partial differential equation in the continuous limit.~\citet{mei2018mean} proposes a mean field analysis of the training dynamics of two-layer NNs via SGD by recasting it as an SPDE in the continuous limit. However, it is not clear how this method can be generalized to multilayer NNs, and the analysis of the resulting SPDE is far from easy except under specific cases; therefore, generalization property analysis is also challenging.

Since the discovery of the double descent phenomenon~\citet{belkin2019reconciling}, many works have attempted to understand it theoretically under various assumptions.~\citet{hastie2022surprises,muthukumar2020harmless,belkin2020two} started detailed analyses for classical linear models or other simple models, e.g., the random feature model, which is equivalent to a two-layer neural network with the first-layer parameters fixed. Another line of attack is to perform asymptotic analysis for these models~\citet{deng2022model,liang2022precise,mei2022generalization,yang2020rethinking}, i.e., letting depth, width, dimension, and the number of training data approach infinity separately or jointly. Random matrix theory is largely leveraged in these works. Using this theory, they can find explicit generalization error expressions in the asymptotic regime and show that its curve with respect to the structural parameter demonstrates the ``double descent'' shape. In this line, a lower generalization error bound for two-layer neural networks is also developed in~\citet{seroussi2022lower}. Notably, a ``multiple descent'' phenomenon in infinite-dimensional linear regression and kernel ridgeless regression is proposed in~\citet{liang2020multiple,li2021multi}. Despite these advances, it is still unclear how these simple models can be helpful or insightful for multilayer NN models. In our opinion, the gap between them is potentially large. The reader can also check a very recent thorough discussion of double descent~\citet{schaeffer2023double}.

Several works have derived generalization properties for two-layer neural networks in certain variation spaces~\citet{bach2017breaking,parhi2021banach,parhi2022near}. Most of the existing work focuses on variational formulations of the empirical risk minimization problem. However, the network width of the solution is required to be smaller than the sample size, which does not fall within our regime of overparameterized NNs. Another attempt~\citet{e2019priori} allows the network width to grow to infinity; however, the $\ell_1$ path norm regularization induces potential sparsity in parameters, which does not seem to have implications for intrinsically overparameterized NNs.

For deep neural networks, there are some empirical analyses regarding the phenomenological aspects of generalization ability. One of the most surprising features is that any deep neural network with a number of parameters greater than the number of data points has perfect finite-sample expressivity, which forces a rethinking of what generalization means~\citet{zhang2021understanding}.~\citet{nagarajan2019generalization} empirically found that generalization crucially depends on a notion of model capacity that is restricted through the implicit regularization of $\ell_2$ distance from initialization.~\citet{neyshabur2017exploring} discussed different candidate complexity measures that might explain generalization in neural networks. It also outlines a concrete methodology for investigating such measures and reports on experiments studying how well the measures explain different phenomena.~\citet{zhao2021estimating,schiff2021predicting} propose some new measures, e.g., sparsity, Gi-score, and Pal-score, that can explicitly calculate and compare generalization ability. A deep insight was analyzed in~\citet{neyshabur2017exploring}, where they show quantitative evidence, based on representational geometry, of when and where memorization occurs, and link double descent to increasing object manifold dimensionality.

To the best of our knowledge, an incomplete list of rigorous generalization theories for multilayer neural networks (more than two layers) includes~\citet{neyshabur2017pac, bartlett2017spectrally, bartlett2019nearly, tirumalageneralization, gu2023unraveling, wang2022generalization, allen2019learning, jakubovitz2019generalization, neyshabur2018towards}. Early classical works~\citet{neyshabur2017pac, bartlett2017spectrally, bartlett2019nearly, tirumalageneralization} bound the generalization error using complexity measures such as VC dimension, Rademacher complexity, and Bayesian PCA. It is well known that these methods often require the loss function to be bounded, which is not practical. As classical theory focuses on bounding the generalization error by empirical errors and other complexity quantities related to neural networks, it does not pay attention to the estimation of empirical error, meaning that it does not offer a complete picture of the generalization error. These bounds are also not optimal, as they do not distinguish between the overparameterization and underparameterization regimes that should have fundamental differences, as the salient double descent phenomenon indicates, and thus the upper bound will never predict the double descent phenomenon.~\citet{tirumalageneralization} does not rely on the aforementioned complexity norms but uses direct analytic analysis to establish concentration inequalities that are, in spirit, similar to our methods, but it still requires boundedness of the loss function and is not optimal for the same reason. For recent work, a kind of generalization error bound analysis appears in~\citet{gu2023unraveling}, but it is for two- and three-layer neural networks only and studies how generalization error converges under the gradient descent algorithm.~\citet{wang2022generalization} provides a generalization error bound for multilayer neural networks under the stochastic gradient descent algorithm. Again, it focuses on a specific optimization algorithm. Furthermore, these two works do not consider the target function space, which means they do not take into account approximation errors and do not estimate empirical errors. The second work also imposes boundedness constraints on certain norms of multilayer neural networks. Our work differs from theirs in that we do not require knowledge of the optimization algorithm beforehand, do not impose any structural constraints on the multilayer neural networks, and integrate the target function space (and thus approximation error) into our generalization bound expression.

After reviewing the aforementioned works, we find that the generalization theories they present either deal with two-layer neural networks, consider only the asymptotic regime, impose structure constraints on neural networks, depend on particular optimization algorithms, require boundedness of the loss function, or are not aware of potential differences between underparameterization and overparameterization and thus are not predictive of the double descent phenomenon, and they do not consider approximation error a priori.

Through inspecting people's experience in deep learning in recent decades, we propose here what we think a satisfactory generalization theory should be. We call a set of generalization error expressions or bounds for multilayer neural networks \textit{complete} if:
\begin{enumerate}
    \item it measures the difference between the estimated model and the true model in a nonasymptotic fashion;
    \item all terms in the expressions or bounds are computable without running experiments first, meaning that it must involve estimation of empirical error;
    \item it applies to all realistic neural networks without \textit{essential} constraints on their structures, e.g., sparsity, or asymptotic regime;
    \item it sheds light on the famous double descent phenomenon in practice (otherwise it is not tailored for neural networks and not optimal);
    \item it applies to most activation functions;
    \item it applies to most common loss functions in practice; in particular, it does not require boundedness of loss functions, as most common losses are not bounded;
\end{enumerate}

Although the aforementioned empirical findings and partial results are necessary steps toward a better understanding of the generalization properties of NNs, no established, even near-complete, theory has been made public for the simplest deep NN: the multilayer ReLU NN.

\section{Notation}\label{sec2}
$\mathbb{E}$ denotes expectation. For a vector $x\in \mathbb{R}^d$, let $\|x\|_q$ represent its $\ell^q$-norm. Denote by $\mathbf{B}^d$ and $\mathbf{S}^{d-1}$ the $\ell_2$ unit ball and the unit sphere, respectively, and by $\mathbf{B}^d(R)$ the ball of radius $R$. For a function $f$, $\|f\|_{L^\infty(D)}$ signifies the $L_{\infty}$-norm on a domain $D$, while $\|f\|_2$ and $\|f\|_n$ denote the $L^2$-norm under the data distribution and its empirical counterpart, respectively. We use $C^0(K)$ and $C^{0,\alpha}(K)$ ($0<\alpha\le 1$) to denote the space of continuous functions and the space of H\"older functions on a domain $K$, respectively. For quantities $X,Y$, we use $X \sim Y$ to mean $C_1Y \le X \le C_2Y$ for some universal constants $C_1,C_2$. Explicitly, $\|f\|_2^2 = \int_{\mathbf{B}^d} f^2 \, d\mu$ and $\|f\|_n^2 = n^{-1}\sum_{i=1}^n f(x_i)^2$, where $\mu$ is the distribution of the predictors fixed in Assumption~\ref{assump2}; the Lebesgue measure is never used, and every generalization error in this paper is measured in $L^2(\mu)$. The approximation results of Section~\ref{sec_approx} are stated for a general probability measure $P$ with $\operatorname{supp}(P)\subset\mathbf{B}^d(R)$, and are applied in the statistical sections with $P=\mu$, for which $R=1$; the norm $\|\cdot\|_{L^2(P)}$ appearing there is thus the same as $\|\cdot\|_2$. Since $\mu$ is supported in the unit ball, $\max_i\|x_i\|_2\le1$ holds automatically under Assumption~\ref{assump2}; we still list it explicitly in those statements that hold for an arbitrary fixed design, where no distributional assumption on the $x_i$ is in force. 

\begin{table}[h]
\centering
\caption{Symbols used throughout. Section numbers point to the defining occurrence.}
\label{tab:notation}
\begin{tabular}{@{}llc@{}}
\toprule
Symbol & Meaning & Defined in \\ \midrule
$\nu(\theta)$ & PeSV (path) norm of the parameters & \eqref{s_path_norm} \\
$W^L$, $W^L_M$ & generalized Barron space; its ball of radius $M$ & \sectionname~\ref{sec3} \\
$\mathcal{F}(\mathbf{m},F)$ & width-$\mathbf{m}$ networks with $\nu(\theta)\le F$ & \sectionname~\ref{sec5} \\
$\|\cdot\|_n$, $\|\cdot\|_2$ & empirical and $L^2(\mu)$ norms & \sectionname~\ref{sec2} \\
$H(\mathbf{m})$ & approximation (bias) factor & \sectionname~\ref{sec_approx} \\
$\delta_n$ & $n^{-1}L_\sigma(dm_1+m_1m_2+\cdots+m_{L-1})\log n$; $\delta_n<1$ defines the underparameterized regime & \sectionname~\ref{sec7} \\
$\beta$ & growth exponent of the weight count in a dilation; $1$ if $L=2$, else $2$ & \eqref{beta_def} \\
$f(n)$ & confidence parameter $n^{(\beta-1)/(\beta+1)}(\log n)^{2/(\beta+1)}$ & \eqref{fn_def} \\
$\lambda_1$, $\lambda_2$ & regularization prescribed in the over- and underparameterized regimes & \sectionname\ \ref{sec6}, \ref{sec7} \\
$\xi$ & empirical process term $L_\sigma^{L-1}\|\epsilon\|_2/\sqrt n$ & \eqref{xi_def} \\
$A_\infty(\mathbf{m},f^*)$ & distance from $f^*$ to width-$\mathbf{m}$ networks & \eqref{Ainf_def} \\
$Q_\lambda$ & budget $\frac{1}{2n\lambda}\sum_iy_i^2$ bounding $\nu(\hat\theta)$ & \sectionname~\ref{sec_lb} \\
$\rho(f^*)$, $\mu(f^*)$ & spread $\|f^*\|_2/\|f^*\|_{W^L}$ and coherence & \eqref{rho_def} \\
$\mathrm{Fl}(\mathbf{m})$ & the floor on the error at width $\mathbf{m}$ & \eqref{total_floor} \\
\bottomrule
\end{tabular}
\end{table}

For any metric $\rho$ on a family of functions $\mathcal{F}$ and $\delta > 0$, we denote $\mathcal{N}(\delta, \mathcal{F}, \rho)$ the covering number, and $\log\mathcal{N}(\delta, \mathcal{F}, \rho)$ the metric entropy, as usual.

Let $\sigma$ denote a Lipschitz continuous activation function with Lipschitz constant $L_\sigma$ and $\sigma(0)=0$. In particular, $|\sigma(x)| \le L_\sigma \|x\|$. Any multilayer network with Lipschitz activation not satisfying $\sigma(0)=0$ can be equivalently transformed to the former, which we discuss at the end of Section~\ref{sec3}.

Keep in mind that \textit{the constants in our theorems, propositions and lemmas may depend on the number of hidden layers $L$ of the neural networks, the activation function-related quantities: the Lipschitz constant $L_\sigma$, and the loss function-related quantities: the Lipschitz constants $L_1, L_1^{\prime}$ and the strong convexity constant $\gamma$, which can be easily inspected from the proofs and is omitted in the statements of relevant results. Importantly, they do not depend on the width vector $\mathbf{m}$ of the networks}.

In the remainder of this paper, we abbreviate \textit{multilayer neural network}, \textit{deep neural network},  \textit{multilayer network} or even \textit{network} to refer to a multilayer fully connected feedforward neural network with Lipschitz activation for brevity, unless otherwise specified.

\section{Multilayer Neural Space and Its Properties}\label{sec3}
As we explicitly take into account the space of true models rather than ignoring it completely as in previous works, to ensure that a machine learning model learned from a predefined model class exhibits favorable empirical and generalization properties, a necessary condition is that this model class possesses the capability to approximate the underlying true model well so that we can perform approximation theory. In this section, we define the true model space specifically associated with multilayer neural networks. We follow the setup of~\citet{wojtowytsch2020banach}.

Keep in mind that $\sigma$ below is any Lipschitz activation with $\sigma(0)=0$, although~\citet{wojtowytsch2020banach} deals with ReLU. Our definition of the function space relies on the construction of the \textit{generalized Barron space}, denoted by $\mathcal{B}_{X,K}$, called the generalized Barron space modeled on $X$, where $K \subset \mathbb{R}^d$ is a compact set and $X$ is a Banach space such that:
\begin{enumerate}
    \item $X$ embeds continuously into the space $C^{0,1}(K)$ of Lipschitz functions on $K$, and
    \item the closed unit ball $B^X$ in $X$, which is a Polish space, is closed in the topology of $C^0(K)$.
\end{enumerate}
$\mathcal{B}_{X,K}$ is constructed as follows:
\begin{equation}
\begin{aligned}
 f_{\mu} & = \int_{B^X} \sigma(g(\cdot))\,\mu(dg) \\
\|f\|_{X,K} & = \inf \{\|\mu\|_{M(B^X)} : \mu \in M(B^X) \text{ s.t. } f = f_{\mu} \text{ on } K \} \\
\mathcal{B}_{X,K} & = \{ f \in C^0(K): \|f\|_{X,K} < \infty \} 
\end{aligned}
\end{equation}
Here, $M(B^X)$ denotes the space of (signed) Radon measures on $B^X$, and $\mu$ is a finite signed measure on the Borel $\sigma$-algebra of $B^X$. The integral $\int_{B^X}$ is the Bocher integral, and $\|\mu\|_{M(B^X)}$ is the total variation norm of the measure $\mu$.  We refer the readers to~\citet{wojtowytsch2020banach} for more details.

We briefly explain the historic development of  this concept. The approximation properties of two-layer neural networks have been a subject of enduring interest, particularly concerning their ability to break the curse of dimensionality when approximating certain high-dimensional functions. This phenomenon is rigorously characterized by the theory of \textbf{Barron spaces}. Originally formulated by \cite{barron1993universal}, the classical Barron space consists of functions that can be represented as an infinite integral over neurons, parameterized by a spectral measure.

A function $f$ belongs to the Barron space if it admits a representation of the form
\begin{equation}
f(\mathbf{x}) = \int_{\Omega} a  \sigma(\mathbf{w} \cdot \mathbf{x} + b)  \rho(da, d\mathbf{w}, db) + c
\label{eq:barron_representation}
\end{equation}
where $\sigma$ is a bounded activation function (e.g., a sigmoid or ReLU$^k$), $\rho$ is a probability measure over the neuron parameters $(a, \mathbf{w}, b)$, and $c$ is a constant. The complexity of the function is measured by the \textbf{Barron norm}, $|f|_{\mathcal{B}}$, which is typically defined as the infimum over all such representations of the total variation of the spectral measure associated with the output weights $a$.

\textbf{Generalized Barron spaces} extend this concept to a broader, more functional analytic framework. They are defined by replacing the specific integral form with a more abstract representation using the \textbf{inverse Radon transform} or, equivalently, by considering functions that lie in the dual space of a particular function space induced by the activation function.

A particularly powerful modern formulation defines the generalized Barron norm for a function $f$ with respect to an activation function $\sigma$ as
\begin{equation}
|f|_{\mathcal{B}_{\sigma}} = \inf_{h \in L^1(\mathbb{R}^{d+1})}  \int_{\mathbb{R}^{d+1}} |h(a,\mathbf{w},b)| da d\mathbf{w} db
\end{equation}
where $f$ has a representation
\begin{equation}
f(\mathbf{x}) = \int_{\mathbb{R}^{d+1}} h(a,\mathbf{w},b) \sigma(\mathbf{w} \cdot \mathbf{x} + b) da d\mathbf{w} db.
\label{eq:generalized_barron_norm}
\end{equation}

This perspective offers several key advantages for machine learning theory:
\begin{enumerate}
\item \textbf{Convexity:} The space $\mathcal{B}_{\sigma}$ is a Banach space, turning the non-convex training of width-limited networks into a convex optimization problem over measures in the infinite-width limit.
\item \textbf{Representation Cost:} The Barron norm precisely quantifies the implicit bias of gradient-based methods (like gradient descent) when training over-parameterized two-layer networks, often correlating with the weight decay regularizer.
\item \textbf{Approximation Theory:} Functions with a finite Barron norm can be approximated by a finite-width neural network with a rate independent of the input dimension $d$, providing a mathematical justification for why neural networks excel in high dimensions \citep{bach2017breaking}.
\item \textbf{Generalization Bounds:} The norm serves as an effective capacity measure, leading to generalization bounds that do not explicitly depend on the number of parameters, aligning with the observed behavior of large neural networks \citep{neyshabur2018towards}.
\end{enumerate}

In summary, generalized Barron spaces provide a rigorous mathematical setting for understanding the approximation, optimization, and generalization properties of shallow neural networks, bridging the gap between finite networks and their infinite-width limits.

\begin{theorem}[Theorem 2.7 in~\citet{wojtowytsch2020banach}]
\label{property_g_barron}
The following are true:
\begin{enumerate}
    \item $\mathcal{B}_{X,K}$ is a Banach space.
    \item $\mathcal{B}_{X,K} \hookrightarrow C^{0,1}(K)$ and the closed unit ball of $\mathcal{B}_{X,K}$ is a closed subset of $C^0(K)$.
    \item If $\sigma$ is positive homogeneity, then $X\hookrightarrow \mathcal{B}_{X,K}$ and $\|f\|_{\mathcal{B}_{X,K}} \le \frac{2}{L_\sigma} \|f\|_X$.
\end{enumerate}
\end{theorem}
Here positive homogeneity means $\sigma(\lambda x)=\lambda x$ for $\lambda > 0$.

The function space that we consider is defined recursively as follows:
\begin{enumerate}\label{neural space}
    \item $W^1(K) = (\mathbb{R}^d)^*\oplus \mathbb{R} \equiv \mathbb{R}^{d+1}$ is the space of affine functions from $\mathbb{R}^d$ to $\mathbb{R}$ (restricted to $K$).
    
    We take the standard Euclidean $\ell_2$-norm on $\mathbb{R}^d$. Thus, the norm of $W^1(K)$ is its dual, which is also the $\ell_2$-norm of $\mathbb{R}^{d+1}$, different from the $\ell_1$-norm used in~\citet{wojtowytsch2020banach}. This change affects the results in Section 3 of~\cite{wojtowytsch2020banach} accordingly but does not change their validity.
    \item For $L\ge 2$, we set $W^L(K) = \mathcal{B}_{W^{L-1}(K),K}$.
\end{enumerate}
We call this space the \textit{multilayer neural space}.

The above definition is somewhat abstract. We will define three kinds of $L$-layer neural networks, including the standard ones used in practice, in explicit ways. It turns out that they are all special cases of $W^L$. They also play a role of the motivation of the proposal of the generalized Barron spaces. These are also introduced in~\citet{wojtowytsch2020banach}.

We start from the definition of a finite-width multilayer neural network and then generalize to infinite width, i.e., the continuous case:

\begin{definition}\label{L_layer_nn}
A \textit{finite $L$-layer neural network} is a function of the type
\begin{equation}
f(x) = \sum_{i_{L-1}=1}^{m_{L-1}} a_{i_{L-1}}^L \sigma\left( \sum_{i_{L-2}=1}^{m_{L-2}} w_{i_{L-1}i_{L-2}}^{L-1} \sigma\left( \sum_{i_{L-3}=1}^{m_{L-3}} \cdots \sigma\left( \sum_{i_1=1}^{m_1} w_{i_2i_1}^2 \sigma\left( \sum_{i_0=1}^{d+1} w_{i_1i_0}^1 x_{i_0} \right) \right) \right) \right),
\end{equation}
where $a^L$ and $w^l$ for $1\le l\le L-1$ are the weight matrices.
\end{definition}

Note that the bias terms in hidden layers are omitted for simplicity, however, any network with bias terms can be transformed into our expression above, which will be discussed in the Supplementary Information~\ref{appendix}. Also, note that the last layer does not have an activation composed with. We call $a^L, w^l$ for $1\le l \le L-1$ the weight matrices, and their elements $w^l_{ij}$ weights. 
We call $L$ the depth and $\mathbf{m}=(m_1,m_2,\dots,m_{L-1})$ the width vector. We denote the width $m$ of an $L$-layer neural network to be the maximum value of the width vector, $m:=\max\{m_1,m_2,\dots,m_{L-1}\}$, and the bottleneck $b$ to be the minimum value of the width vector, $b:=\min\{m_1,m_2,\dots,m_{L-1}\}$.


We denote the parameters of this neural network as $\theta = (a^L, w^{L-1},\dots,w^2,w^1)$, then we also write the above function as $f(x;\theta)$. The space of $f(x;\theta)$ is denoted by $X_{m_{L-1},\dots,m_1;K}$, and the space of parameters $\theta$ is denoted by $\Theta$.

At the end of this section, for the reader's convenience, we will argue why any network with Lipschitz activation can be transformed into this form, i.e. transformed to an activation $\sigma$ with $\sigma(0)=0$.


The above expression can be immediately generalized to the countably infinite-width case.

\begin{definition}
A \textit{fully connected $L$-layer infinite-width neural network} is a function of the type
\begin{equation}\label{countable_nn}
f(x) = \sum_{i_{L-1}=1}^{\infty} a_{i_{L-1}}^L \sigma\left( \sum_{i_{L-2}=1}^{\infty} w_{i_{L-1}i_{L-2}}^{L-1} \sigma\left( \sum_{i_{L-3}=1}^{\infty} \cdots \sigma\left( \sum_{i_1=1}^{\infty} w_{i_2i_1}^2 \sigma\left( \sum_{i_0=1}^{d+1} w_{i_1i_0}^1 x_{i_0} \right) \right) \right) \right).
\end{equation}
\end{definition}

To fully generalize to a continuous neural network, one can set measures on index sets that parameterize the weights on different layers.

\begin{definition}
For $0\le i \le L$, let $(\Omega_i,\mathcal{A}_i,\pi^i)$ be probability spaces, where $\Omega_0 = \{0,1,\dots,d\}$ and $\pi^0$ is the normalized counting measure. Consider measurable functions $a^L: \Omega_{L-1} \rightarrow \mathbb{R}$ and $w^i:\Omega_{i} \times \Omega_{i-1} \rightarrow \mathbb{R}$ for $1 \le i \le L-1$. Then define
\begin{align}\label{con_nn}
f_{a^L,w^{L-1},\dots,w^1}(x) = & \int_{\Omega_{L-1}} a^{L}_{\theta_{L-1}} \sigma \left( \int_{\Omega_{L-2}} \cdots \sigma \left( \int_{\Omega_1} w^2_{\theta_2,\theta_1} \sigma \left( \int_{\Omega_0} w^1_{\theta_1,\theta_0} x_{\theta_0}\,\pi^0(d\theta_0) \right) \pi^1(d\theta_1) \right) \right. \\
& \left. \cdots \,\pi^{L-1}(d\theta_{L-1}) \right).
\end{align}
\end{definition}

In particular, if each $\Omega_i$ is a finite discrete space (or countably infinite space) with a discrete uniform measure (or discrete probability measure), we recover the previous two definitions of fully connected $L$-layer (or countably infinite-width) neural networks.

\begin{remark}
Note that what we term an $L$-layer neural network above has $L-1$ hidden layers plus one output layer. This is slightly different from the definition in Section 3 of~\citet{wojtowytsch2020banach}, where it has $L$ hidden layers.
\end{remark}

For the $L$-layer neural network~\ref{L_layer_nn},~\citet{wojtowytsch2020banach} shows a crucial property.

\begin{theorem}\label{L_layer_property}
If $f$ is of the form~\ref{L_layer_nn}, then:
\begin{enumerate}
    \item $f\in W^L$,
    \item $\|f\|_{W^L} \le \sum_{i_{L-1}=1}^{m_{L-1}} \cdots \sum_{i_1=1}^{m_1} |a_{i_{L-1}}^L w_{i_{L-1}i_{L-2}}^{L-1}\cdots w_{i_2i_1}^2| \,\|w_{i_1}^1\|_2$ \label{norm_compare}
\end{enumerate}
where $w_{i_1}^1 = (w_{i_1,0}^1,w_{i_1,1}^1, \dots, w_{i_1,d}^1)$.
\end{theorem}

The expression on the right-hand side of~\eqref{norm_compare} is the discrete analog of $\|f\|_{W^L}$. Notice that the scaled variation norm proposed in~\citet{parhi2022near} corresponds to the right-hand side of~\eqref{norm_compare} for $L=2$. Partially motivated by this and other thorough verifications throughout this work, we propose our novel definition of the regularization term for multilayer neural networks~\ref{L_layer_nn}.

\begin{definition}
For an $L$-layer neural network with parameters $\theta = (a^L, w^{L-1},\dots,w^1)$, its \textit{path enhanced scaled variation norm}, abbreviated as PeSV norm and denoted by $\nu(\theta)$, is defined as the right-hand side of~\eqref{norm_compare}:
\begin{equation}
\nu(\theta) = \sum_{i_{L-1}=1}^{m_{L-1}} \cdots \sum_{i_1=1}^{m_1} |a_{i_{L-1}}^L w_{i_{L-1}i_{L-2}}^{L-1}\cdots w_{i_2i_1}^2| \,\|w_{i_1}^1\|_2.
\end{equation}
Sometimes we will write $\nu(\theta)$ as $\nu(f)$ if $f$ is of the form~\ref{L_layer_nn} for convenience.
\end{definition}

There is an obvious extension of PeSV to neural networks~\eqref{countable_nn} and~\eqref{con_nn}, but we will not need to discuss it here.

~\citet{wojtowytsch2020banach} shows that $W^L$ is the most suitable space to study $L$-layer neural networks in the sense that $W^L$ is the smallest space containing all limits of $L$-layer neural networks~\ref{L_layer_nn} in the H\"older function space $C^{0,\alpha}$ for any $\alpha < 1$. The reader should look into that paper for full details. Furthermore, all index set-based definitions~\ref{L_layer_nn}, \eqref{countable_nn}, and \eqref{con_nn} are all contained in $W^L$.
The path enhanced scaled variation norm can also be written in terms of weight matrices. Let the vector $W$ be the sequential product of the weight matrices except for the first layer, i.e., $W := a^L \prod_{i=2}^{L-1} w^i$. Let $W_k$ be its $k$-th element. Then:
\begin{equation}\label{s_path_norm}
\nu(\theta):= \sum_{k=1}^{m_1} |W_k| \,\|w^1_k\|_2.
\end{equation}

Note that when $L=2$, $\nu(\theta)$ becomes the scaled variation norm denoted in~\citet{huiyuan2023Nonasymptotic}. Thus~\eqref{s_path_norm} is indeed a multilayer generalization of the scaled variation norm. This regularization (or scaled variation norm) is not as strange as it appears. The origin of the scaled variation norm was discussed in~\citet{parhi2021banach,parhi2022near}, where it corresponds to the second-order total variation of a function in a certain function space in the offset variable of the (filtered) Radon domain. The authors established a representer theorem for two-layer ReLU networks as the solution of a variational problem with this total variation as the regularization. When the function is represented by a two-layer ReLU network, it reduces to the scaled variation norm. It is argued that when trained on two-layer ReLU neural networks, it promotes a sparse superposition representation of ReLU ridge functions as the solution. It also shows that the scaled variation norm is equivalent to weight decay regularization for two layer networks in the sense that the minimizers of problems with these two regularizations are the same.

As aforementioned, any multilayer network with a Lipschitz activation $\sigma$ with $\sigma(0) \neq 0$ can be transformed into a network with a Lipschitz activation with $\sigma(0)=0$. This can be easily seen from the following expression:
\begin{align}\label{general_activation}
\sigma\left(\sum_{j=1}^n a_{i,j}x_j\right) = \mathbf{I}\left(\sigma^*\left(\sum_{j=1}^n a_{i,j}x_j\right) \times 1 + \sigma(0) \times 1\right)
\end{align}
where $\sigma^* = \sigma - \sigma(0)$ so that $\sigma^*(0)=0$ and $\mathbf{I}$ is the identity activation. We can immediately derive from this expression that our assertion is correct. Thus our work also works for general  Lipschitz activations.

All the results in the following sections are presented in the language of multilayer neural networks without biases and activations $\sigma$ with $\sigma(0)=0$. As we have discussed above, multilayer neural networks with biases or activations $\sigma$ with $\sigma(0)\neq 0$ can be considered as the former with some weights fixed a priori. It is easy to adapt our results to this situation.

\section{Problem Setup}\label{sec4}
The framework of the learning problem we discuss in this work is described as follows: suppose we have observed predictors $\mathbf{x}_i \in \mathbb{R}^d$ and responses $y_i \in \mathbb{R}$ generated from the nonparametric regression model
\begin{equation}
\label{problem}
y_i = f^*(x_i) + \epsilon_i, \quad i=1,\dots,n
\end{equation}
where $f^*$ is an unknown function to be estimated and $\epsilon_i$ are random errors.

In order to learn $f^*$ from the training sample, we adopt the penalized empirical risk minimization (ERM) framework and seek to minimize
\begin{align}\label{opt_problem}
J_n(\theta;\lambda) = \frac{1}{2n} \sum_{i=1}^n (y_i - g(x_i;\theta))^2 + \lambda \nu(\theta)
\end{align}
where $g(\cdot ;\theta)$ is the $L$-layer neural network~\ref{L_layer_nn}, $\nu(\theta)$ is the PeSV norm in~\eqref{s_path_norm}, and $\lambda > 0$ is a regularization parameter.

We denote the solution of this optimization problem in the space $X_{m_{L-1},\dots,m_1;\mathbf{B}^d}$ by
\begin{equation}
\hat{\theta} = \operatorname*{argmin}_{\theta \in \Theta} J_n(\theta;\lambda).
\end{equation}

We impose the following assumptions on our problem~\eqref{problem}, following~\citet{huiyuan2023Nonasymptotic}.

\begin{assumption}
\label{assump1} $f^* \in W^L_M \equiv \{f\in W^L: \|f\|_{W^L} \le M\}$ for a prespecified $L$ and some constant $M > 0$.
\end{assumption}

\begin{assumption}\label{assump2} $x_i \sim \mu$ independently, where $\mu$ is supported on $\mathbf{B}^d$.
\end{assumption}

\begin{assumption}
\label{assump3} $\epsilon_i \sim N(0,\sigma_{\epsilon}^2)$ independently and are independent of $x_i$.
\end{assumption}

These assumptions are quite standard and commonly adopted in most literature. e.g., Assumption~\ref{assump2} holds because the predictors are normally bounded and can be normalized. See~\citet{huiyuan2023Nonasymptotic} for further discussion of these assumptions. Since we follow the aforementioned assumptions in the remainder of the paper where $\mu$ is supported on a unit ball, for convenience we abbreviate $W^l(\mathbf{B}^d)$ to $W^l$, unless otherwise specified.

\textit{In Section~\ref{sec8}, we extend our theory to general Lipschitz loss functions (under some very mild conditions). Thus the problem we study also includes other common machine learning problems, not restricted to regression problems.}

Since the PeSV norm is a slight modification of the path norm, and given the superior performance of the scale-invariant optimization algorithm Path-SGD for ReLU networks proposed in~\citet{neyshabur2015path}, it is highly plausible that the variant of Path-SGD optimization adapted to the PeSV norm obtains better results than SGD and Adam for ReLU networks, justifying the value of this regularization. On the other hand, the scaled variation norm is derived from the second-order total variation of a target function in the offset variable of the (filtered) Radon domain as mentioned before; the parallel result for the PeSV norm is an interesting open question.

\section{Approximation Property of Multilayer Neural Spaces}\label{sec_approx}
We establish the approximation property of the multilayer neural space~\ref{neural space} by deep networks in terms of $L^2$ and $L_{\infty}$ norms. This property is in fact demonstrated in Theorem 3.6 of~\citet{wojtowytsch2020banach} and Theorem 1 of~\citet{huiyuan2023Nonasymptotic}. It can be regarded as the multilayer generalization of the same result in~\citet{huiyuan2023Nonasymptotic} for two-layer networks, albeit the latter is in $L_{\infty}$ norm and we are in $L^2$ norm. We show that the approximation can achieve Monte Carlo rate. All proofs of the results in this section are deferred to the Supplementary Information~\ref{appendix}.

The first $L^2$ approximation result we will present is Theorem 3.6 from~\citet{wojtowytsch2020banach}.

\begin{theorem}\label{approx1}
Let $P$ be a probability measure with compact support $\operatorname{supp}(P) \subset \mathbf{B}^d(R)$. Then for any $L \ge 1$, $f \in W^L$, and $m \in \mathbb{N}$, there exists a finite $L$-layer ReLU neural network
\begin{align}
f_m(x) = \sum_{i_{L-1}=1}^m a_{i_{L-1}}^L \sigma\left( \sum_{i_{L-2}=1}^{m^2} w_{i_{L-1}i_{L-2}}^{L-1} \sigma\left( \sum_{i_{L-3}=1}^{m^3} \cdots \sigma\left( \sum_{i_1=1}^{m^{L-1}} w_{i_2i_1}^2 \sigma\left( \sum_{i_0=1}^{d} w_{i_1i_0}^1 x_{i_0} \right) \right) \right) \right)
\end{align}
such that
\begin{align}\label{approx1_norm_b}
\|f_m-f\|_{L^2(P)} \le \frac{(L-1)(2+R)\|f\|_{W^L}}{\sqrt{m}}
\end{align}
and the norm bound
\begin{align}
\sum_{i_{L-1}=1}^m \cdots \sum_{i_1=1}^{m^{L-1}} \sum_{i_0=1}^{d} |a_{i_{L-1}}^L w_{i_{L-1}i_{L-2}}^{L-1} \cdots w^2_{i_2,i_1}| \,\|w^1\|_2 \le \|f\|_{W^L}
\end{align}
holds.
\end{theorem}

The above result can be slightly modified for any Lipschitz activation. Although it is in terms of $L^2$ norm, it already suffices for the purpose of this paper. However, due to the exponential dependence of the network widths on the depth, using this bound will not yield a strong empirical error bound. This problem was noticed in~\citet{wojtowytsch2020banach}. It finds that the convergence rate is roughly $W^{-1/[2(2L-3)]}$, where $W$ is the number of parameters. Only for $L=2$ does this rate achieve the Monte Carlo rate $W^{-1/2}$. In contrast, the approximation property for two-layer neural networks is in terms of $L_{\infty}$ norm; however, it depends on the dimension $d$ of the problem, whereas~\eqref{approx1_norm_b} is independent of $d$ and therefore does not incur the curse of dimensionality (but does have a curse of depth).

We first easily transform the above results from ReLU to any Lipschitz activation $\sigma$ with $\sigma(0)=0$. To rigorously state the result, we need the following notation. For a width vector $\mathbf{m}$, we call a vector $\mathbf{m}^{\uparrow}$ of the same length the \textit{decreasing power sequence of $\mathbf{m}$} if $\mathbf{m}^{\uparrow}$ is of the form $\{m'^{L-1},m'^{L-2},\cdots, m'\}$ for some positive integer $m'$ and $\mathbf{m}^{\uparrow} \le \mathbf{m}$ elementwise. $\mathbf{m}^{\uparrow}$ is called \textit{maximum} if there does not exist $m''>m'$ such that the decreasing power sequence of $\mathbf{m}$ with the last term $m''$ exists. 
For example, $\mathbf{m}=\{130,166,32,17,8,5,7\}$, then the maximum decreasing power sequence of $\mathbf{m}$ is $\mathbf{m}^{\uparrow}=\{128,64,32,16,8,4,2  \}$ 

\begin{theorem}~\label{approx1_2}
With the same assumptions as in Theorem~\ref{approx1}, except that the activation can be any Lipschitz function $\sigma$ with Lipschitz constant $L_\sigma$, and given a depth $L$ and width vector $\mathbf{m}=\{m_1,\dots, m_{L-1}\}$, and letting $\mathbf{m}^{\uparrow} = \{m'^{L-1},m'^{L-2},\cdots, m'\}$, there exists $\hat{f} \in X_{m_1,\dots,m_{L-1};\mathbf{B}^d(R)}$ such that
\begin{equation}\label{approx2_norm_b}
\|\hat{f} - f\|_{L^2(P)} \le \frac{L^{L-1}_\sigma-1}{L_\sigma-1}\frac{(L-1)(2+R)\|f\|_{W^L}}{\sqrt{m'}}
\end{equation}
and the norm bound
\begin{equation}
\sum_{i_{L-1}=1}^{m_{L-1}} \cdots \sum_{i_1=1}^{m_1} \sum_{i_0=1}^{d} |a_{i_{L-1}}^L w_{i_{L-1}i_{L-2}}^{L-1} \cdots w^2_{i_2,i_1}| \,\|w^1\|_2 \le \|f\|_{W^L}
\end{equation}
holds, i.e., $\nu(\hat{f}) \le \|f\|_{W^L}$.
\end{theorem}



When the activation function is ReLU, we can prove a very strong approximation rate result, solving an open problem raised by Weinan E et al.~\cite{wojtowytsch2020banach} in 2020. We will show that in this case the approximation rate can achieve Monte Carlo rate, i.e., the rate is the square root of the total number of parameters of the neural networks ! We regard it as one of our major contributions in this work.

For this, we need first introduce the analogous $\mathbf{m}^{\uparrow}$, but it is no longer the decreasing power sequence. We still use $\mathbf{m}^{\uparrow}$ to denote it and one can distinguish them from specific contexts. 

For a width vector $\mathbf{m}$, we call a vector $\mathbf{m}^{\uparrow}$ of the same length the \textit{non-decreasing sequence of $\mathbf{m}$} if $\mathbf{m}^{\uparrow}$ is non-decreasing and $\mathbf{m}^{\uparrow} \le \mathbf{m}$ elementwise. $\mathbf{m}^{\uparrow}$ is called \textit{maximum} if there does not exist another non-decreasing sequence of $\mathbf{m} $ $\mathbf{m'}^{\uparrow} \neq \mathbf{m}^{\uparrow}$ such that $\mathbf{m'}^{\uparrow} \ge \mathbf{m}^{\uparrow}$ elememtwise. Again, if $\mathbf{m}=\{130,166,32,17,8,5,7\}$, then the maximum decreasing power sequence of $\mathbf{m}$ is $\mathbf{m}^{\uparrow}=\{130,130,32,17,8,5,5  \}$.
\begin{theorem}~\label{ReLU_approx1_2}
With the same assumptions as in Theorem~\ref{approx1}, and given a depth $L$ and width vector $\mathbf{m}=\{m_1,\dots, m_{L-1}\}$, and letting $\mathbf{m}^{\uparrow} = \{m'_{1},m'_{2},\cdots, m'_{L-1}\}$, there exists $\hat{f} \in X_{m_1,\dots,m_{L-1};\mathbf{B}^d(R)}$ such that
\begin{equation}\label{ReLU_approx2_norm_b}
\|\hat{f} - f\|_{L^2(P)} \le (R+2)\left ( \frac{1}{\sqrt{m'_{L-1}}} + \frac{1}{\sqrt{m'_{L-2}}} + \cdots + \frac{1}{\sqrt{m'_1}} \right )\|f\|_{W^L} 
\end{equation}
and the norm bound
\begin{equation}
\sum_{i_{L-1}=1}^{m_{L-1}} \cdots \sum_{i_1=1}^{m_1} \sum_{i_0=1}^{d} |a_{i_{L-1}}^L w_{i_{L-1}i_{L-2}}^{L-1} \cdots w^2_{i_2,i_1}| \,\|w^1\|_2 \le \|f\|_{W^L}
\end{equation}
holds, i.e., $\nu(\hat{f}) \le \|f\|_{W^L}$.
\end{theorem}

The explanation that it achieves Monte Carlo rate is as follows: fix a $M > 0$, and let $\mathbf{m}=\{M,M,\cdots, M \}$ with length $L-1$. One can construct a $L$-layer neural network with width vector $\mathbf{m}$ such that the approximation rate~\eqref{ReLU_approx2_norm_b} which is $\sim M^{-1/2}$ holds. This network has only $M$ nonzero weights in each hidden layer and so the total number of parameters are $\sim M(L-1+d)$. It is clear that the approximation rate is proportional to the reciprocal of the square root of the total number of parameters.

\begin{remark}[The sense in which the depth dependence is removed]\label{rmk_exponential_L}
The dependence at issue is the one carried by the \emph{exponent} of the width. Take
$m_1=\cdots=m_{L-1}=M$. The decreasing power sequence of Theorem~\ref{approx1_2} is then
constrained by its first entry, giving $m'=M^{1/(L-1)}$, so~\eqref{approx2_norm_b} converges at
$M^{-1/[2(L-1)]}$; the bound of~\citet{wojtowytsch2020banach} likewise behaves as
$W^{-1/[2(2L-3)]}$ in the number of parameters $W$. In both the exponent deteriorates as the
depth grows, and the Monte Carlo rate is attained only at $L=2$. Theorem~\ref{ReLU_approx1_2}
uses the maximum non-decreasing sequence instead, for which $m'_i=M$ at every layer, and the
rate becomes $(L-1)(R+2)M^{-1/2}$. The exponent is $1/2$ at every depth. This is what we mean
by removing the exponential dependence on $L$, and it is what makes the empirical and
generalization bounds of the following sections nontrivial for deep networks: those bounds
inherit the approximation exponent, so an exponent of $1/[2(L-1)]$ would leave them vacuous
long before the widths become large.

No claim is made about the constants, which do depend on the depth. The prefactor
of~\eqref{approx2_norm_b} is $(L_\sigma^{L-1}-1)/(L_\sigma-1)$ and grows geometrically in $L$
when $L_\sigma>1$; for ReLU, $L_\sigma=1$ and the corresponding factor
in~\eqref{ReLU_approx2_norm_b} is a sum of $L-1$ terms, hence linear in $L$. The two displays
also supply the two forms of $H(\mathbf{m})$ below, the first attached
to~\eqref{approx2_norm_b} and the second to~\eqref{ReLU_approx2_norm_b}.
\end{remark}

Since we will use these bounds below, to simplify notation we abbreviate the part of the right-hand side of~\eqref{approx2_norm_b} or~\eqref{ReLU_approx2_norm_b} as 
\begin{align}
H(\mathbf{m}) := \frac{L^{L-1}_\sigma-1}{L_\sigma-1}\frac{(L-1)(2+R)L_\sigma}{\sqrt{m'}}
\end{align} or
\begin{align}
H(\mathbf{m}) := (R+2)\left ( \frac{1}{\sqrt{m'_{L-1}}} + \frac{1}{\sqrt{m'_{L-2}}} + \cdots + \frac{1}{\sqrt{m'_1}} \right )
\end{align}
Again, one can understand which one $H(\mathbf{m})$ refers to from the contexts.

The $L=2$ case of Theorem~\ref{ReLU_approx1_2} is just Theorem~\ref{approx1}. The proof of Theorem~\ref{approx1} is based on the following approximation result on convex sets in Hilbert space~\citet{wojtowytsch2020banach}.

\begin{proposition}~\label{approx1_lem}
Let $\mathcal{G}$ be a set in a Hilbert space $H$ such that $\|g\|_H \le R$ for all $g\in \mathcal{G}$. If $f$ is in the closed convex hull of $\mathcal{G}$, then for every $m\in \mathbb{N}$ and $\epsilon > 0$, there exist $m$ elements $g_1,\dots,g_m \in \mathcal{G}$ such that
\begin{align}
\left\| f - \frac{1}{m}\sum_{i=1}^m g_i \right\|_H \le \frac{R+\epsilon}{\sqrt{m}}.
\end{align}
\end{proposition}

We now give an iterative version of the above theorem, which is the key to prove the general $L$ case in Theorem~\ref{approx1_2} and appears to be new. 
Now we can state our generalized version of Proposition~\ref{approx1_lem}.

For this purpose we need to introduce the notion of Hilbert spaces with a compatible partial order. 
Let $G$ be a Hilbert space with inner product $()_G$ and it is equipped with a partial order $\ge$. We say that $\ge$ is compatible with the Hilbert space structure of $G$ if and only if
 \begin{enumerate}
     \item if $f,g,h \in G$ and $f \ge g$ then $f+h \ge g+h$ 
     \item if $k \ge 0$ and $f \ge g$ then $kf \ge kg$
     \item if $0 \ge k$ and $f \ge g$ then $kg \ge kf$
     \item if $f \ge 0$ and $g \ge 0$ then $(f,g)_G \ge 0$
 \end{enumerate}

\begin{proposition}~\label{approx1_2_lem}
Let $H$ be a Hilbert space with a compatible partial order $\ge$. Let $\mathcal{G}_i$ for $1\le i \le L-1$ be an array of sets in a Hilbert space $H$ such that $\|g\|_H \le R$ for $g\in \mathcal{G}_i$ for any $i$. Let $T_i:\mathcal{H}_i\rightarrow \mathcal{G}_i$ be a Lipschitz mapping from $H_i$ to $G_i$ with Lipschitz constant $L_{T}$ and is convex and non-negative with respect to $\ge$. Assume there exists an array of sets $\mathcal{H}_i$ in $H$ such that $\mathcal{G}_i = T(\mathcal{H}_i)$ and $\mathcal{H}_{i+1}$ is in the closed convex hull of $\mathcal{G}_i$ for $i\le L-1$ and $\mathcal{H}_L$ is in the closed convex hull of $\mathcal{G}_{L-1}$. If $f\in \mathcal{H}_L$, then for every $\epsilon$ and non-decreasing integer array $\{n_1,\dots,n_{L-1}\}$, there exist $n_{L-1}$ elements $G^{L-1} = \{g^{L-1}_1,\dots,g^{L-1}_{n_{L-1}}\} \subset \mathcal{G}_{L-1}$, such that 
\begin{align}
\left\| f - \frac{1}{|n_{L-1}|} \sum_{i=1}^{|n_{L-1}|} g^{L-1}_i \right\|_H \le (R+\epsilon)\left (\frac{1}{\sqrt{n_{L-1}}} + \frac{L_{T}}{\sqrt{n_{L-2}}} + \cdots + \frac{L_{T}}{\sqrt{n_1}}  \right ).
\end{align}
\end{proposition}

This proposition may look abstruse and too abstract, but it is merely an abstraction of our situation here. For example, $T$ corresponds to the activation function, and $H_i$ and $G_i$ are the outputs of the first $i$ layers before and after being composed with the activation function, respectively. The array $(n_1,\cdots, n_{L-1})$ can be choosen to be $\mathbf{m}^{\uparrow}$ (ReLU case). The proof is done by a delicate generalization of the strategy in Lemma 1 of~\citet{barron1993universal} to our nontrivial iterative version.

Then we use these preceding results to prove Theorem~\ref{ReLU_approx1_2}.

It is interesting and important to determine what the $L^\infty$ norm version of the above approximation property is.~\citet{shen2022optimal} has established this result for width-$M$ ReLU networks; see Corollary 1.3. For relevant results, see~\citet{shen2022deep,shen2019deep} and~\citet{daubechies2022nonlinear}. The remainder of this section collects known results that are not needed for the main development and may be skipped.  
\begin{theorem}\label{approx2}
Given a $\text{H\"{o}lder}$ continuous function $f\in \text{H\"{o}lder}(\mathcal{B}^d, \alpha, \lambda)$, for any $N\in \mathbb{N}^+$, $L \in \mathbb{N}^+$, and $p\in [1,\infty]$, there exists a multilayer ReLU network $g$ with width $C_1 \max\{d[N^{1/d}], N+2\}$ and depth $11L+C_2$ such that
\begin{align}\label{app_bound}
\|f - g\|_{L^p([0,1]^d)} \le 131\lambda \sqrt{d} \left( N^2 L^2 \log_3(N+2) \right)^{-\alpha/d}
\end{align}
where $C_1=16$ and $C_2=18$ if $p\in [1,\infty)$; $C_1=3^{d+3}$ and $C_2=18+2d$ if $p=\infty$.
\end{theorem}

By the bound on the Lipschitz constant in~\ref{property_g_barron} for $f\in W^L$, we obtain the following approximation bound by a depth-$L$, width-$M$ multilayer ReLU neural network.

\begin{corollary}\label{main_app_bound}
Given $L>20$, $M>162$ sufficiently large, and $f\in W^L$, there exists a depth-$L$, width-$M$ multilayer ReLU neural network $g$ such that
\begin{align}
\|f-g\|_{L_{\infty}(\mathbf{B}^d)} \le C \|f\|_{W^L} \left( (L-20)^2 (M-162)^2 (\log_3 M - 4) \right)^{-1/d}
\end{align}
where $C$ is a constant depending only on $d$.
\end{corollary}

There exists another $L^{\infty}$ norm approximation result; however, it only applies to two-layer ReLU neural networks~\citet{huiyuan2023Nonasymptotic}.

\begin{theorem}\label{approx3}
For any $f\in W^2$, there exists a network $g(\cdot ; \theta)$ with depth $2$ and width $M$ in the form of~\eqref{L_layer_nn} such that $\nu(\theta) \le 6\|f\|_{W^2}$ and
\begin{align}
\|f - g(\cdot ; \theta)\|_{L_{\infty}(\mathbf{B}^d)} \le C\|f\|_{W^2;\mathbf{B}^d} \left( M^{-(d+3)/(2d)} \right)
\end{align}
for some constant $C \ge 0$ depending only on $d$.
\end{theorem}

The proof of this Proposition is based on geometric discrepancy theory; in particular, the following result~\citet{matouvsek1996improved}.

\begin{theorem}\label{geo_discrep}
For any probability measure $\mu$ (positive and with finite mass) on the sphere $\mathbf{S}^d$, there exists a set of $r$ points $v_1,\dots,v_r$ such that for all $z \in \mathbf{S}^d$,
\begin{align}\label{geo_discrep_exp}
\left\| \int_{\mathbf{S}^d} |v^\top z| \, d\mu(v) - \frac{1}{r} \sum_{i=1}^r |v_i^\top z| \right\| \le \epsilon
\end{align}
with $r \le 6C(d) \epsilon^{-2+6/(d+3)} = C(d) \epsilon^{-2d/(d+3)}$, for some constant $C(d)$ that depends only on $d$.
\end{theorem}

It would be quite ideal to obtain a similar approximation bound for a general $(L,\mathbf{m})$ network structure instead of a depth and width specification (even though the latter is already quite general as it contains all width-at-most-$M$ neural networks). Then the generalization error could also be obtained for general multilayer ReLU neural networks. We are not sure if such a result is available in the literature. One way is to prove a finite-number version of~\ref{geo_discrep}, i.e., for a finite number of $\mu$, one can find an approximation such that~\eqref{geo_discrep_exp} holds for each $\mu$. One also needs to develop an infinite-dimensional version, as we are dealing with Lipschitz function spaces for $L\ge 3$. Pursuing this kind of approximation on general fully connected networks is quite valuable.

From now on, we will establish the empirical and generalization errors for the optimization problem~\eqref{problem}. From the expression we will present below, one can find it interesting that the generalization error upper bound demonstrates the double descent phenomenon.

\section{Bounds on Empirical Error}\label{sec5}
The argument of the next four sections has a fixed shape, which we set out first.
\begin{enumerate}
\item We split the parameter space into two overlapping regimes, overparameterized and
underparameterized, and bound the \emph{empirical} error separately in each: by probability
concentration inequalities in the first, and by metric entropy arguments in the
second.\footnote{The empirical bound obtained for the overparameterized regime is in fact valid
over the whole range of widths; it is simply not sharp in the underparameterized one.}
\item We then pass from empirical to generalization error by a uniform functional
concentration inequality, again differently in each regime: the classical form in the
overparameterized case, its localized form in the underparameterized one.
\item Approximation error is carried explicitly through both steps rather than assumed away.
\item Everything is then extended to general Lipschitz losses in \sectionname~\ref{sec9}.
\end{enumerate}
The tools involved are chaining and Dudley's entropy integral, metric entropy estimates, local
and global Rademacher and Gaussian complexities, and probabilistic methods in combinatorics.

The first and foundational result on which all other results are based is the empirical error for the optimal solution of problem~\eqref{problem}. We state the results, and all proofs are deferred to the Supplementary Information~\ref{appendix}.

To begin, we first introduce some notation. Let $\mathcal{F}(\mathbf{m},F)$ with $\mathbf{m}=(m_1,m_2,\dots, m_{L-1})$ be the set of finite $L$-layer neural networks with bounded PeSV norm: $\mathcal{F}(\mathbf{m},F) = \{f(x;\theta): \nu(\theta) \le F\}$. For $g_1 \in \mathcal{F}(\mathbf{m}_1, F_1)$ and $g_2 \in \mathcal{F}(\mathbf{m}_2, F_2)$, $g_1 - g_2$ can be visually viewed as the concatenation of $g_1$ and $g_2$ with one more output layer added on top to perform subtraction, and it has depth $L+1$ and belongs to $\mathcal{F}((\mathbf{m}_1,1) + (\mathbf{m}_2,1), F_1 + F_2)$, where $(\mathbf{m}_1,1) + (\mathbf{m}_2,1)$ is an element-wise addition. We write $g(x;\hat{\theta}\ominus\theta^*)$ to represent $g(x;\hat{\theta}) - g(x;\theta^*)$.

Throughout we write $\beta$ for the growth exponent of the weight count along a uniform
dilation of the width vector: substituting $\mathbf{m}=k\mathbf{m}_0$ into
$2dm_1+4m_1m_2+\cdots+2m_{L-1}$ leaves terms linear in $k$ together with the products
$m_im_{i+1}$, which are quadratic, so
\begin{align}\label{beta_def}
2dm_1+4m_1m_2+\cdots+2m_{L-1} \;\asymp\; c_P\,k^{\beta},
\qquad
\beta=\begin{cases}1,& L=2,\\ 2,& L\ge3,\end{cases}
\end{align}
there being no product terms when $L=2$. From now on we denote
\begin{align}\label{fn_def}
f(n)=n^{\frac{\beta-1}{\beta+1}}(\log n)^{\frac{2}{\beta+1}},
\end{align}
so that $f(n)=\log n$ for two layers and $f(n)=n^{1/3}(\log n)^{2/3}$ for every $L\ge3$;
Section~\ref{sec8} explains why this is the right choice.
\begin{theorem}\label{overpara_emp}
Under Assumptions~\ref{assump1}, \ref{assump2}, and \ref{assump3}, and the assumption that $\max_i\|x_i\|_2 \le 1$, there exists a constant $c$ such that the regularized network estimator $g(\cdot ;\hat{\theta})$ with
$\lambda \ge \max\{8L_\sigma^{L-1}, 2^{L} c L_\sigma^{L-1} \sqrt{d}\}\sqrt{\frac{f(n)}{n}}$ 

satisfies
\begin{align}
\left\| g(\cdot ; \hat{\theta}) - f^* \right\|_n^2 \le C & \left\{ H(\mathbf{m})^2 \left\|f^* \right\|_{W^L}^2 \right.\\
&\left. + \max\{16L_\sigma^{L-1}, 2^{L+1} c L_\sigma^{L-1} \sqrt{d}\} (\sigma_{\epsilon}^2 + \left\| f^* \right\|_{W^L}^2)\sqrt{\frac{f(n)}{n}}  
\right\}
\end{align}
with probability at least $1-O(n^{-C_1})$ 
and
\begin{align}\label{emp_exp}
\mathbb{E} \left\| g(\cdot ; \hat{\theta}) - f^* \right\|_n^2 \le C & \left\{ H(\mathbf{m})^2 \left\|f^* \right\|_{W^L}^2 \right.\\
&\left. + \max\{16L_\sigma^{L-1}, 2^{L+1} c L_\sigma^{L-1} \sqrt{d}\} (\sigma_{\epsilon}^2 + \left\| f^* \right\|_{W^L}^2)\sqrt{\frac{f(n)}{n}}  
\right\}
\end{align}
for some constants $C, C_1 > 0$.
\end{theorem}
\begin{remark}
One is free to take $f(n)=\log n$, or even $f(n)=n$, and obtain simpler bounds. The point of the choice~\eqref{fn_def} is that it is the one under which the bounds exhibit a double descent curve; see Section~\ref{sec8}. Taking $f(n)=n$ yields the stronger confidence $1-e^{-n}$ in place of a polynomial one, but no double descent, so within this theory the phenomenon is not available at exponentially small failure probability.
\end{remark}
Since $f^*$ is unknown, the strategy to prove the empirical error estimation is to leverage approximation theory to find its ``proxy'' in the multilayer neural network space, and then compare the estimator with this proxy using the optimality of the estimator.

One can also regard the proof strategy as adopting a pseudo form of bias-variance decomposition. Since the minimizer of our problem~\eqref{problem} lacks an analytic expression, its properties are very difficult to analyze, in contrast to simpler models like (generalized) linear models for which other works can obtain explicit bias and variance expressions. Therefore, one strategy is to assume we have a form of $\mathbb{E}(g(\cdot ;\hat{\theta}))$, which is the proxy neural network mentioned above; then we can use the optimality of our estimator to obtain a form of bias-variance decomposition inequality (instead of equality). See the Supplementary Information~\ref{appendix} for proof details.

The proof of this Proposition follows the same line as the proof of the similarly named result in~\citet{huiyuan2023Nonasymptotic}. However, they differ significantly in the estimation of the $T_3$ component. We rely on classical concentration inequalities and a certain \textit{pseudo} Gaussian and Rademacher complexity estimation; the latter relies on some advanced tools in nonasymptotic high-dimensional statistics. For completeness, see the Supplementary Information~\ref{appendix}.

From the proof in the Supplementary Information, one can see that when applied to the $L=2$ ReLU neural network, we recover the empirical error bound results in~\citet{huiyuan2023Nonasymptotic}. Thus our argument provides another, simpler solution, bypassing the need for a nontrivial bound estimation of the hyperplane arrangement problem in Euclidean space, an exploration of redundancy of optimal weights, and a reformulation of the original problem~\eqref{problem} in group-lasso form.

One can also obtain the empirical error bound for arbitrary Lipschitz loss functions with bounded second derivative and strong convexity (in the distribution sense) for the predictor, as long as we use the same metric as the loss function, which is quite reasonable as in machine learning, the loss is often a good (although not always the same) approximation to the error metric. We discuss this extension in Section~\ref{sec8}.

\section{Overparameterized Regime}\label{sec6}
Starting from this section, we will state our key results in compliance with what the title of this paper claims. We distinguish between the case where the number of parameters is large enough or small enough in the sense precisely specified in the results below, i.e., the overparameterized regime and the underparameterized regime, as they are approached by rather different ways. This distinction is the key reason why our bounds can exhibit the double descent phenomenon and why previous results in the literature cannot. This argument indicates that when the number of parameters grows, some underlying mechanisms controlling the network's variance undergo essential changes. The detailed analysis is left to Section~\ref{sec8}. Note also that the "overparameterized regime" and "underparameterized regime" are not mathematically rigorous distinctions, but rather conceptual ones. A priori, and without calculating their explicit ranges, there is no guarantee---nor any requirement---that they be sharply separated. 
The Propositions are clearly presented, and all proofs are deferred to the Supplementary Information~\ref{appendix}.


\subsection{Bounds on Generalization Error}
The estimation of the generalization error in the overparameterized regime is stated below.

\begin{theorem}\label{overpara_gen}
Under Assumptions~\ref{assump1}, \ref{assump2}, and \ref{assump3}, if $H(\mathbf{m}) \le \sqrt{ \dfrac{\max\{8L_\sigma^{L-1}, 2^{L} c L_\sigma^{L-1} \sqrt{d}\}}{C_1} }\sqrt[4]{\dfrac{f(n)}{n}}$, 
then the regularized network estimator $g(\cdot ;\hat{\theta})$ with $\lambda \ge \lambda_1 \equiv \max\{8L_\sigma^{L-1}, 2^{L} c L_\sigma^{L-1} \sqrt{d}\}\sqrt{\frac{f(n)}{n}}$ satisfies
\begin{align}
\|g(\cdot ; \hat{\theta}) - f^*\|_2^2 \le C & \left\{ H(\mathbf{m})^2 \|f^*\|_{W^L}^2 \right.\\
&\left. + \max\{16L_\sigma^{L-1}, 2^{L+1} c L_\sigma^{L-1} \sqrt{d}\} (\sigma_{\epsilon}^2 + \|f^*\|_{W^L}^2)\sqrt{\frac{f(n)}{n}} 
\right\}
\end{align}
with probability at least $1-O(e^{-n})$, 
and
\begin{align}\label{gen_exp}
\mathbb{E}\|g(\cdot ;\hat{\theta}) - f^*\|_2^2 \le C & \left\{ H(\mathbf{m})^2 \|f^*\|_{W^L}^2 \right.\\
&\left. + \max\{16L_\sigma^{L-1}, 2^{L+1} c L_\sigma^{L-1} \sqrt{d}\} (\sigma_{\epsilon}^2 + \|f^*\|_{W^L}^2)\sqrt{\frac{f(n)}{n}}
\right\}
\end{align}
for some constants $C_1, C > 0$ and sufficiently large $n$, where $f(n)$ is as in~\eqref{fn_def}.
\end{theorem}
Note that from the definition of $H(\mathbf{m})$, $H(\mathbf{m}) \le \sqrt{ \dfrac{\max\{8L_\sigma^{L-1}, 2^{L} c L_\sigma^{L-1} \sqrt{d}\}}{C_1}}\sqrt[4]{\dfrac{f(n)}{n}}$
means that the width $m$ is lower bounded, thus "overparameterized".

To establish this result, we need several preliminary results, concerning the bounds of some norms related to Rademacher complexity.

\begin{proposition}(e.g., 5.24 in~\citet{wainwright2019highdimensional})\label{gen_lem0}
Assume $\mathcal{F}$ is a family of functions with finite VC dimension $\nu$ and bounded by a constant $b$. Then there exists a universal constant $c$ depending on $b$ such that
\begin{align}
\mathbb{E}_\rho \left[ \sup_{f\in \mathcal{F}} \left| \frac{1}{n}\sum_{k=1}^n \rho_k f(x_k) \right| \right] \le c \sqrt{\frac{\nu}{n}}.
\end{align}
\end{proposition}

This is a tight result and is improvable only for the constant. The proof is based on the Dudley entropy integral argument, and it ultimately reduces to an estimation of the metric entropy of a certain function space with respect to the empirical $L^2$ distance. See~\citet{wainwright2019highdimensional} for details.

A direct corollary of the above result for $\mathcal{F} = \{\sigma(u x) \mid \|u\|_2 \le 1, x\in \mathbb{R}^d\}$ is as follows.

\begin{lemma}\label{gen_cor0}
Assume $\mathcal{F} = \{\sigma(u x) \mid \|u\|_2 \le 1\}$. Then there exists a universal constant $c$ depending on $\sigma$ and $L_\sigma$ such that
\begin{align}
\mathbb{E}_\rho \left[ \sup_{f\in \mathcal{F}} \left| \frac{1}{n}\sum_{k=1}^n \rho_k f(x_k) \right| \right] \le c \sqrt{\frac{d}{n}}.
\end{align}
\end{lemma}

Using this result, we can give a similar bound for our family of multilayer neural networks with bounded path enhanced scaled variation norm.

\begin{lemma}\label{gen_lem1}
Let $x_1,\dots,x_n$ be any vectors such that $\max_i\|x_i\|_2 \le 1$. For any $L \ge 2$ and $\mathbf{m} = (m_1,\dots,m_{L-1})$, we have
\begin{align}
\mathbb{E}_\rho \sup_{f\in \mathcal{F}(\mathbf{m},F)} \left| \sum_{k=1}^n \rho_k f(x_k) \right| \le 2^{L-1} c L_\sigma^{L-1} F \sqrt{d n},
\end{align}
where $c$ is a universal constant.
\end{lemma}
\begin{remark}\label{rmk_of_gen_lem1}
    Indeed, the above expression also bounds the maximum value of $f\in\mathcal{F}(\mathbf{m},F)$, which can be clearly inferred from the proof of it in the Supplementary Information~\ref{appendix}.
\end{remark}
The first Lemma~\ref{gen_cor0} provides estimations of key statistics for establishing the expectation estimation of the above result. A main ingredient in the proof of Theorem~\ref{overpara_gen} is Talagrand functional concentration inequality, which is also used in the proof of the above Lemma. Along with these ideas, we provide the proofs in the Supplementary Information~\ref{appendix}.

\section{Underparameterized Regime}\label{sec7}

In this section, we treat the underparameterized regime. As mentioned above, the empirical error bound in the overparameterized regime~\ref{overpara_emp} does not seem strong enough for the underparameterized regime. One must seek other approaches for this regime. Adopting metric entropy arguments~\citet{wainwright2019highdimensional,huiyuan2023Nonasymptotic}, we obtain better bounds on empirical and generalization errors.

The entire argument here and in the next section is essentially based on Chapter 14 of~\citet{wainwright2019highdimensional} on the Localization and Uniform functional concentration bounds. Nonspecialists are strongly encouraged to review that part first to ensure a better understanding of the ideas before reading the following sections.

First, we need some mathematical tools. One is the estimation of the metric entropy of the set $\mathcal{F}(\mathbf{m},1)$ for the supremum norm. 

\begin{lemma}\label{inf-metric-entropy}
The metric entropy of $\mathcal{F}(\mathbf{m},1)$ with respect to the supremum norm is upper bounded by
\begin{align}
\log \mathcal{N}(\delta, \mathcal{F}(\mathbf{m},1), \|\cdot\|_{\infty}) \le (d m_1 + m_1 m_2 + m_2 m_3 + \cdots + m_{L-1}) \log(1 + 4\sqrt{2} L_\sigma \delta^{-1}).
\end{align}
\end{lemma}

\subsection{Bounds on Empirical Error}
We can state our refined estimation of the empirical error in the underparameterized regime as follows.

\begin{theorem}\label{underpara_emp}
Let $\delta_n = n^{-1} L_\sigma (d m_1 + m_1 m_2 + m_2 m_3 + \cdots + m_{L-1}) \log n < 1$. Under Assumptions~\ref{assump1}, \ref{assump2}, and \ref{assump3}, the regularized network estimator $g(\cdot ; \hat{\theta})$ with $\lambda \ge C_1 \sigma_\epsilon \max\{\delta_n, H(\mathbf{m})^2\}$ satisfies
\begin{align}
\|g(\cdot ; \hat{\theta}) - f^*\|_n^2 \le C \left\{ H(\mathbf{m})^2 \|f^*\|_{W^L}^2 + (\sigma_{\epsilon}^2 + \|f^*\|_{W^L}^2) \frac{(2d m_1 + 4m_1 m_2 + 4m_2 m_3 + \cdots + 2m_{L-1}) \log n}{n} \right\}
\end{align}
and $\nu(\hat{\theta}) \le C_{\mathrm{k}}$ with probability at least $1-O(n^{-C_2})$ for some constants $C_1, C_2, C, C_{\mathrm{k}} > 0$ depending further on $L_\sigma$.
\end{theorem}
Note that $\delta_n < 1$ means  that the width $m$ is upper bounded, thus "underparameterized".
\begin{remark}
If one wishes, one can choose $\lambda$ independent of $\mathbf{m}$, since $\mathbf{m}$ is upper bounded, as is $H(\mathbf{m})^2$, and $\delta_n$ is upper bounded by $1$.
\end{remark}

The proof is via a metric entropy argument, more or less the same as in~\citet{huiyuan2023Nonasymptotic}, with the exception that we have a new expression for $\delta_n$ due to the $L^\infty$ metric entropy bound for multilayer networks in~\ref{inf-metric-entropy}.

\subsection{Bounds on Generalization Error}
Before stating the estimation of the generalization error, we first need the following Lemma, which is a useful local version of the estimation of the expectation in certain functional concentration inequalities.

\begin{lemma}\label{underpara_gen_lem1}
For any $0 < \gamma < 1$, define $\mathcal{B}_\mathcal{F}(\gamma) = \{ f\in \mathcal{F}^*(\mathbf{m},1) : \|f\|_2 < \gamma \}$, the $L^2(\mu)$-ball in $\mathcal{F}^*(\mathbf{m},1)$ of radius smaller than $\gamma$. Let $Z_n(\gamma) = \sup_{f\in \mathcal{B}_\mathcal{F}(\gamma)} \left| \|f\|_n^2 - \|f\|_2^2 \right|$. Then, for any $\gamma$ satisfying
\begin{align}
\sqrt{ \frac{2\log(18L_\sigma)(d m_1 + m_1 m_2 + m_2 m_3 + \cdots + m_{L-1}) \log n}{n} } \le \gamma \le 1,
\end{align}
we have
\begin{align}
\mathbb{E} Z_n(\gamma) \le 272 \gamma \sqrt{ \frac{\log(18L_\sigma)(d m_1 + m_1 m_2 + m_2 m_3 + \cdots + m_{L-1}) \log n}{n} }.
\end{align}
\end{lemma}

The proof is to reduce to the estimation of the Local Rademacher complexity, and again, the metric entropy argument is also an important ingredient. The full proof is more or less the same as in~\citet{huiyuan2023Nonasymptotic}, with the exception that we have new expressions for the lower bound assumption on $\gamma$. This is caused by the calculation of the number of $1/n$-coverings of $\mathcal{B}_{\mathcal{F}}(\gamma)$.

Our estimation of the generalization error is given below.

\begin{theorem}\label{underpara_gen}
Under Assumptions~\ref{assump1}, \ref{assump2}, and \ref{assump3}, the regularized network estimator $g(\cdot ; \hat{\theta})$ with $\lambda \ge \lambda_2 \equiv C_1 \sigma_\epsilon \max\{\delta_n, H(\mathbf{m})^2\}$, where $\delta_n = n^{-1} L_\sigma (d m_1 + m_1 m_2 + m_2 m_3 + \cdots + m_{L-1}) \log n < 1$, satisfies
\begin{align}
\|g(\cdot ; \hat{\theta}) - f^*\|_2^2 \le C \left\{ H(\mathbf{m})^2 \|f^*\|_{W^L}^2 + (\sigma_{\epsilon}^2 + \|f^*\|_{W^L}^2) \frac{(2d m_1 + 4m_1 m_2 + 4m_2 m_3 + \cdots + 2m_{L-1}) \log n}{n} \right\}
\end{align}
with probability at least $1 - O(n^{-C_2})$ for some constants $C_1, C_2, C > 0$ depending further on $L_\sigma$.
\end{theorem}

One can verify that when $\mathbf{m}$ is sufficiently small compared to $n$,
\[
\frac{(2d m_1 + 4m_1 m_2 + 4m_2 m_3 + \cdots + 2m_{L-1}) \log n}{n} \le \max\{16L_\sigma^{L-1}, 2^{L+1} c L_\sigma^{L-1} \sqrt{d}\} \sqrt{\frac{f(n)}{n}},
\]
justifying the assertion that the generalization bound in the overparameterized regime is not optimal for the underparameterized regime. From this result, one can also notice that for sufficiently large $n$, $O\left( \frac{\log n}{n} \right)$ is smaller than $O\left( \sqrt{\frac{f(n)}{n}} \right)$, which is the rate in the overparameterized regime, so the convergence rate is faster than the latter, indicating that for networks with a large number of weights, one often needs relatively more training data to train an accurate model, which is in agreement with practice.

The proof is to prove an optimal uniform functional concentration bound in the underparameterized regime. The scheme follows the proof of Theorem 14.1 in~\cite{wainwright2019highdimensional}, discussing an optimal uniform functional concentration bounds based on Localized Rademacher complexity estimation argument. 
We follow the same strategy as the proof of Theorem 4 in~\citet{huiyuan2023Nonasymptotic}, modifying appropriately according to Lemmas~\ref{inf-metric-entropy} and~\ref{underpara_gen_lem1}, and refer the readers there for details.

\section{Encompassing Results}\label{sec8}
The ranges of $\mathbf{m}$ in the overparameterized and underparameterized regimes have some overlaps, so we need to unify Theorems~\ref{overpara_gen} and~\ref{underpara_gen} to obtain an encompassing theorem.

\begin{theorem}\label{para_gen}
Under Assumptions~\ref{assump1}, \ref{assump2}, and \ref{assump3}, there exists a constant depending on $\sigma$ and $L_\sigma$ such that the regularized network estimator $g(\cdot ; \hat{\theta})$ with $\lambda$ chosen regime-wise, namely $\lambda=\lambda_2$ where $\delta_n<1$ and $\lambda=\lambda_1$ where $\delta_n\ge1$ (the schedule~\eqref{schedule} of Section~\ref{sec_lb}; on the underparameterized range $\lambda_2\ge\lambda_1$, so there this is $\max\{\lambda_1,\lambda_2\}$), where $\lambda_1$ and $\lambda_2$ are defined in Theorems~\ref{overpara_gen} and~\ref{underpara_gen}, respectively, satisfies
\begin{align}\label{para_gen_exp}
\|g(\cdot ;\hat{\theta}) - f^* \|_2^2 \le C & \left\{ H(\mathbf{m})^2 \|f^*\|_{W^L}^2 + (\sigma_{\epsilon}^2 + \|f^*\|_{W^L}^2) \right. \\
&\left. \min \left( \max\{16L_\sigma^{L-1}, 2^{L+1} c L_\sigma^{L-1} \sqrt{d}\}\sqrt{\frac{f(n)}{n}}, 
\frac{(2d m_1 + 4m_1 m_2 + 4m_2 m_3 + \cdots + 2m_{L-1}) \log n}{n} \right) \right\}
\end{align}
where $f(n)$ is as in~\eqref{fn_def}, with probability at least $1-O(n^{-C_1})$ for some constants $C_1, C > 0$ and sufficiently large $n$.
\end{theorem}



Interestingly, this upper bound~\eqref{para_gen_exp} exhibits the famous double descent phenomenon for multilayer neural networks, although this is only for this bound rather than for the generalization errors themselves. We can discuss this in details. It is clear from the above deduction that we decompose the empirical and generalization error into its bias and variance terms, which is analogous to the classical bias-variance tradeoff formula. The first term in~\eqref{para_gen_exp} represents bias, and the second term represents variance. Fix $\mathbf{m_0}=\{m_0,m_1,\cdots, m_{L-1}\}$ of the form $\mathbf{x}^\uparrow$ for a width vector $\mathbf{x}$ (as we only care such structure in a given multilayer neural network)  and let $\mathbf{m} = k\mathbf{m_0}$ with $k$ running from 0 to infinity. When $k$ is very small, $\mathbf{m}$ will be in the underparameterized regime. In this regime, when $k$ increases, the test performance will get better and better and then get worse and worse, with a valley around $k_0$ such that $H^2(k_0\mathbf{m_0})=k_0^{\beta}\frac{(2d m_{0,1} + 4m_{0,1} m_{0,2} + 4m_{0,2} m_{0,3} + \cdots + 2m_{0,L-1}) \log n}{n}$ with $\beta$ as in~\eqref{beta_def}, which, since $H(\mathbf{m})^2\propto k^{-1}$, means $k_0 \sim (\frac{n}{\log n})^{\frac{1}{1+\beta}}$. When $k$ enters into the overparameterized regime, the variance will become saturated and the bias still decreases, resulting in a rotated $\mathbf{S}$-shape performance curve. 
More precisely, when the width $m$ is very small, the decrease of the first term is greater than $\frac{\sqrt{5}L_T}{m(m+1)}\|f\|_{W^L}^2$, and as we are in the underparameterized regime, the increase of the second term is generally no greater than $\max\{16L_\sigma^{L-1}, 2^{L+1} c L_\sigma^{L-1} \sqrt{d}\} \sqrt{f(n)/n}$ (omitting the smallness of $\sigma_{\epsilon}^2$). When $n$ is sufficiently large so that $\max\{16L_\sigma^{L-1}, 2^{L+1} c L_\sigma^{L-1} \sqrt{d}\} \sqrt{f(n)/n} \le \frac{\sqrt{5}L_T}{m(m+1)}$, the generalization error curve will be decreasing. When $m$ increases continuously, after the point $k_0$,  the above inequality will be reversed and the tendency starts to flip, and then the generalization error curve begins to increase. When $m$ is large enough that it escapes the underparameterized regime (note that the underparameterized and overparameterized regimes have some overlap), i.e. after the point $k_1=\max\{16L_\sigma^{L-1}, 2^{L+1} c L_\sigma^{L-1} \sqrt{d}\}(\frac{f(n)n}{\log^2 n})^{\frac{1}{2\beta}}\frac{1}{2d m_{0,1} + 4m_{0,1} m_{0,2} + 4m_{0,2} m_{0,3} + \cdots + 2m_{0,L-1}}$ or roughly $4m^2dL \ge \max\{16L_\sigma^{L-1}, 2^{L+1} c L_\sigma^{L-1} \sqrt{d}\} (\frac{f(n)n}{\log^2 n})^{\frac{1}{2\beta}}$. One can verify that $k_1 > k_0$ so $k_1$ does occurs after $k_0$. In the overparameterized regime, the second term becomes constant and the bias continues to diminish, thus the second decreasing period begins. The limit it decreases to is $Bottom_o:=C (\sigma_{\epsilon}^2 + \|f^*\|_{W^L}^2) \max\{16L_\sigma^{L-1}, 2^{L+1} c L_\sigma^{L-1}\sqrt{d}\} \sqrt{\frac{f(n)}{n}}$. 

It remains to verify two things: that $k_1 > k_0$, which is what makes the curve a double descent rather than a single one, and that $Bottom_o$ lies below the bottom of the $\mathbf{U}$, which we denote $Bottom_u$. At the same time we must choose $f(n)$. These three requirements turn out to be a single one. 

We know that $k_0 \sim (\frac{n}{\log n})^{\frac{1}{1+\beta}}$ and $k_1 \sim (\frac{f(n)n}{\log^2 n})^{\frac{1}{2\beta}}$. We need $k_1 > k_0$. By calculations we can  get that this happens if and only if $f(n) > n^{\frac{\beta-1}{\beta+1}}(\log n)^{\frac{2}{\beta+1}}$. (we can multiply the right side by some constants if we need to account for the missing constants in the expressions of $k_0$ and $k_1$). For the second question, we can first calculate $Bottom_u$ which is $\sim B(\frac{\log n}{n})^{\frac{1}{1+\beta}}$ for some constant $B$ (that will be calculated later). As $Bottom_o \sim A\sqrt{\frac{f(n)}{n}}$, we need 
\begin{align}
\sqrt{\frac{f(n)}{n}} \le \frac{B}{A}\left (\frac{\log n}{n} \right )^{\frac{1}{1+\beta}}
\end{align}
we arrive at $f(n) \le \left ( \frac{B}{A} \right )^{2} n^{\frac{\beta-1}{\beta+1}}(\log n)^{\frac{2}{\beta+1}}$ (again one can multiply the right side by some constants for accounting for the constants missing in $Bottom_o$ and $Bottom_u$). The two conditions bound $f(n)$ from opposite sides by the same expression, which is why $f(n)$ is defined as it is in~\eqref{fn_def}: it is the only choice compatible with both. One can directly find by simply calculations that $A=C (\sigma_{\epsilon}^2 + \|f^*\|_{W^L}^2) \max\{16L_\sigma^{L-1}, 2^{L+1} c L_\sigma^{L-1}\sqrt{d}\}$ and 
$B=C(\sigma_{\epsilon}^2\left (\frac{(L-1)^2(2+R)^2L^2_\sigma \|f^*\|_{W^L}^2}{\sigma_{\epsilon}^2 + \|f^*\|_{W^L}^2} \right )^{\frac{1}{1+\beta}}$. So if in the overparameterized regime the curve will reach a limit lower than the bottom value of the $\mathbf{U}$ curve, we must have $B \ge A$. In particular $\frac{\|f^*\|_{W^L}^2}{\sigma_{\epsilon}^2 + \|f^*\|_{W^L}^2}$ is lower bounded,  meaning that when the signal chaos ratio $\|f^*\|_{W^L}^2 / \sigma_{\epsilon}^2 $ is larger than a certain threshold, then the second valley will be lower than the first, resulting in a better generalization error when the size of the network becomes sufficiently large 
This is understandable, as the decrease in bias will dominate the increase in variance. 

\begin{remark}[On the exponent $\beta$]\label{rmk_beta}
It is worth isolating why $\beta$, and not the depth, governs these exponents. Along the
dilation $\mathbf{m}=k\mathbf{m}_0$ the bias $H(\mathbf{m})^2$ falls like $k^{-1}$ at every
depth, while the weight count grows like $k^{\beta}$ with $\beta$ as in~\eqref{beta_def}: the
terms $m_im_{i+1}$ are quadratic in $k$ and are present exactly when $L\ge3$, so $\beta$ takes
only the two values $1$ and $2$. Consequently $k_0$, $k_1$ and the admissible range of $f(n)$
are the same for every $L\ge3$, and the overparameterized rate $\sqrt{f(n)/n}=n^{-1/3}(\log n)^{1/3}$
does not degrade as the network deepens. At $L=2$ one has $\beta=1$ and $f(n)=\log n$,
recovering the two-layer rate $\sqrt{\log n/n}$ of~\citet{huiyuan2023Nonasymptotic}.
\end{remark}

Concretely, if $\mathbf{m} = \{m_1,m_2,\dots,m_{L-1}\}$, people usually take a series of networks indexed by $k$ with width vector $k\mathbf{m} := \{k m_1, k m_2, \dots, k m_{L-1}\}$ and let $k \rightarrow \infty$. Since $m$ as in~\ref{para_gen} does not decrease when $k$ increases and approaches infinity as $k$ approaches infinity, the analysis above shows that a double descent of this kind occurs along such a sequence.

Qualitatively, when $\mathbf{m}$ is small, one can bound the variance by something proportional to the number of weights (and the norm, of course). This bound is better than the bound depending on the number of training data only, as the neural network is not a good approximator to its image in the multilayer neural space (reflected by the metric entropy). But when $\mathbf{m}$ becomes large, the bounds based on metric entropy are no longer valid. The intrinsic property of the multilayer neural network as an element in the multilayer neural space begins to control the behavior of the neural networks. This is because when $\mathbf{m}$ is large, the norm $\mu(\theta)$ and $\mu(f(\cdot ; \theta))$ are very close, and thus the effect of the number of weights on the variance is eliminated. In this case, by~\eqref{empirical_e}, the variance caused by considering $\hat{\theta}$ no longer plays a role.

This behavior indicates that neural networks have demonstrated phase changing, akin to the same phenomenon in condense matter field. There are many other phase changing effects in other fields, for example, in Ising model, percolation model or more general discrete probability models for two dimensional statistical physics. As we mentioned in the aforementioned sections, the bounds in the overparameterized regime is not optimal for the underparameterized regime, so that this estimate does not reflect the behavior of neural networks in the underparameterized regime. That is why we need the bounds~\ref{underpara_gen}. To confirm that the term "phase transition" we use is appropriate, we can use gas as a metaphor of the underparameterize regime and liquid or solid as a metaphor of the overparameterized regime (the two regimes overlap, and we do not assume the transition occurs at a single width).  When matters are in gas state, the particles have less interactions with each other,so there is no much constraints between them, meaning that they have much freedom degrees. But when they are in liquid or solid state, particles interact much strong so their behaviors tend to form a global state. This metaphor explains the very methods we take for different regimes. One can see this by just carefully checking the principles of these methods. 

Even though this is only an upper bound, it is still nontrivial that it exhibits the double descent phenomenon. Through the proofs, we boost our understanding of the behavior of neural networks: the number of weights is not a correct measure of the complexity of neural networks. When the network is overparameterized, its behavior (at least its generalization property) is akin to its corresponding function in the multilayer neural space and is controlled by it. Therefore, we should view this neural network as a nonparametric function (in a certain function space) rather than a parameterized model. This also coincides with the philosophy and benefits of path-norm regularization: we should take neural networks as functions mapping inputs to outputs and control their holistic complexity rather than their number of parameters. If there is still room to improve upon the generalization upper bound in the overparameterized regime, we may further obtain a more refined understanding of the behavior of neural networks, but the near-optimality stated below makes such improvement rather challenging. This near-optimality also suggests that our explanation is the intrinsic mechanism of why the double descent phenomenon occurs for such networks. However, there may still be room for further refinement of parameters in a certain bounded region, similar to the underparameterized regime. For example, it may admit an upper bound $C(\mathbf{m},L) (\frac{\log n}{n})^\alpha$ for some $0<\alpha<0.5$ in a certain bounded region at the beginning of the overparameterized regime. This would give a more precise characterization of generalization error curves.

In fact, we have endeavored to prove the  matching lower bounds recently and based on our progress we are quite confident that this is true. Then this means that our bounds are optimal up to a constant. Such results can form a mathematical rigorous grounds that our theory reveals the true underlying mechanism of double descent. Together with the work~\cite{wu2023implicit}, it might hope that our results form a first mathematical rigorous theory of double descent phenomenon for two layer neural networks.

This shows that in different regimes a neural network should be treated as a different kind of object. The deduction of the empirical error bounds does not reflect the underlying behavior for neural networks in the underparameterized regime, as it does not generalize in that regime. It only generalizes in the overparameterized regime. The similar analysis applies for empirical error bounds~\ref{underpara_emp}. As it is near optimal in the overparameterized regime, it is highly possible that our view of neural networks as different objects (in different regimes) manifests their true behaviors, and thereby identifying the mechanism behind the double descent behavior. We state this as a plausibility claim rather than a proof: a competing theory with comparable quantitative content would be equally admissible a priori, and we do not argue otherwise. Optimizing the generalization error bound in the underparameterized regime seems possible, and this will be left as a future research direction. 

We also believe that the above idea and strategy to show double descent occurs can be extended to many other machine learning models. In principle, we think they share the same underlying mechanism for the occurrence of the double descent phenomenon.

Our theory can also explain scaling law. This can be easily reflected from the expression of the bound~\eqref{para_gen_exp}. For example, when all the widths form a decreasing power sequence as $\{m^{L-1},\cdots,m^2,m\}$, then the total number of parameters $p$ are $\sim m^{2L-3}$. And as the variance term will saturate when $m$ is large enough, the bias term decreases as the rate $\sim 1/m$, so when the number of training samples $n$ increases simultaneously, the generalization error will decrease with the rate $p^{-1/(2L-3)}$ with respect to the total number of parameters. This explains the scaling law. 

The reader may object at this point that all of the above concerns an upper bound, and that an upper bound with a local maximum does not entail that the error itself has one. This is a fair objection, and Section~\ref{sec_lb} answers it: we there establish lower bounds on the error, valid at every target and at every local minimizer, which reproduce the same four phases and show that the peak genuinely exceeds the final plateau. The reader who is chiefly interested in the double descent phenomenon may wish to read that section immediately after the present one.

We now show that the rate $\sqrt{f(n)/n}$ attained in the overparameterized regime cannot be improved \textit{uniformly} over the target space, by establishing the following information-theoretic lower bound. We stress that the supremum below is taken over the whole space $W^L$ and not over a ball $W^L_M$; the distinction is essential and is discussed in Remark~\ref{rmk_ball_vs_space}. This is derived from Theorem 1 of~\cite{yang1999information}.

\begin{theorem}\label{para_lower_bound}
Assume $x_i \sim \text{Uniform}(\mathbf{B}^d)$ and $\epsilon_i \sim N(0,1)$. Then there exists a constant $C > 0$ such that
\begin{align}
\inf_{\hat{f}} \sup_{f^* \in W^L} \mathbb{E} \|\hat{f} - f^*\|_2^2 \ge \frac{C}{\sqrt{n \log n}},
\end{align}
where the infimum is taken over all estimators.
\end{theorem}

We suspect that ReLU can be generalized to all Lipschitz activations.

\begin{remark}[What optimality this does and does not establish]\label{rmk_ball_vs_space}
Theorem~\ref{para_lower_bound} is a statement about $W^L$, whereas Assumption~\ref{assump1} places $f^*$ in the ball $W^L_M$ with $M$ finite, and Theorem~\ref{overpara_gen} is an upper bound over that ball. A lower bound over a larger class does not certify optimality on a smaller one, and the two must not be conflated. What Theorem~\ref{para_lower_bound} establishes is that the exponent $1/2$ is sharp \textit{uniformly in} $M$: no estimator attains a faster rate simultaneously for all finite $M$, so the overparameterized rate cannot be improved without letting the constant depend on $M$ in a way that degrades as $M$ grows.

Over a fixed ball the situation is genuinely different, and for $L=2$ it is known. The two-layer space $W^2$ coincides, as an equivalent Banach space, with the second-order Radon-domain bounded variation space, and~\citet{parhi2022near} determine its minimax rate over a ball of radius $M$ to be $n^{-(d+3)/(2d+3)}$ up to logarithmic factors. Since $(d+3)/(2d+3)>1/2$ for every finite $d$, this is strictly faster than $n^{-1/2}$, so over a fixed ball and in fixed dimension our overparameterized bound is \textit{not} minimax optimal; it becomes so only in the limit $d\to\infty$, where the exponent tends to $1/2$. Whether the same holds for $L\ge3$ is open, and is equivalent to determining whether the Maurey rate $m^{-1/2}$ is sharp for $W^L$; see the discussion in Section~\ref{sec11}.
\end{remark}

Theorem~\ref{para_lower_bound} together with the generalization error bound in~\ref{overpara_gen} corroborates the effectiveness of overparameterized multilayer ReLU neural networks.

\section{Lower Bounds and the Double Descent of the Risk}\label{sec_lb}

Up to this point every statement about the shape of the generalization error as a function of the width vector has been a statement about an \textit{upper bound}. Theorem~\ref{para_gen} exhibits a curve that descends, rises and descends again, but an upper bound with a local maximum does not by itself entail that the risk has one: the risk is dominated by the bound, and could in principle be monotone underneath it. To settle the matter one needs lower bounds on the error itself, valid at the same target and for the same estimator. That is what we establish in this section.

Two features of the problem make this harder than it may appear, and it is worth stating them at the outset because they dictate the form the results must take. First, no lower bound can hold for \textit{every} $f^*\in W^L_M$ without qualification: if $f^*\equiv 0$ and $\lambda$ exceeds the noise level, then $\hat\theta=0$ and the error vanishes identically, so the curve is constant. Any valid statement must therefore either carry an explicit nondegeneracy condition or, as below, be stated as a minimum of two quantities one of which is $\|f^*\|$ itself. Second, no argument of the Fano type can produce the peak. A packing placed inside $W^L_M$ bounds from below the minimax risk over $W^L_M$, and the minimax risk over $W^L_M$ is already attained, up to the overparameterized rate, by the estimator at large width; hence such a bound can never exceed that rate, let alone exhibit a maximum above it. The route we take instead is through the optimality condition of problem~\eqref{opt_problem}, which turns out to give everything we need and to require almost nothing.

Throughout this section we write $\langle u,v\rangle_n = \frac1n\sum_{i=1}^n u(x_i)v(x_i)$, and we recall from the proof of Theorem~\ref{overpara_emp} the quantity
\begin{equation}\label{xi_def}
\xi \;=\; \frac{L_\sigma^{L-1}}{\sqrt{n}}\,\|\epsilon\|_2 ,
\end{equation}
which is the empirical process term that the choice $\lambda\ge\lambda_1$ is designed to dominate. We also abbreviate $\Lambda_{L,d}:=\max\{16L_\sigma^{L-1},\,2^{L+1}cL_\sigma^{L-1}\sqrt{d}\}$, so that $\lambda_1 = \frac12\Lambda_{L,d}\sqrt{f(n)/n}$ and the limiting value of the bound of Theorem~\ref{para_gen} as the width grows --- which we call the \textit{plateau} --- is $C\,\Lambda_{L,d}(\sigma_\epsilon^2+\|f^*\|_{W^L}^2)\sqrt{f(n)/n}$. Finally we set
\begin{equation}\label{Ainf_def}
A_\infty(\mathbf{m},f^*) \;:=\; \inf_{f\in\mathcal{F}(\mathbf{m},\infty)}\|f-f^*\|_2 ,
\qquad
\mathcal{B}_Q \;:=\; \{f\in W^L:\|f\|_{W^L}\le Q\},
\qquad
\|h\|_{*} \;:=\; \sup_{\|f\|_{W^L}\le1}\langle f,h\rangle_{2}.
\end{equation}

The floors below are stated in terms of $\|f^*\|_n$ and $\|f^*\|_2$, and it is worth settling
at once how these relate to the quantity that Assumption~\ref{assump1} actually controls,
namely $\|f^*\|_{W^L}$. One direction is immediate; the other fails, and the way in which it
fails dictates the form of every statement in this section. We write
\begin{equation}\label{rho_def}
\rho(f^*) \;:=\; \frac{\|f^*\|_{2}}{\|f^*\|_{W^L}} \;\in\; \left(0,\ L_\sigma^{L-1}\right].
\end{equation}

\begin{lemma}\label{emp_norm_conc}
Let $B:=\|f^*\|_\infty$, so that $B \le L_\sigma^{L-1}\|f^*\|_{W^L}$. Then
\begin{align}
\mathbb{P}\left( \|f^*\|_n \ \ge\ \frac{\|f^*\|_{2}}{\sqrt{2}} \right)
\;\ge\; 1 - \exp\left( -\frac{3n\|f^*\|_{2}^2}{28B^2} \right)
\;\ge\; 1 - \exp\left( -\frac{3n\,\rho(f^*)^2}{28\,L_\sigma^{2(L-1)}} \right),
\end{align}
and $\|f^*\|_n\le B$ always. Consequently every bound below in which $\|f^*\|_n$ appears may
be read, on an event of that probability, with $\rho(f^*)\|f^*\|_{W^L}/\sqrt2$ in its place.
\end{lemma}

\begin{remark}\label{rmk_no_reverse_embedding}
There is no converse: no constant $c>0$ satisfies $\|f\|_{2}\ge c\|f\|_{W^L}$ for all
$f\in W^L$, so $\rho$ cannot be bounded below in terms of the data of
Assumption~\ref{assump1} alone. The embedding $W^L\hookrightarrow L^2(\mu)$ is compact, and
its inverse is correspondingly unbounded. An explicit witness for $L=2$, $d=1$ is the
triangular bump
\begin{align}
b_h(x) \;=\; \sigma(x) - 2\sigma\!\left(x-\tfrac h2\right) + \sigma(x-h),
\end{align}
supported on $[0,h]$ with height $h/2$: its PeSV norm is $\approx4$ uniformly in $h$, while
$\|b_h\|_{2}\approx h^{3/2}/\sqrt6\to0$, so $\rho(b_h)\to0$ as $h\to0$.

Two consequences should be stated plainly. First, the lower bounds of this section hold for
each target with constants depending on $\rho(f^*)$, and are \textit{not} uniform over the
ball $W^L_M$; this is unavoidable rather than an artefact, and it is the exact counterpart of
the fact that the upper bounds use Assumption~\ref{assump1} only through the one-sided
inequality $\|f^*\|_{W^L}\le M$. Second, $\rho$ enters Lemma~\ref{emp_norm_conc} inside an
exponent, so a target that is spread very thinly relative to its PeSV norm requires
correspondingly more data before the empirical and population norms may be identified.
\end{remark}

\begin{remark}\label{rmk_rho_mu}
The parameter $\rho$ and the coherence $\mu$ of Theorem~\ref{budget_floor} measure the same
degeneracy from two sides. Since $\|h\|_{*}\le L_\sigma^{L-1}\|h\|_{2}$ for every $h$, we have
\begin{align}
1 \;\le\; \mu(f^*) \;\le\; \frac{L_\sigma^{L-1}}{\rho(f^*)},
\end{align}
so the condition $\mu(f^*)Q_\lambda<\|f^*\|_{W^L}$ under which the floor of
Theorem~\ref{budget_floor} is positive is implied by
$Q_\lambda<\rho(f^*)\|f^*\|_{W^L}/L_\sigma^{L-1}$, and both forms degrade together as the
target is spread more thinly.
\end{remark}

\subsection{Three Floors}

The first and most useful of the three rests on a single observation: although $\nu$ is neither convex nor everywhere differentiable as a function of $\theta$, both $g(x;\theta)$ and $\nu(\theta)$ are positively homogeneous of degree one in the last weight matrix $a^L$, and of degree $L$ in $\theta$ as a whole. Restricting the objective to the corresponding ray through $\hat\theta$ therefore reduces the first-order condition to the derivative of a univariate polynomial, and no subgradient of $\nu$ is ever needed.

\begin{theorem}\label{shrinkage_floor}
Let $\hat\theta$ be any local minimizer of $J_n(\cdot\,;\lambda)$. Then under Assumption~\ref{assump2} and the assumption that $\max_i\|x_i\|_2\le1$, for \textit{every} $f^*$ and every $\lambda>0$,
\begin{align}\label{shrinkage_floor_eq}
\left\|g(\cdot;\hat\theta)-f^*\right\|_n \;\ge\; \min\left\{\;\|f^*\|_n\;,\;\frac{(\lambda-\xi)_+}{L_\sigma^{L-1}}\;\right\}.
\end{align}
No assumption is made on $f^*$ beyond Assumption~\ref{assump1}, none on the location of $\hat\theta$, and none on the differentiability of $\nu$.
\end{theorem}

\begin{remark}\label{rmk_floor_threshold}
The threshold $\lambda>\xi$ in~\eqref{shrinkage_floor_eq} is exactly the threshold at which the term $2(\xi-\lambda)\nu(\hat\theta\ominus\theta^*)$ in the proof of Theorem~\ref{overpara_emp} turns nonpositive. The same inequality thus governs both halves of the theory, and it is instructive that it does so in opposite directions: the upper bound requires $\lambda\ge\xi$ in order to close, and it is precisely for $\lambda>\xi$ that the estimator is prevented from tracking the noise and the floor~\eqref{shrinkage_floor_eq} becomes positive. We return to this point in Remark~\ref{rmk_mechanism}.
\end{remark}

The second floor requires no information about the design points at all, and is therefore stated directly in the $L^2(\mu)$ norm.

\begin{theorem}\label{budget_floor}
For every $f^*$, every $\mathbf{m}$, every $\lambda>0$ and every realization of the data,
\begin{align}
\nu(\hat\theta)\;\le\;Q_\lambda:=\frac{1}{2n\lambda}\sum_{i=1}^n y_i^2 ,
\qquad\text{and consequently}\qquad
\left\|g(\cdot;\hat\theta)-f^*\right\|_2 \;\ge\; \operatorname{dist}_{2}\!\left(f^*,\mathcal{B}_{Q_\lambda}\right)
\;=\;\sup_{\|h\|_2\le1}\Big\{\langle f^*,h\rangle_2 - Q_\lambda\|h\|_{*}\Big\} .
\end{align}
Writing $\kappa:=\|f^*\|_2$, $M_*:=\|f^*\|_{W^L}$ and $\mu(f^*):=\|f^*\|_{*}M_*/\kappa^2\;(\ge1)$, the choice $h=f^*/\kappa$ gives the explicit form
\begin{align}\label{budget_floor_exp}
\left\|g(\cdot;\hat\theta)-f^*\right\|_2 \;\ge\; \kappa\left(1-\mu(f^*)\frac{Q_\lambda}{M_*}\right)_+ ,
\qquad \mathbb{E}\,Q_\lambda = \frac{\kappa^2+\sigma_\epsilon^2}{2\lambda} .
\end{align}
\end{theorem}

The next result concerns the estimator itself rather than its error, and will be used to
sharpen Theorem~\ref{budget_floor}. It is obtained from the same identity~\eqref{euler_eq}
that underlies Theorem~\ref{shrinkage_floor}, read in the two opposite directions.

\begin{proposition}\label{path_norm_bounds}
Let $\hat\theta$ be any local minimizer of $J_n(\cdot\,;\lambda)$ and suppose
$\max_i\|x_i\|_2\le1$. Then
\begin{align}
\textup{(i)}\qquad \nu(\hat\theta) \;\ge\; \frac{\left\|g(\cdot;\hat\theta)\right\|_n}{L_\sigma^{L-1}}
\;\ge\; \frac{\left(\|f^*\|_n-\left\|g(\cdot;\hat\theta)-f^*\right\|_n\right)_+}{L_\sigma^{L-1}},
\end{align}
and, if in addition $\lambda>\xi$,
\begin{align}
\textup{(ii)}\qquad \nu(\hat\theta) \;\le\; \frac{\|f^*\|_n^2}{4(\lambda-\xi)} .
\end{align}
\end{proposition}

\begin{remark}\label{rmk_no_collapse}
Part (i) says that the estimator cannot collapse: combining it with the empirical error bound
of Theorem~\ref{overpara_emp} or Theorem~\ref{underpara_emp}, whichever applies, gives the
explicit floor
\begin{align}
\nu(\hat\theta)\;\ge\;\frac{1}{L_\sigma^{L-1}}\left(\|f^*\|_n-\sqrt{\mathcal{U}(\mathbf{m})}\right)_+,
\end{align}
where $\mathcal{U}(\mathbf{m})$ denotes the right-hand side of the applicable bound. Thus as
soon as the error bound falls below $\|f^*\|_n^2$ -- which, by Theorem~\ref{para_gen}, happens
for every width once $n$ is large -- the path norm of the fitted network is bounded away from
zero by a quantity of the order of $\|f^*\|_n$. As it must, the bound degenerates when
$\|f^*\|_n$ is small: for $f^*\equiv0$ and $\lambda\ge\xi$ one has $\hat\theta=0$ exactly, so
no bound independent of $f^*$ is available.
\end{remark}

\begin{remark}\label{rmk_sharper_budget}
Part (ii) is sharper than the budget $Q_\lambda$ of Theorem~\ref{budget_floor}: for
$\lambda\gg\xi$ their ratio is approximately $\|f^*\|_2^2/\left(2(\|f^*\|_2^2+\sigma_\epsilon^2)\right)\le\frac12$,
and it is much smaller still when the noise dominates the signal. Theorem~\ref{budget_floor}
therefore holds with $Q_\lambda$ replaced by
\begin{align}
Q'_\lambda := \min\left\{\frac{1}{2n\lambda}\sum_{i=1}^n y_i^2,\ \frac{\|f^*\|_n^2}{4(\lambda-\xi)}\right\},
\end{align}
which lowers by a factor of at least $2(1+\sigma_\epsilon^2/\|f^*\|_2^2)$ the value of
$\lambda$ at which the floor becomes positive, and so widens the range of widths over which
the ascending branch of Corollary~\ref{four_phases} is visible in the $L^2(\mu)$ norm.
\end{remark}

Part (i) of Proposition~\ref{path_norm_bounds} is tied to the error bound, and therefore to
the width through the approximation term $H(\mathbf{m})^2\|f^*\|^2_{W^L}$, which does not
vanish with $n$. The following gives a floor for $\nu(\hat\theta)$ that is free of both. Write
\begin{equation}\label{dual_emp}
\|h\|_{*,n} \;:=\; \sup_{\nu(\theta)\le1}\ \langle h,\ g(\cdot;\theta)\rangle_n
\end{equation}
for the empirical dual norm, so that $\|\epsilon\|_{*,n}\le\xi$ by Lemma~\ref{embed_lem} and
$\|f^*\|_n^2\le\|f^*\|_{W^L}\,\|f^*\|_{*,n}$ by duality.

\begin{proposition}\label{nu_floor}
Let $\hat\theta$ be a global minimizer of $J_n(\cdot\,;\lambda)$ and suppose
$\max_i\|x_i\|_2\le1$. Then
\begin{align}\label{nu_floor_eq}
\nu(\hat\theta)\;\ge\;\frac{\left(\|y\|_{*,n}-\lambda\right)_+^2}{2L_\sigma^{2(L-1)}\,\|y\|_{*,n}} .
\end{align}
If moreover $\lambda+\xi\le\tfrac12\|f^*\|_{*,n}$, then
\begin{align}\label{nu_floor_eq2}
\nu(\hat\theta)\;\ge\;\frac{\|f^*\|_{*,n}}{12L_\sigma^{2(L-1)}}
\;\ge\;\frac{\|f^*\|_n^2}{12L_\sigma^{2(L-1)}\|f^*\|_{W^L}},
\end{align}
and hence, on the event of Lemma~\ref{emp_norm_conc},
$\nu(\hat\theta)\ge\rho(f^*)^2\|f^*\|_{W^L}\big/\big(24L_\sigma^{2(L-1)}\big)$.
\end{proposition}

\begin{remark}\label{rmk_nu_sharp}
The hypothesis $\lambda+\xi\le\frac12\|f^*\|_{*,n}$ is of the right shape and cannot be
dispensed with: whenever $\lambda\ge\|y\|_{*,n}$ one has $\hat\theta=0$ exactly, since then
$\langle y,g(\cdot;\theta)\rangle_n\le\lambda\nu(\theta)$ for every $\theta$ and the objective
is minimized at the origin. The threshold therefore separates the regime in which the signal
is detectable above the penalty from the one in which the estimator is annihilated, and the
case $f^*\equiv0$ of Remark~\ref{rmk_no_collapse} is its degenerate instance. Combining with
$\|f^*\|_{*,n}\ge\|f^*\|_n^2/\|f^*\|_{W^L}$, a sufficient form of the hypothesis is
$\lambda+\xi\le\rho(f^*)^2\|f^*\|_{W^L}/4$ on the event of Lemma~\ref{emp_norm_conc}.
\end{remark}

\begin{remark}\label{rmk_nu_at_peak}
At the widths $\mathbf{m}_\sharp$ of Theorem~\ref{double_descent}, where $\lambda_2\to C_1\sigma_\epsilon$,
the hypothesis of Proposition~\ref{nu_floor} becomes a signal-to-noise condition,
$C_1\sigma_\epsilon\le\rho(f^*)^2\|f^*\|_{W^L}/4$. It may fail, in which case
$\hat\theta=0$ at those widths. This does not affect the conclusions of
Section~\ref{sec_lb}: when it happens the minimum in Theorem~\ref{shrinkage_floor} is attained
by its first branch $\|f^*\|_n$, which is likewise a quantity not depending on $n$, so the
peak of Theorem~\ref{double_descent} stands either way. The two statements describe the
same shrinkage from opposite sides -- one bounds how far the estimator is pushed from the
target, the other how much of it survives.
\end{remark}

\begin{remark}[The empirical norm cannot be traded for $L^2(\mu)$ at the peak]\label{rmk_transfer}
Theorem~\ref{shrinkage_floor} is stated in $\|\cdot\|_n$ because the optimality condition lives
there. One might hope to convert it: Lemma~\ref{underpara_gen_lem1} gives
$\|h\|_n^2\le\|h\|_2^2+Z_n(\gamma)$ for $\|h\|_2\le\gamma$, so a floor $\mathrm{Fl}$ on
$\|\cdot\|_n$ yields one on $\|\cdot\|_2$ provided
$\mathrm{Fl}\ge\tfrac{4}{3}F\sqrt{(dm_1+m_1m_2+\cdots+m_{L-1})\log n/n}\asymp F\sqrt{\delta_n/L_\sigma}$,
where $F$ bounds the PeSV norm of the difference. This cannot be arranged at $\delta_n\approx1$:
there the number of weights is of the order of $n$, the localization radius is of order one, and
the two norms genuinely decouple, a class able to interpolate carrying no information from the
sample to the population. Any transfer argument must therefore stop short of the interpolation
threshold -- which is exactly where the ascent of Corollary~\ref{four_phases} culminates. This
is the reason Theorem~\ref{budget_floor}, which never touches the design points, is
indispensable rather than merely convenient.
\end{remark}

The third is immediate but worth recording, since it is the only one of the three that is positive at every width precisely when the target is \textit{not} realizable by any finite network --- that is, when $f^*$ lies in $W^L_M$ but on no finite-width slice of it.

\begin{proposition}\label{approx_floor}
For every $f^*$, every $\mathbf{m}$, every $\lambda$ and every realization, $\|g(\cdot;\hat\theta)-f^*\|_2\ge A_\infty(\mathbf{m},f^*)$. In particular $A_\infty(\mathbf{m},f^*)>0$ for all $\mathbf{m}$ if and only if $f^*\notin\bigcup_{\mathbf{m}}\mathcal{F}(\mathbf{m},\infty)$, and $A_\infty(\mathbf{m},f^*)\to0$ as the width grows, at the rate given by Theorem~\ref{approx1_2}.
\end{proposition}

\subsection{The Regularization Schedule}\label{subsec_schedule}

The floors of the previous subsection grow with $\lambda$, so before reading off a shape we
must fix which $\lambda$ is used at which width. Recall the two prescribed values:
$\lambda_1=\max\{8L_\sigma^{L-1},2^LcL_\sigma^{L-1}\sqrt d\}\sqrt{f(n)/n}$, required by
Theorem~\ref{overpara_gen} in the overparameterized regime, and
$\lambda_2=C_1\sigma_\epsilon\max\{\delta_n,H(\mathbf{m})^2\}$, required by
Theorem~\ref{underpara_gen} in the underparameterized regime $\delta_n<1$. Each is the
smallest value admissible in its own regime.

\begin{remark}[Why $\max\{\lambda_1,\lambda_2\}$ is not the right prescription]\label{rmk_schedule}
Both theorems are stated as ``$\lambda\ge\lambda_i$'', which invites the reading that any
larger $\lambda$ is equally good. It is not. Returning to~\eqref{empirical_e}, the bound
actually proved is
\begin{align}\label{lambda_dependence}
\|g(\cdot;\hat\theta)-f^*\|_n^2 \;\le\; C_1H(\mathbf{m})^2\|f^*\|_{W^L}^2 + 4\lambda\|f^*\|_{W^L}
+ 2(\xi-\lambda)\nu(\hat\theta\ominus\theta^*),
\end{align}
whose second term is $4\lambda\|f^*\|_{W^L}$, not $4\lambda_1\|f^*\|_{W^L}$; the right-hand
side displayed in Theorem~\ref{overpara_gen} is the value of~\eqref{lambda_dependence} at
$\lambda\asymp\lambda_1$. Enlarging $\lambda$ beyond $\lambda_1$ therefore does not preserve
that conclusion -- it degrades it proportionally -- and Proposition~\ref{nu_floor} and
Theorem~\ref{shrinkage_floor} show the degradation is real rather than an artefact of the
proof. Taking $\lambda=\max\{\lambda_1,\lambda_2\}$ at every width would consequently forfeit
the overparameterized rate wherever $\lambda_2>\lambda_1$, which is most of the
underparameterized regime.

The correct prescription is the regime-wise one: use each regime's own value inside that
regime,
\begin{equation}\label{schedule}
\lambda(\mathbf{m})\;=\;
\begin{cases}
\lambda_2(\mathbf{m}), & \delta_n(\mathbf{m})<1 \quad\text{(underparameterized)},\\[2pt]
\lambda_1, & \text{otherwise}\quad\text{(overparameterized)} .
\end{cases}
\end{equation}
This is also what makes the $\min$ in Theorem~\ref{para_gen} the correct reading of the upper
bound, in the sense of~\citet{huiyuan2023Nonasymptotic}: on the underparameterized range both
theorems apply, since $\lambda_2\ge\lambda_1$ there, and one keeps whichever bound is smaller;
past $\delta_n=1$ only Theorem~\ref{overpara_gen} is available and $\lambda$ must revert to
$\lambda_1$. The upper bound is therefore the lower envelope of two curves, each used only
where its hypotheses hold, and the switch of $\lambda$ at $\delta_n=1$ produces a drop in both
bounds at that width rather than a continuous join.
\end{remark}

\subsection{The Shape of the Risk}\label{subsec_shape}

Collecting Theorems~\ref{shrinkage_floor},~\ref{budget_floor} and
Proposition~\ref{approx_floor}, the error obeys at every width the floor
\begin{equation}\label{total_floor}
\left\|g(\cdot;\hat\theta)-f^*\right\|_n \;\ge\; \mathrm{Fl}(\mathbf{m}) :=
\max\left\{\;A_\infty(\mathbf{m},f^*)\;,\;
\min\Big\{\|f^*\|_n,\tfrac{(\lambda(\mathbf{m})-\xi)_+}{L_\sigma^{L-1}}\Big\}\;\right\},
\end{equation}
the first entry being the irreducible approximation error and the second the price of
shrinkage. Retaining the approximation term is what keeps the floor positive after the
regularization drops back to $\lambda_1$, and is what makes the shape of the lower bound match
that of the upper bound.

\begin{corollary}\label{four_phases}
Under the schedule~\eqref{schedule}, and with $f^*\notin\bigcup_{\mathbf{m}}\mathcal{F}(\mathbf{m},\infty)$
so that $A_\infty(\mathbf{m},f^*)>0$ at every width, the floor~\eqref{total_floor} behaves as
follows.
\begin{enumerate}
\item[(i)] \textup{(First descent.)} For small $\mathbf{m}$ both entries decrease: $A_\infty$
by the approximation theorem, and $\lambda_2\asymp C_1\sigma_\epsilon H(\mathbf{m})^2$ because
$H$ decreases.
\item[(ii)] \textup{(Valley.)} $\mathrm{Fl}$ attains an interior minimum at the width where
the two entries of~\eqref{total_floor} cross; the minimum is strictly positive, being at least
$A_\infty$ there.
\item[(iii)] \textup{(Ascent.)} Once $\delta_n$ overtakes $H(\mathbf{m})^2$ the second entry
grows like $C_1\sigma_\epsilon\delta_n$ and eventually dominates, so $\mathrm{Fl}$ increases,
reaching $\min\{\|f^*\|_n,(C_1\sigma_\epsilon-\xi)/L_\sigma^{L-1}\}$ as $\delta_n\to1^-$.
\item[(iv)] \textup{(Second descent.)} At $\delta_n=1$ the schedule reverts to $\lambda_1$ and
$\mathrm{Fl}$ drops to $\max\{A_\infty(\mathbf{m},f^*),(\lambda_1-\xi)_+/L_\sigma^{L-1}\}$,
thereafter decreasing with $A_\infty$ towards $(\lambda_1-\xi)_+/L_\sigma^{L-1}$.
\end{enumerate}
The upper bound of Theorem~\ref{para_gen} has the same four phases, so both bounds are a
$U$-shaped curve followed, after a drop at $\delta_n=1$, by a decreasing one.
\end{corollary}

\subsection{Double Descent of the Generalization Error}\label{subsec_dd}

We can now state the phenomenon as a property of the error rather than of a bound on it. Write
$\mathbf{m}_0$ for a width at the valley of the underparameterized bound, where
$H(\mathbf{m}_0)^2\|f^*\|^2_{W^L}\asymp(\sigma_\epsilon^2+\|f^*\|^2_{W^L})\delta_n(\mathbf{m}_0)$,
and $\mathbf{m}_\sharp$ for a width at the end of the underparameterized regime,
$\delta_n(\mathbf{m}_\sharp)\to1^-$.

\begin{theorem}[Double descent]\label{double_descent}
Let Assumptions~\ref{assump1}--\ref{assump3} hold, let $\lambda$ follow the
schedule~\eqref{schedule}, and let $\hat\theta$ be a minimizer of $J_n(\cdot\,;\lambda)$.

\textup{(a)} \textit{(Empirical norm; no further hypothesis.)} Put
\begin{align}
\mathrm{Fl}_\sharp \;:=\; \min\left\{\frac{\|f^*\|_{2}}{\sqrt2},\ \frac{C_1\sigma_\epsilon-\xi}{L_\sigma^{L-1}}\right\},
\end{align}
a quantity that does not depend on $n$. Then, on an event of probability at least
$1-e^{-f(n)}-\exp\big(-3n\rho(f^*)^2/(28L_\sigma^{2(L-1)})\big)-O(n^{-C_2})$,
\begin{align}
\left\|g(\cdot;\hat\theta_{\mathbf{m}_0})-f^*\right\|_n^2
&\;\le\; 2C\big(\sigma_\epsilon^2+\|f^*\|^2_{W^L}\big)\,\delta_n(\mathbf{m}_0), \label{dd_valley}\\[2pt]
\left\|g(\cdot;\hat\theta_{\mathbf{m}_\sharp})-f^*\right\|_n^2 &\;\ge\; \mathrm{Fl}_\sharp^2, \label{dd_peak}\\[2pt]
\left\|g(\cdot;\hat\theta_{\mathbf{m}})-f^*\right\|_n^2 &\;\le\;
C\Big\{H(\mathbf{m})^2\|f^*\|^2_{W^L}+\Lambda_{L,d}\big(\sigma_\epsilon^2+\|f^*\|^2_{W^L}\big)\sqrt{\tfrac{f(n)}{n}}\Big\}
\quad\text{for }\delta_n(\mathbf{m})\ge1. \label{dd_tail}
\end{align}

\textup{(b)} \textit{(Generalization error.)} Assume in addition $C_1>L_\sigma^{L-1}\mu(f^*)$ and
that $M_*/\sigma_\epsilon$ lies in the interval of Remark~\ref{rmk_rho_mu}, and put
\begin{align}
\mathrm{Fl}^{(2)}_\sharp \;:=\; \|f^*\|_2\left(1-\frac{\mu(f^*)\big(\|f^*\|_2^2+\sigma_\epsilon^2\big)}{2C_1\sigma_\epsilon\|f^*\|_{W^L}}\right)_+ \;>\;0,
\end{align}
again independent of $n$. Then, on an event of the same probability,
\begin{align}
\left\|g(\cdot;\hat\theta_{\mathbf{m}_0})-f^*\right\|_2^2
&\;\le\; 2C\big(\sigma_\epsilon^2+\|f^*\|^2_{W^L}\big)\,\delta_n(\mathbf{m}_0), \label{dd_valley2}\\[2pt]
\left\|g(\cdot;\hat\theta_{\mathbf{m}_\sharp})-f^*\right\|_2^2 &\;\ge\; \big(\mathrm{Fl}^{(2)}_\sharp\big)^2, \label{dd_peak2}\\[2pt]
\left\|g(\cdot;\hat\theta_{\mathbf{m}})-f^*\right\|_2^2 &\;\le\;
C\Big\{H(\mathbf{m})^2\|f^*\|^2_{W^L}+\Lambda_{L,d}\big(\sigma_\epsilon^2+\|f^*\|^2_{W^L}\big)\sqrt{\tfrac{f(n)}{n}}\Big\}
\quad\text{for }\delta_n(\mathbf{m})\ge1. \label{dd_tail2}
\end{align}

In either part, $\delta_n(\mathbf{m}_0)\to0$ and $\sqrt{f(n)/n}\to0$ while the floor is fixed,
so the ratio of the middle display to either of the others diverges: the error descends,
rises by a factor tending to infinity, and descends again.
\end{theorem}

\begin{remark}[Why the peak changes floor between the two parts]\label{rmk_which_floor}
Part (a) is the statement in the empirical norm and uses the shrinkage floor
(Theorem~\ref{shrinkage_floor}), which lives there natively; the two upper bounds are the
empirical versions, Theorems~\ref{overpara_emp} and~\ref{underpara_emp}. Part (b) is the
statement for the generalization error, and there the shrinkage floor \textit{cannot} be
used: transferring it from $\|\cdot\|_n$ to $\|\cdot\|_2$ requires the localized concentration
inequality, which degenerates precisely at $\delta_n\approx1$ -- that is, exactly at the width
$\mathbf{m}_\sharp$ where the peak sits, as explained in Remark~\ref{rmk_transfer}.
Part (b) therefore uses the budget floor (Theorem~\ref{budget_floor}), which is a statement
about $\|\cdot\|_2$ from the outset because it never touches the design points, and pays for
this with the two-sided signal-to-noise hypothesis.

This is not a defect of the argument but a feature of the problem: at the interpolation
threshold the class can interpolate, so the empirical and population norms genuinely decouple,
and any floor for the generalization error at that width must be obtained without passing
through the sample. We regard the pair (a)--(b) as the correct form of the statement: (a) is
unconditional but concerns the training-sample error, (b) concerns the quantity one cares
about and costs a hypothesis on $f^*$.
\end{remark}

\begin{remark}\label{rmk_dd_reading}
Three features of this statement are worth isolating. First, it uses both halves of the
theory and neither suffices alone: \eqref{dd_valley} and~\eqref{dd_tail} are
Theorems~\ref{underpara_emp} and~\ref{overpara_emp}, \eqref{dd_valley2} and~\eqref{dd_tail2}
their counterparts Theorems~\ref{underpara_gen} and~\ref{overpara_gen}, while the two middle
displays are Theorems~\ref{shrinkage_floor} and~\ref{budget_floor} under the
schedule~\eqref{schedule}. An upper bound with a local
maximum does not entail a maximum of the error; a lower bound alone says nothing about what
happens on either side. Second, the ascent is located at the end of the underparameterized
regime, that is at the interpolation threshold $\delta_n=1$, and its height $\mathrm{Fl}_\sharp$ does
not shrink with the sample size, whereas both valleys do. Third, the mechanism is the one
described in Remark~\ref{rmk_mechanism}: what grows with the width is the regularization that
Theorem~\ref{underpara_gen} requires, and $\mathrm{Fl}_\sharp$ is the shrinkage it causes. Combining
with Remark~\ref{rmk_schedule}, no width-independent $\lambda$ avoids this, since by
Theorem~\ref{no_fixed_lambda} a single $\lambda$ cannot both validate the underparameterized
bound near the threshold and reach the plateau.
\end{remark}

\begin{remark}\label{rmk_dd_picture}
The two curves may also be read off directly. On $\delta_n<1$ the upper bound of
Theorem~\ref{underpara_gen} and the floor~\eqref{total_floor} are both $U$-shaped, the former
because $H(\mathbf{m})^2$ falls while $\delta_n$ rises, the latter because
$\lambda_2=C_1\sigma_\epsilon\max\{\delta_n,H(\mathbf{m})^2\}$ does the same; at $\delta_n=1$
both drop, the upper bound to the overparameterized curve of Theorem~\ref{overpara_gen} and the
floor to $\max\{A_\infty,(\lambda_1-\xi)_+/L_\sigma^{L-1}\}$; and both then decrease, the
former to the plateau and the latter with $A_\infty$. The drop is a genuine discontinuity of
the prescription rather than of the error, since the two regimes are governed by different
regularizations; one may of course interpolate the two curves across a neighbourhood of
$\delta_n=1$ if a continuous presentation is wanted, at the cost of the constants there.
\end{remark}

\subsection{Why the Ascent Cannot Be Tuned Away}

It is natural to ask whether the ascent is an artefact of tying $\lambda$ to the width vector, and would disappear for a regularization parameter chosen once and for all. It would not, and the reason is that the two regimes make incompatible demands.

\begin{theorem}\label{no_fixed_lambda}
Suppose a single $\lambda$, independent of $\mathbf{m}$, is used at all widths, and that one requires both
\begin{enumerate}
\item[(i)] the conclusion of Theorem~\ref{underpara_gen} at some width with $\delta_n(\mathbf{m})\ge\delta$, which demands $\lambda\ge C_1\sigma_\epsilon\delta$, and
\item[(ii)] that the error at large widths reach the plateau.
\end{enumerate}
Then (i) and (ii) are incompatible as soon as
\begin{align}
\delta\;\ge\;\frac{2L_\sigma^{L-1}}{C_1\sigma_\epsilon}\sqrt{C\Lambda_{L,d}(\sigma_\epsilon^2+\|f^*\|^2_{W^L})}\left(\frac{f(n)}{n}\right)^{1/4},
\end{align}
that is, as soon as the number of weights exceeds $n^{3/4}$ up to logarithmic factors.
\end{theorem}

\begin{remark}
Consequently, for \textit{every} schedule $\mathbf{m}\mapsto\lambda(\mathbf{m})$ that both validates Theorem~\ref{underpara_gen} at widths of that size and attains the plateau at large widths, the floor at those widths exceeds the plateau, and $\lambda$ must afterwards decrease. The ascent is a consequence of the two requirements rather than of the particular schedule we prescribe. What is \textit{not} claimed is that a regularization parameter tuned to minimize the error separately at each width --- carrying, in consequence, no guarantee of any kind --- must also produce a peak; such a choice would fall below $\xi$ in the underparameterized range, where Theorems~\ref{shrinkage_floor} and~\ref{budget_floor} are vacuous. The precise statement is that the peak occurs for every regularization under which the theory's guarantees hold.
\end{remark}

\subsection{What Drives the Ascent}

\begin{theorem}\label{slow_rate}
If $\lambda\ge\xi$, then for every $f^*\in W^L_M$ and every $\mathbf{m}$,
\begin{align}\label{slow_rate_eq}
\left\|g(\cdot;\hat\theta)-f^*\right\|_n^2 \;\le\; C_1H(\mathbf{m})^2\|f^*\|^2_{W^L} + 4\lambda\|f^*\|_{W^L}.
\end{align}
\end{theorem}

Inequality~\eqref{slow_rate_eq} is a coarsening of~\eqref{empirical_e}, and we record it separately because of what it does \textit{not} contain: there is no term involving the number of weights, and none involving $\sigma_\epsilon^2$ divided by $n$. Above the threshold $\lambda\ge\xi$, therefore, the error carries no variance contribution of the form $\sigma_\epsilon^2(dm_1+m_1m_2+\cdots)/n$ for any target whatsoever.

\begin{remark}\label{rmk_mechanism}
This identifies the mechanism of the ascent, and it is not the one a reader might expect. Above the threshold the estimator does not track the noise: a displacement along a tangent direction of the network manifold improves the data-fitting term by at most $\|h\|_2^2$ while costing at least $\lambda\|h\|_2^2/\|h\|_{*}$ in the penalty, so it is profitable only when $\lambda$ is at most of the order of $\xi$. The rise of the error with the width is instead the price of the shrinkage that Theorem~\ref{underpara_gen} demands: as the number of weights grows, $\delta_n$ grows, $\lambda_2$ grows with it, and the estimator is pushed further from the target. The two regimes of Sections~\ref{sec6} and~\ref{sec7} are thus separated not by a change in how much noise is absorbed, but by a change in how much regularization is required in order for the error to be controlled at all; and the peak marks the width at which that requirement is most severe. We regard this as a sharpening of, rather than a departure from, the discussion following Theorem~\ref{para_gen}: the two regimes do behave differently, and the transition is real, but the quantity that changes is the admissible regularization rather than the absorbed variance.
\end{remark}

\begin{remark}
The lower bounds above complement Theorem~\ref{para_lower_bound} rather than replacing it. That result is a minimax statement in the sample size, and pins the plateau; Theorems~\ref{shrinkage_floor} and~\ref{budget_floor} are statements in the width at fixed sample size, and pin the shape. Neither implies the other, and both are needed: without the former the plateau could be far from optimal, and without the latter the shape would remain a property of an upper bound. We also note that no minimax quantity can supply the latter, since $\inf_{\hat f}\sup_{f^*\in W^L_M}$ and $\inf_{\hat f\in\mathcal{F}(\mathbf{m},\infty)}\sup_{f^*\in W^L_M}$ are monotone in $\mathbf{m}$, the first by enlarging the class over which the supremum is taken and the second by enlarging the class from which the estimator is drawn.
\end{remark}

\section{General Lipschitz Loss Functions}\label{sec9}
In this section, we extend the previous results from the mean squared loss to any Lipschitz loss function under some very mild conditions. This includes the mean squared loss, cross-entropy loss, hinge loss, etc. Such generalization is highly nontrivial, for example, involving localization of Rademacher or Gaussian complexity and uniform estimations. We first make the following assumption.

\begin{assumption}\label{assump4}
Let $\mathcal{L}(f,y)$ denote the loss function, and assume $\mathcal{L}$ is differentiable almost everywhere with respect to $y \in \mathbb{R}$.
$\mathcal{L}$ is Lipschitz in the predictor, uniformly in $y$: for all $f_1,f_2$ and all $y$,
\[
|\mathcal{L}(f_1,y) - \mathcal{L}(f_2,y)| \le L_1 |f_1 - f_2|.
\]
We write $L_1^{\prime}$ for the corresponding Lipschitz constant of $\tilde{\mathcal{L}}$, so that $L_1^{\prime}\le L_1$.
Moreover, for each $y$, it is $\gamma$-strongly convex in the predictor $f^*$ with respect to the $L^2(\mu)$ norm for some uniform $\gamma > 0$, i.e.,
\[
\int_{\mathbf{B}^d} \big( \mathcal{L}(f,y) - \mathcal{L}(f^*,y) - \frac{\partial \mathcal{L}}{\partial f}(f^*,y)(f-f^*) \big) \, d\mu \ge \frac{\gamma}{2} \|f - f^*\|_2^2
\]
for any $f$ (note that this is not $\gamma$-strong convexity in the usual sense).
\end{assumption}

Loss functions satisfying this assumption include MSE loss and, under some mild conditions (e.g. on the form of target functions and/or data distribution), also cross-entropy loss, hinge loss, and Huber loss. See Section 14.3 in~\citet{wainwright2019highdimensional} for a detailed discussion of this property. Thus, it applies to most common machine learning problems.

On a dataset $(x_1,y_1), (x_2,y_2), \dots, (x_n,y_n)$, the empirical version of the loss function is denoted by $\mathcal{L}_n(f,y) := \frac{1}{n} \sum_{i=1}^n \mathcal{L}(f(x_i),y_i)$, where $x = (x_1,x_2,\dots,x_n)$ and $y = (y_1,y_2,\dots,y_n)$. By the Cauchy inequality, $|\mathcal{L}_n(f_1,y) - \mathcal{L}_n(f_2,y)| \le L_1 \|f_1 - f_2\|_n$.  

Since we are now facing a general Lipschitz loss, there is not always an explicit expression like $\|\hat{f}-f^*\|_2$ as the measure of the difference between the estimator and the true model. We propose such a measure. First, let $\tilde{\mathcal{L}}(g(x;\hat{\theta}),f^*(x)) := \mathbb{E}_{y} \mathcal{L}(g(x;\hat{\theta}),y)$, the expectation over $y$ conditioned on $x$. It must be Lipschitz with respect to the first argument $g$ as we similarly have $\forall x, |\tilde{\mathcal{L}}(g_1,f^*) - \tilde{\mathcal{L}}(g_2,f^*)| \le L_1 |g_1 - g_2|$ and its empirical version. We denote its Lipschitz constant by $L_1^{\prime}$. We suggest using
\[
\mathbb{D}(g(\cdot ;\hat{\theta}), f^*(\cdot)) := \mathbb{E}_x \tilde{\mathcal{L}}(g(x;\hat{\theta}),f^*(x)) - \mathbb{E}_x \tilde{\mathcal{L}}(f^*(x),f^*(x))
\]
as the measure of the difference between $g(\cdot;\hat{\theta})$ and $f^*$ (without using the data $(x,y)$). For example, for regression problem~\eqref{problem}, $L$ is MSE, and $\mathbb{D}(g(\hat{\theta}),f^*)=\int_x\|(g(x;\hat{\theta}) - f^*(x)\|_2^2dx + \sigma^2$; for binary classification problem, $L$ is $y logp(x) + (1-y) \log(1-p(x))$, and $\Tilde{L}= p^*logp(x) + (1-p^*)\log(1-p(x))$, and $\mathbb{D}(g(\hat{\theta}),f^*)=\int_\mathbf{B^d}  p^*logp(x) + (1-p^*)\log(1-p(x))d\mu$. Both agree with our common practice, possibly up to a constant. For a given training dataset $(x_i,y_i)$, $i=1,2,\dots,n$, we similarly have the empirical version $\tilde{\mathcal{L}}_n(g(x;\hat{\theta}),f^*) - \tilde{\mathcal{L}}_n(f^*,f^*)$. The second term $\tilde{\mathcal{L}}(f^*,f^*)$ plays the role of a normalization factor, so that $\mathbb{D}(f^*,f^*) = 0$, which is required.

Generalization to an arbitrary loss function satisfying the above assumption is completely nontrivial, requiring more techniques and mathematical tools. All the techniques and methods are detailed in the Supplementary Information~\ref{appendix}.

We first successfully obtain the following analogous empirical error bound to~\eqref{emp_exp}. 

Throughout this section $f(n)$ is the same quantity~\eqref{fn_def} used in the squared-loss
case, so that $f(n)=\log n$ for $L=2$ and $f(n)=n^{1/3}(\log n)^{2/3}$ for $L\ge3$.

\begin{theorem}\label{general_loss_overpara_emp}
Under Assumptions~\ref{assump1}, \ref{assump2}, \ref{assump3}, and \ref{assump4}, and the assumption that $\max_i\|x_i\|_2 \le 1$, there exists a constant $c$ such that the regularized network estimator $g(\cdot ;\hat{\theta})$ with
$\lambda \ge \lambda_1 \equiv \max\{8L_{1} L_\sigma^{L-1}, 2^{L+1} c L_{1} L_\sigma^{L-1} \sqrt{d}\}\sqrt{\frac{f(n)}{n}}$ satisfies
\begin{align}\label{emp_general_loss}
\tilde{\mathcal{L}}_n(g(\cdot;\hat{\theta}),f^*) - \tilde{\mathcal{L}}_n(f^*,f^*) & \le C_1 H(\mathbf{m}) \|f^*\|_{W^L} / \sqrt{n} \\
& + \max\{16L_\sigma^{L-1}, 2^{L+1} c L_\sigma^{L-1} \sqrt{d}\} \|f^*\|_{W^L} \sqrt{\frac{f(n)}{n}}
\end{align}
with probability at least $1 - O(n^{-C_2})$, and
\begin{align}\label{emp_exp_general_loss}
\mathbb{E} \tilde{\mathcal{L}}_n(g(\cdot;\hat{\theta}),f^*) - \mathbb{E} \tilde{\mathcal{L}}_n(f^*,f^*) & \le C_1 H(\mathbf{m}) \|f^*\|_{W^L} / \sqrt{n} \\
& + \max\{16L_\sigma^{L-1}, 2^{L+1} c L_\sigma^{L-1} \sqrt{d}\} \|f^*\|_{W^L} \sqrt{\frac{f(n)}{n}}
\end{align}
for some constants $C_1, C_2, C > 0$.
\end{theorem}

The generalization error bound in the overparameterized regime is stated as follows.

\begin{theorem}\label{general_loss_overpara_gen}
Under Assumptions~\ref{assump1}, \ref{assump2}, \ref{assump3}, and \ref{assump4}, and the assumption that $\max_i\|x_i\|_2 \le 1$, there exists a constant $c$ such that the regularized network estimator $g(\cdot ;\hat{\theta})$ with
$\lambda \ge \lambda_1 \equiv \max\{8L_{1} L_\sigma^{L-1}, (L_1+L_1^{\prime})2^{L-1} c L_{1} L_\sigma^{L-1} \sqrt{d}\}\sqrt{\frac{f(n)}{n}}$ satisfies
\begin{align}
\int_{\mathbf{B}^d} \tilde{\mathcal{L}}(g(\cdot;\hat{\theta}),f^*) \, d\mu - \int_{\mathbf{B}^d} \tilde{\mathcal{L}}(f^*,f^*) \, d\mu & \le C_1 H(\mathbf{m}) \|f^*\|_{W^L} / \sqrt{n} \\
& + \max\{16L_\sigma^{L-1}, 2^{L+1} c L_\sigma^{L-1} \sqrt{d}\} \|f^*\|_{W^L} \sqrt{\frac{f(n)}{n}}
\end{align}
with probability at least $1 - O(n^{-C_2})$ for some constants $C_1, C_2, C > 0$.
\end{theorem}

This result is also based on the analog of the functional concentration lemma~\ref{f_concentration}.

\begin{lemma}\label{general_loss_f_concentration}
Assume that Assumptions~\ref{assump1}, \ref{assump2}, and \ref{assump3} hold. Let $\mathbf{m} = (m_1,\dots,m_{L-1})$, let $f^*$ with $\|f^*\|_{W^L} \le 1$ be fixed, and let $Z_n = \sup_{f\in \mathcal{F}(\mathbf{m},1)} \left| \int_{\mathbf{B}^d} \tilde{\mathcal{L}}(f,f^*) \, d\mu - \tilde{\mathcal{L}}_n(f,f^*) \right|$. Then $\mathbb{E} Z_n \le C_\mathcal{F} n^{-1/2}$ for some constant $C_\mathcal{F} > 0$ depending only on $L$, $L_\sigma$, and $L_1^{\prime}$. Furthermore, if $n \ge C_\mathcal{F}^2$, then
\begin{align}
P\left( Z_n \ge \frac{C_\mathcal{F}}{\sqrt{n}} + t \right) \le \exp\left( -\frac{n}{(16L_1^{\prime})} \min\left( \frac{t^2}{4e(2+(L_1^{\prime})^2)/L_1^{\prime}}, t \right) \right).
\end{align}
\end{lemma}

To obtain good empirical and generalization error bounds in the underparameterized regime is highly sophisticated. We present the results as follows.

\begin{theorem}\label{general_loss_underpara_emp}
Let $\delta_n = n^{-1} (2L_\sigma)^{L-1} (2L_{1} + 2|\mathcal{L}_{n}(0,f^*)|) (2d m_1 + 4m_1 m_2 + 4m_2 m_3 + \cdots + 2m_{L-1}) \log n < 1$. Under Assumptions~\ref{assump1}, \ref{assump2}, and \ref{assump3}, the regularized network estimator $g(\cdot ; \hat{\theta})$ with $\lambda \ge C_1 \sigma_\epsilon \max\{\delta_n, H(\mathbf{m})^2\}$ satisfies
\begin{align}
\tilde{\mathcal{L}}_n(g(\cdot ;\hat{\theta}), f^*) - \tilde{\mathcal{L}}_n(f^*,f^*) \le C \left\{ L_1^{\prime} H(\mathbf{m}) \|f^*\|_{W^L} + (\sigma_{\epsilon}^2 + \|f^*\|_{W^L}^2) \frac{(2d m_1 + 4m_1 m_2 + 4m_2 m_3 + \cdots + 2m_{L-1}) \log n}{n} \right\}
\end{align}
and $\nu(\hat{\theta}) \le C_{\mathrm{k}}$ with probability at least $1-O(n^{-C_2})$ for some constants $C_1, C_2, C, C_{\mathrm{k}} > 0$.
\end{theorem}

\begin{theorem}\label{general_loss_underpara_gen}
Let $\delta_n = n^{-1} (2L_\sigma)^{L-1} (2L_{1} + 2|\mathcal{L}_{n}(0,f^*)|) (2d m_1 + 4m_1 m_2 + 4m_2 m_3 + \cdots + 2m_{L-1}) \log n < 1$. Under Assumptions~\ref{assump1}, \ref{assump2}, and \ref{assump3}, then the regularized network estimator $g(\cdot ; \hat{\theta})$ with $\lambda \ge \lambda_2 \equiv C_1 \sigma_\epsilon \max\{\delta_n, H(\mathbf{m})^2\}$ satisfies
\begin{align}
\int_{\mathbf{B}^d} (\tilde{\mathcal{L}}(g(\cdot ;\hat{\theta}), f^*) - \tilde{\mathcal{L}}(f^*,f^*))  d\mu \le & C  \left\{ L_1^{\prime} H(\mathbf{m}) \|f^*\|_{W^L} + (\sigma_{\epsilon}^2 + \|f^*\|_{W^L}^2) \left ( \frac{(2d m_1 + 4m_1 m_2 + 4m_2 m_3 + \cdots + 2m_{L-1}) \log n}{n} \right. \right .\\
& \left. \left. + \sqrt{\frac{(2d m_1 + 4m_1 m_2 + 4m_2 m_3 + \cdots + 2m_{L-1}) \log n}{n}}  \right ) \right\}
\end{align}
with probability at least $1 - O(n^{-C_2})$ for some constants $C_1, C_2, C > 0$.
\end{theorem}

These two estimations are based on the analog of Lemma~\ref{underpara_gen_lem1} and a key ingredient relying on local Rademacher complexity estimation for our general $\mathcal{L}$, which are summarized in the following two results.

\begin{lemma}\label{general_loss_underpara_gen_lem1}
For any $0 < \gamma < 1$, define $\mathcal{B}_\mathcal{F}(\gamma) = \{ f\in \mathcal{F}^*(\mathbf{m},1) : \|f\|_2 < \gamma \}$, the $L^2(\mu)$-ball in $\mathcal{F}^*(\mathbf{m},1)$ of radius smaller than $\gamma$. Let
\begin{align}
Z_n(\gamma) = \sup_{f\in \mathcal{B}_\mathcal{F}(\gamma)} \left| \int_{\mathbf{B}^d} \tilde{\mathcal{L}}(f,f^*) \, d\mu - \tilde{\mathcal{L}}_n(f,f^*) \right|.
\end{align}
Let $m = \max\{m_1,m_2,\dots,m_{L-1}\}$. Then, for any $\gamma$ satisfying
\begin{align}
\sqrt{ \frac{2\log(18L_\sigma)(d m_1 + m_1 m_2 + m_2 m_3 + \cdots + m_{L-1}) \log n}{n} } \le \gamma \le 1,
\end{align}
we have
\begin{align}
\mathbb{E} Z_n(\gamma) \le 272 L_1^{\prime} \gamma \sqrt{ \frac{\log(18L_\sigma)(d m_1 + m_1 m_2 + m_2 m_3 + \cdots + m_{L-1}) \log n}{n} }.
\end{align}
\end{lemma}

The following is a key novel result of ours, relating the risk measure for a general loss to that for the $L^2$ loss. It is an indispensable result for our generalization theory to hold for general Lipschitz losses.

\begin{theorem}\label{indispensable result}
Given the uniformly 1-bounded function class $\mathcal{F}(\mathbf{m},1)$, which is clearly star-shaped around the ground truth $f^*$ (i.e., $c f \in \mathcal{F}(\mathbf{m},1)$ for any $c \in [0,1]$ and $f \in \mathcal{F}(\mathbf{m},1)$ near $f^*$), let
\[
r_n = \sqrt{ \frac{(2L_\sigma)^{L-1} (2L_{1} + 2|\mathcal{L}_{n}(0,f^*)|) (2d m_1 + 4m_1 m_2 + 4m_2 m_3 + \cdots + 2m_{L-1}) \log n}{n} } < 1.
\]
Here $r_n$ is the critical radius of the localized analysis, so that, up to the loss-dependent
constants displayed, $r_n^2$ is the quantity written $\delta_n$ in
Theorems~\ref{general_loss_underpara_emp} and~\ref{general_loss_underpara_gen}. We keep the two
symbols apart because both scales occur below: the conclusion~\eqref{event_1} carries one term
of each, and these are exactly the two terms $(2d m_1 + \cdots + 2m_{L-1})\log n/n$ and its
square root appearing in Theorem~\ref{general_loss_underpara_gen}.
Then:
\begin{enumerate}
    \item \label{general_loss_aux_thm_1} Recall that $\tilde{\mathcal{L}}(f,f^*)$ is $L_1^{\prime}$-Lipschitz with respect to the first argument $f$. Then
    \begin{align}
    \sup_{f\in \mathcal{F}(\mathbf{m},1)} \frac{|\int_{\mathbf{B}^d} (\tilde{\mathcal{L}}(f,f^*) - \tilde{\mathcal{L}}(f^*,f^*)) \, d\mu - (\tilde{\mathcal{L}}_n(f,f^*) - \tilde{\mathcal{L}}_n(f^*,f^*))|}{\|f-f^*\|_2 + r_n} \le 10 L_1^{\prime} r_n
    \end{align}
    with probability at least $1 - c_1 e^{-c_2 n r_n^2}$.
    \item \label{general_loss_aux_thm_2} 
    Furthermore, assume that $\Tilde{\mathcal{L}}(f,y)$ is $\gamma$-strongly convex for the first argument $f$ for each $y$, then we have 
    \begin{align}
    \|\hat{f}-f^*\|_2 \le c_2r_n + c_3 
    \end{align}
    and then,
    \begin{align}\label{event_1}
    \sup_{f\in \mathcal{F}(\mathbf{m},1)} {|\int_{\mathbf{B}^d}((\Tilde{\mathcal{L}}(f,f^*) - \Tilde{\mathcal{L}}(f^*,f^*))d\mu - (\Tilde{\mathcal{L}}_n(f,f^*)-\Tilde{\mathcal{L}}_n(f^*,f^*)))|} \le 
    c_2r_n^2 + c_3r_n
    \end{align}
\end{enumerate}
with the same probability as~\ref{general_loss_aux_thm_1}, for some constants $c_2,c_3$.
\end{theorem}

Finally, we also have an encompassing result unifying the underparameterized and overparameterized regimes.

\begin{theorem}\label{general_loss_para_gen}
Under Assumptions~\ref{assump1}, \ref{assump2}, \ref{assump3}, and \ref{assump4}, and the assumption that $\max_i\|x_i\|_2 \le 1$, 
There exists a constant $c$ such that the regularized network estimator $g(\cdot ; \hat{\theta})$ with $\lambda$ chosen regime-wise as in~\eqref{schedule}, namely $\lambda=\lambda_2$ where $\delta_n<1$ and $\lambda=\lambda_1$ where $\delta_n\ge1$, with $\lambda_1$ and $\lambda_2$ defined in Theorems~\ref{general_loss_overpara_gen} and~\ref{general_loss_underpara_gen}, respectively, satisfies
\begin{align}\label{loss_para_gen_exp}
\int_{\mathbf{B}^d} |\mathcal{L}(g(\cdot ;\hat{\theta}), f^*) \, d\mu| \le C & \left\{ L_1^{\prime} H(\mathbf{m}) \|f^*\|_{W^L} + (\sigma_{\epsilon}^2 + \|f^*\|_{W^L}^2) \right. \\
& \min \left( \max\{16 L_\sigma^{L-1}, 2^{L+1} c L_{1} L_\sigma^{L-1} \sqrt{d}\} \sqrt{\frac{f(n)}{n}}, \right. \\
&  \frac{(2d m_1 + 4m_1 m_2 + 4m_2 m_3 + \cdots + 2m_{L-1}) \log n}{n} \\
& + \left. \left. \sqrt{\frac{(2d m_1 + 4m_1 m_2 + 4m_2 m_3 + \cdots + 2m_{L-1}) \log n}{n}} \right ) \right\}
\end{align}
with probability at least $1-O(n^{-C_1})$ for some constants $C_1, C > 0$ (which may depend on  $L_\sigma$) and sufficiently large $n$.
\end{theorem}



We suspect that the above bound is also near optimal up to a logarithmic factor. 
The derivation of $f(n)$ and the resulting double descent phenomenon is similar to the MSE case. 

\section{Comparison with Existing Results}\label{sec10}
Classical data-dependent generalization error bounds, e.g., Rademacher complexity-based bounds, typically have the following form.

\begin{proposition}
Assume the loss function $\mathcal{L}$ is 0-1 valued. For any $\delta > 0$, with probability greater than $1 - \delta$, the following holds:
\begin{align}\label{rad_gen}
\mathcal{L}(f(x),y) \le \mathcal{L}_n(f(x_i),y_i) + R_n(\varphi \circ f) + \sqrt{\frac{8\log(2/\delta)}{n}},
\end{align}
where $\varphi$ is a dominating cost function, $\varphi \circ F = \{(x,y) \rightarrow \varphi(f(x),y) - \varphi(0,y), f \in F\}$, and $R_n(f)$ is the sample Rademacher complexity with respect to the training data $(x_i, y_i)$, $1\le i \le n$.
\end{proposition}

This kind of expression has some features different from ours. First, it works for any predictor $f$ regardless of how the model and training process are; second, it uses the McDiarmid concentration inequality to relate the left and right sides since the loss is 0-1 valued; third, it uses the Rademacher complexity of the predictors and does not consider any other complexity measures.

Since we do not require the loss function to be bounded, we rely on the Talagrand and local Talagrand concentration inequalities for the overparameterized and underparameterized regimes, respectively, instead of the McDiarmid concentration inequality, but the same thing is that we all need to estimate the Rademacher complexity required by these concentration inequalities. The second essential difference is that we need to estimate the empirical error $\mathcal{L}_n$, which highly relies on the optimality of the network estimator. The third is that we explicitly distinguish between the overparameterized and underparameterized regimes since their behavior is quite different, making the error bounds more accurate and predictive of the double descent phenomenon.

\section{Discussion}\label{sec11}
\begin{enumerate}
    \item 

We present the first \textit{near}-complete generalization theory for \textit{almost} general multilayer fully connected feedforward neural networks with path regularizations. \textit{Near} and \textit{almost} mean that we impose some very mild conditions on activation functions and loss functions. The theory is nonetheless widely applicable. Its overparameterized rate is sharp uniformly in the radius of the target ball when the loss is MSE and the activation is ReLU, with the qualification of Remark~\ref{rmk_ball_vs_space}. This theory predicts the double descent phenomenon.
\item
The regularization term is needed: without it the empirical error bounds, for instance
Theorem~\ref{overpara_emp}, are no longer valid.

\item
Our upper bounds use the hypothesis $f^*\in W^L_M$ only through a bound on $\nu(\hat\theta)$,
and such a bound is exactly what the dynamical stability of SGD supplies: a global minimum
that is linearly stable at learning rate $\alpha$ has bounded path norm~\citep{wu2023implicit}.
In the overparameterized regime, where that analysis applies, the basic inequality is therefore
available for the unregularized estimator as well, and the upper bounds carry over with the
learning rate playing the role of an inverse regularization parameter. We develop this
elsewhere. The lower bounds do not transfer in the same way, because they are driven by the
penalty itself, which is absent there.
\item
The empirical and generalization error bounds are, however, not optimal. Establishing a lower bound for the error estimation will give us a hint and insight into what the optimal error upper bound should be. Having a stronger approximation theory for multilayer networks will lead to a better bias bound. One critical direction is that we do not fully leverage the structure of the minimizer of our learning problem in our estimation of variance terms, as it is hard to explicitly characterize it for highly nonlinear neural networks. Any breakthrough in this aspect will make our variance estimation sharper and more useful (i.e., dependent on the width vector), which will also be our future research direction.
\item
To fully understand the generalization theory of neural networks, fully connected multilayer neural networks are only the first step. There remains a number of research questions. First, what is the corresponding theory for more general loss functions and activation functions? Second, what is the corresponding theory for other regularizations, including implicit regularization like early stopping? Lastly, what is the corresponding theory for other types of neural networks, e.g., LSTM, CNN, or transformer and combinations thereof? We expect the techniques developed here to adapt to these settings. Convolutional and recurrent networks are multilayer networks with symmetry constraints on the weights, so the argument should carry over with little change; transformers would require separate treatment of their arithmetic operations.
\end{enumerate}
\section{Acknowledgment}\label{ack}
We are grateful to Xiong Zhou, who was an intern here when this work was finished, for pointing out some errors in the preliminary version of this work and suggesting methods to overcome the problems.

\section{Tool Use}
The author used Claude Code to check the validity of the proofs of this work and for some
additional writing. The mathematical ideas and the details of the proofs, as well as the
original version of this manuscript, are the author's own.
\section{Supplementary Information}\label{appendix}

In this section, we provide detailed proofs of the results in main sections. 

\subsection{Proofs on results in section~\ref{sec3}}
As stated, under the definition of the finite $L$-layer neural network~\ref{L_layer_nn}, that it puts weights and biases in the same position. However, it does not really tell us why every finite $L$-layer neural network in usual appearance can be transformed in this form. Let us elaborate on this. 


First of all, figure~\ref{figure1: the structure of a layer of a transformed multilayer neural network} visually illustrates a layer of the transformed neural network:

\begin{figure}[htp]
    \centering  
    \includegraphics[width=.90\textwidth]{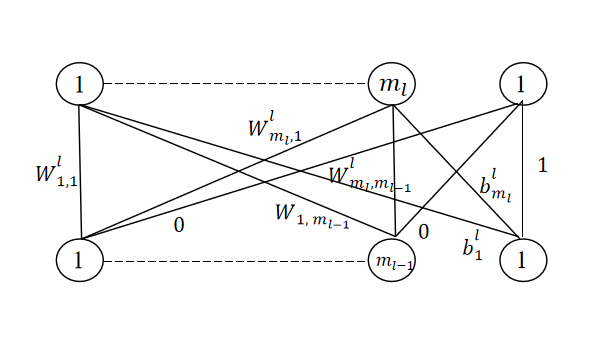}
    \caption{The structure of the $l$-th layer of a transformed multilayer neural network defined in~\eqref{transformed_nn}}
    \label{figure1: the structure of a layer of a transformed multilayer neural network}
\end{figure}

Mathematically, e.g. let us consider two-layer neural network firstly, i.e. 
\begin{align}
f = \sum_{j=1}^n c_i\sigma (\sum_{i=1}^m a_{ij}x_i + b_j)
\end{align}
The reformulated expression is $f' = \sum_{j=1}^n c_i\sigma (\sum_{i=1}^{m+1} a'_{ij}x'_i)$ where $x' = (x^T, 1)^T$ and $a'_{i,j} = a_{i,j}$, if $1\le i \le m, 1\le j \le n$ and $a'_{m+1,j} = b_j$. In general, for neural network \eqref{L_layer_nn}, the reformulated expression is 
\begin{equation}\label{transformed_nn}
f'(x) = \sum_{i_L=1}^{m_L+1} a_{i_L}^L \sigma\left (\sum_{i_{L-1}=1}^{m_{L-1}+1}a_{i_Li_{L-1}}^{L-1}\sigma \left(\sum_{i_{L-2}} \cdots \sigma \left (\sum_{i_1=1}^{m_1+1} a_{i_2i_1}^1 \sigma \left (\sum_{i_0=1}^{d+1}a_{i_1i_0}^0x_{i_0}\right ) \right ) \right ) \right )
\end{equation}
where $a_{i_t,i_{t-1}} = a_{i_t,i_{t-1}}$ if $1 \le i_t \le m_{i_t}$  and $1 \le i_{t-1} \le m_{i_{t-1}}$, $a_{i_t,m_{t-1}+1} = b_{t-1}$ if $1\le i_t \le m_{i_t} $, $a_{i_{m_t},i_{t-1}} = 0$ if $1\le i_{t-1} \le m_{i_{t-1}}$, $a_{i_{m_t},i_{m-1}} = 1$.

This reformulation is straightforwardly checked to be valid. 

\subsection{Proofs on results in section~\ref{sec_approx}}

\begin{proof}[Proof of Theorem~\ref{approx1_2}] Without loss of generality $\|f\|_{W^L} = 1$. Since $W^L \rightarrow C^{0,1}$, we find that $\|f\|_{L^2(P)} \le (1+R) \|f\|_{W^L}$ for all $f \in W^L$.
Recall from the proof of Theorem 2.10 in~\citep{wojtowytsch2020banach} that the unit ball of $W^l$ is the closed convex hull of
the class $\mathcal{G} = \{\pm\sigma(g) : \|g\|_{W^{l-1}} \le 1\}$ for $1\le l$ (holds for general Lipschitz activation as well).
Then applying Proposition~\ref{approx1_lem}
recursively and completing the proof.

For the norm bound, notice that when $L=2$, $\hat{f}= \frac{\sum_{i=1}^{m_1} \epsilon_i\sigma (g_i)}{m_1}$, where $\epsilon_i$ is + or -, each $\nu(g_i) = \|g_i\|_{W^1} \le 1$, thus $\nu(\hat{f})  \le 1$. $L>2$ case can be done by recursive relations. By the very computation of $\nu$, each $\|g_i\|_{W^{L-1}} \le 1$ will result in $\nu(g_i) \le 1$, thereby $\nu(\hat{f})  \le 1$. 
\end{proof}

\begin{proof}[Proof of Proposition~\ref{approx1_2_lem}] $L=2$ case is treated as Lemma 1 in~\citep{barron1993universal}. For any $\delta > 0$, since $f$ is in the closed convex hull of $\mathcal{G}$, we can find $f_1,\dots, f_n \in \mathcal{G}$ such that a convex combination of them $f^* = \gamma_1f_1 + \dots + \gamma_nf_n$ is within $\delta$ distance to $f$, that is, $\|f-f^* \|\le \delta$. Now let $g$ be a random variable taking values in $H$ with support $\{f_1,\dots, f_n\}$ and the probability valued in $f_i$ is $\gamma_i$. Let $g_1,g_2,\dots g_m$ be $m$ independent samples of $g$ and $\overline{g} = \frac{1}{m}\sum_{i=1}^mg_i$. By independence, $\E\|\overline{g}-f^*\|^2 = \frac{1}{m}\E\|g-f^*\|^2$. And since $\E g = f^*$ we get $\frac{1}{m}\E\|g-f^*\|^2 = \frac{1}{m}(E\|g\|^2-E\|f^*\|^2) = \frac{1}{m}(R^2-\|f^*\|^2)$. So there must exist some $g_1,g_2,\dots ,g_m$ such that $\|\overline{g}-f^*\|^2 \le \frac{1}{m}(R^2-\|f^*\|^2) \le \frac{1}{m}R^2$. One can then choose $\delta = \epsilon/\sqrt{m}$ and use triangular inequality to complete the proof.

Furthermore, if $f^*\ge 0$, and $f^*\le \gamma_1f_1 + \dots + \gamma_nf_n$, then by compatibility of $\ge$ with the Hilbert space structure of $H$ one can directly check that the above conclusion also holds. We need this to treat $L\ge 3$ case as follows.

When $L=3$, $\mathbf{m}^{\uparrow}=\{m_1,m_2\}$. Recall that $m_1 \le m_2$. We just need the size $\mathbf{m}^{\uparrow}$ sub-network. Without loss of generality, we choose the first $m_1$ nodes in the first layer and the first $m_2$ nodes in the second layer. And we set the weights connected to other nodes zeros (thus these nodes will not have any effects). 

For any $\delta > 0$, since $f$ is in the closed convex hull of $T(\gG_2)$, we can also find $f_1,\dots, f_n \in \mathcal{G}_2$ such that a convex combination of them $f^* = \gamma_1f_1 + \dots + \gamma_nf_n$ is within $\delta_2$ distance to $f$, that is, $\|f-f^* \|\le \delta_2$. We can find $g^2_1,\dots, g^2_{m_2}$ so that  $\left \| f^* - \frac{1}{m_2}\sum_{i=1}^{m_2}g^2_i   \right \|_H^2 \le \frac{R^2}{{m_2}}$. Let $h^2_i\in \mathcal{H}_2, 1\le i\le m_2$ so that $T(h^2_i) = g^2_i$. For each $h^2_i$ we can find $h^{2*}_i$ in $\mathcal{G}_1$ so that $\|h^2_i-h^{2*}_i\|\le \delta_1$.  Then we repeat the procedure as what we did for $f^*$: for each $h^{2*}_i$ we know it is a convex combination $h^{2*}_i=\gamma^1_{i,1}g^1_{i,1} + \dots +\gamma^1_{i,k_i}g^1_{i,k_i}$ for $g^1_{i,\cdot}\in \mathcal{G}_1$. Let's take $T(h^{2*}_i) = T(\gamma^1_{i,1}g^1_{i,1} + \dots +\gamma^1_{i,k_i}g^1_{i,k_i}) \le \gamma^1_{i,1}T(g^1_{i,1}) + \cdots + \gamma^1_{i,k_i}T(g^1_{i,k_i})$ due to the convexity of $T$ with respect to $\ge$. We average over $i$ on both side and  we still denote the normalized coefficient on the right side by $\gamma^1_{i,\cdot}$. That is $f^{**}:=\frac{T(h^{2*}_1)+\cdot + T(h^{2*}_{m_2})}{m_2} \le \sum_{i=1}^{m_2}\gamma^1_{i,1}T(g^1_{i,1}) + \cdots + \gamma^1_{i,k_i}T(g^1_{i,k_i})$. By the non-negativity of $T$, $f^{**}$ is non-negative with respect to $\ge$ and we can apply the conclusion in the last paragraph of discussed in the $L=2$ case. We can do the sampling for $f^{**}$ as in $L=2$ case: we independently sample $m_1$ number of $T(g^1_{\cdot,\cdot})$ with probability $\gamma^1_{\cdot,\cdot}$. So we conclude that 
there exists an $m_1$ length array of samples, denoted this array by $B^1$, such that
\begin{align}
\left \|f^{**} - \frac{\sum_{g^1\in B^1}T(g^1)}{m_1}\right \| ^2_H \le \frac{1}{m_1}L^2_\sigma R^2
\end{align}
As $m_1 \le m_2$, our strategy is to assign one $g^1\in B^1$ to only one node in the second layer, that is, there exists only $m_1$ number of weights from layer 1 to layer 2 with weight $\frac{1}{m_1}$, the other $m_1m_2 - m_1$ weights are all zeros. In particular, if $m_1 < m_2$, there are $m_2-m_1$ nodes in the second layer with weights connecting to the nodes in the first layer all setted to be zeros. These additional zero weights will not affect the other parts of the neural networks at all.

Now
\begin{align}\label{approx_L_3}
    \left \|f - \frac{\sum_{g^1\in B^1}T(g^1)}{m_1}\right \|_H  \le & \left \|f-f^{*}     \right \|_H + \left \|f^* -  \frac{1}{m_2}\sum_{i=1}^{m_2}g^2_i   \right \|_H  + \left \|\frac{1}{m_2}\sum_{i=1}^{m_2}T(h^2_i) - \frac{1}{m_2}\sum_{i=1}^{m_2}T(h^{2*}_i)   \right \|_H + \left \|f^{**}- \frac{\sum_{g^1\in B^1}T(g^1)}{m_1}\right \|_H \\ \le & \delta + \frac{R}{\sqrt{m_2}} + L_\sigma \delta_1 + \frac{L_\sigma R}{\sqrt{m_1}}
\end{align}
We choose $\delta = \frac{\epsilon}{\sqrt{m_2}}$ and $\delta_1 = \frac{\epsilon}{\sqrt{m_1}}$, and we arrive at
\begin{align}
    \left \|f - \frac{\sum_{g^1\in B^1}T(g^1)}{m_1}\right \|_H \le (R+\epsilon)(\frac{L_\sigma}{\sqrt{m_1}} + \frac{1}{\sqrt{m_2}})
\end{align}

We need emphasize that the weights from the second layer to the output node are all 1. And the nodes in the second layer no longer represent $g^2_i$, otherwise the weights there should be $\frac{1}{m_2}$. The critic point is that: our strategy is to approximate the sum $\frac{1}{m_2}\sum_{i=1}^{m_2}g^2_i $ instead of each $g^2_i$ individually ! All the approximators are actually $g^1$ for $g^1 \in B^1$ in the first layer. The nodes in the second layer represent $\frac{1}{m_1}g^1$ for some $g^1\in B^1$ or 0. The same is true for $L > 3$.

For $L > 3$ the proof is completely analogous to $L=3$ case. In each layer from the top to the bottom, we can always use the same procedure as we did for $L=3$ case, such that the right side of~\eqref{approx_L_3} has $2(L-1)$ terms: one term is for approximation by the closure condition and the other term is for sampling approximation. Therefore, finally we have 
\begin{align}
     \left \|f - \frac{\sum_{g^1\in B^1}T(g^1)}{m_1}\right \|_H \le
     \delta + \frac{R}{\sqrt{m_{L-1}}} + L_\sigma \delta_1 + \frac{L_\sigma R}{\sqrt{m_{L-2}}} + \cdots + L_\sigma \delta_{L-2} + \frac{L_\sigma R}{\sqrt{m_1}}
\end{align}
Now we choose $\delta = \frac{\epsilon}{\sqrt{m_{L-1}}}, \delta_1 = \frac{\epsilon}{\sqrt{m_{L-2}}}$ and $\delta_{L-2}=\frac{\epsilon}{\sqrt{m_1}}$. The $g^{L-1}$s are simply constructed by combining $g^1$s along the weights in layers from the bottom to the top and we complete.

\end{proof}
\begin{proof}[Proof of Theorem~\ref{ReLU_approx1_2}]
First of all, for ReLU activations, $W^l \subset W^L$ for all $1\le l\le L$, with
$\|f\|_{W^L}\le\|f\|_{W^l}$. Indeed the identity is represented exactly by
$x=\sigma(x)-\sigma(-x)$, so a network of depth $l$ may be extended to depth $L$ by inserting
$L-l$ such layers, and this leaves the path norm unchanged; passing to the closure gives the
inclusion of the spaces~\citep{wojtowytsch2020banach}. This corresponds to $\mathbf{G}_i \subset H$ in Proposition~\ref{approx1_2_lem}.

Without loss of generality $\|f\|_{W^L} = 1$. Since $W^L \rightarrow C^{0,1}$, we find that $\|f\|_{L^2(P)} \le (1+R) \|f\|_{W^L}$ for all $f \in W^L$.
Recall from the proof of Theorem 2.10 in~\citep{wojtowytsch2020banach} that the unit ball of $W^l$ is the closed convex hull of
the class $\mathcal{G} = \{\pm\sigma(g) : \|g\|_{W^{l-1}} \le 1\}$ for $1\le l$ (holds for general Lipschitz activation as well).
Then applying Proposition~\ref{approx1_2_lem}
recursively and completing the proof.

The norm bounds can be done similarly as Theorem~\ref{approx1_2}.
\end{proof}

\begin{proof}[Proof of corollary~\ref{main_app_bound}] Since $\lambda \le \|f\|_{W^L}$ and 
when $N>d^{d/d-1}$, $N+2>d[N^{1/d}]$, so when $M>C_1d^{d/d-1}$, $\max\{d\{N^{\lfloor 1/d \rfloor},N+2\} = N+2$. In this case, $N=M/C_1-2$. The right hand side of~\eqref{app_bound} 
\begin{align}
& 131\|f\|_{W^L} \sqrt{d}(((M/C_1-2)^2((L-C_2)/11)^2log_3(M/C_1))^{-1/d} \\
& = 131\sqrt{d}(C_1)^{1/d}/121\|f\|_{W^L}((M-2C_1)^2(L-18-2d)^2(log_3M-3-d))^{-1/d} \\
\end{align}
As $d\ge 1$
\begin{align}
& 131\sqrt{d}(C_1)^{1/d}\|f\|_{W^L}((M-2C_1)^2(L-18-2d)^2(log_3M-3-d))^{-1/d} \\
& \le 131\sqrt{d}(C_1)^{1/d}/121\|f\|_{W^L}((L-20)^2(M-162)^2(log_3M-4))^{-1/d} \\
& = C\|f\|_{W^L}((L-20)^2(M-162)^2(log_3M-4))^{-1/d}
\end{align}
where $C=131\sqrt{d}(C_1)^{1/d}/121$ depending only on $d$.
\end{proof}

\subsection{Proofs on results in section~\ref{sec5}}
\begin{proof}[Proof of Theorem~\ref{overpara_emp}]
For any $f^* \in W^L$, let $g(\cdot ; \theta^*)$ denote the $L_{2}(B^{d+1})$-approximation of $f^*$  in Theorem ~\ref{approx1} where $\theta^* = (a^L, w^{L-1}, \dots, w^2,w^1)^T$. 

Because of the choice of $m$, one can equivalently think of $g(\cdot ; \theta^*)$ as a neural network with width vector $\mathbf{m}$ with weights connecting to extra nodes being all 0. Thus
by the optimality of $\hat{\theta}$, we have
\begin{align}
{1 \over 2n}\sum_{i=1}^n(g(x_i;\hat{\theta}) - y_i)^2 + \lambda \nu(\hat{\theta}) & \le {1\over 2n}\sum_{i=1}^n(g(x_i;\theta^*)-y_i)^2 + \lambda \nu(\theta^*)
\end{align}
Taking $y_i = f^*(x_i) + \epsilon_i$ and rearranging terms gives
\begin{align}\label{emp_decomposition}
& {1\over 2}\|g(\cdot ; \hat{\theta}) - f^*(\cdot)\|^2_n \\
& \le \lambda (\nu(\theta^*) - \nu(\hat{\theta})) + {1\over 2}\|g(\cdot ; \theta^*) - f^*(\cdot)\|^2_n + {1 \over n} |\sum_{i=1}^n \epsilon_i(g(x_i;\hat{\theta}) - g(x_i;\theta^*))  | \\
& \equiv T_1 + T_2 + T_3.
\end{align}

One can regard the above manipulation as a pseudo form of bias-variance decomposition inequality (instead of equality as we take the surrogate $g(\cdot ; \theta^*)$ as the mean of our estimator $g(\cdot ; \hat{\theta}$)) from both the idea and the expressions point of views.
\noindent
By definition, we have
\begin{align}\label{T_1}
T_1 = \lambda(\nu(\theta^*) - \nu(\hat{\theta})) = 2\lambda \nu(\theta^*)  - \lambda \nu(\theta^* \ominus \hat{\theta})
\end{align}
Theorem~\ref{approx1_2} gives the bound for $T_2$
\begin{align}\label{T_2}
T_2 = {1 \over 2}\|g(\cdot ;\theta^*) - f^*(\cdot)\|^2_{L^2(\mathrm{P}_n)} \le C_1 H^2(\mathbf{m}){\|f\|^2_{W^L} } 
\end{align}
for some constant $C_1 > 0$. 

For $T_3$, we first derive a coarse bound. This bound is based on the following estimate
\begin{equation}\label{T3_coarse_bound}
\begin{aligned}
T_3 & = {1 \over n} |\sum_{i=1}^n \epsilon_i(g(x_i;\hat{\theta}) - g(x_i;\theta^*))| \\
    & = {1 \over n} {\sum_{i=1}^n} |\epsilon_i\| |g(x_i;\hat{\theta}- \theta^*)| \\
    & \le {1 \over \sqrt{n}} \sqrt{\sum_{i=1}^n \epsilon_i^2/n}\sqrt{\sum_{i=1}^ng^2(x_i;\hat{\theta} - \theta^*)}
\end{aligned}
\end{equation}

We can prove by induction that, for any $l\ge 1$ and $f\in W^l$ 
\begin{align}\label{key_res}
| f(x) | \le L_\sigma^{l-1}\|x\|_2\|f\|_{W^l}
\end{align}
Let us give a proof here:
when $l=1$, from the definition of neural space~\ref{neural space} $W^1$ is the space of affine functions $f(x) = \sum_{i=1}^{d+1}a_ix_i$ with $l^2$ norm $\|f\|_2 = \sqrt{\sum_{i=1}^{d+1}a^2_i}$, the result follows from Cauchy inequality.

If~\eqref{key_res} holds for $l$, then for any $f(x) \in W^{l+1}$, from the definition of $W^{l+1}$
\begin{align}
|f_\mu(x)| & = |\int_{B^{W^l}}\sigma(g(x))\mu(dg) |\\
        & \le \int_{B^{W^l}}|\sigma(g(x))\|\mu(dg)| \\
        & \le \int_{B^{W^l}}L_\sigma|g(x)||\mu(dg)| \\
        & \le L_\sigma^{l}\|x\|_2 \int_{B^{W^l}}|\mu(dg)|
\end{align}
As $g(x)\in B^{W^l}$ has norm $\|g\|_{W^l} \le 1$. Taking infimum with respect to $\mu$, we complete the proof. 

 Thereby, the upper bound of $T_3$ becomes
\begin{align}
T_3 & \le {1 \over \sqrt{n}}\sqrt{\sum_{i=1}^n \epsilon_i^2/n}\sqrt{n} L_\sigma^{L-1}\|g(x;\hat{\theta}\ominus\theta^*)\|_{W^L} \\
& = {1 \over \sqrt{n}}\sqrt{\sum_{i=1}^n \epsilon_i^2} L_\sigma^{L-1}\|g(x;\hat{\theta}\ominus\theta^*)\|_{W^L} \\
& \le {1 \over \sqrt{n}}\sqrt{\sum_{i=1}^n \epsilon_i^2} L_\sigma^{L-1}\nu(\hat{\theta} - \theta^*) 
\end{align}
Letting $\epsilon = (\epsilon_1, \cdots, \epsilon_n), \xi = {1 \over \sqrt{n}}\|\epsilon\|_2L^{L-1}_\sigma$ and choosing $\lambda \ge 2\xi$ and noting that $\nu(\theta^*) \le \|f^*\|_{W^L}$(\ref{approx3}), we obtain
\begin{equation}
\begin{aligned}
\label{empirical_e}
& \|g(\cdot ;\hat{\theta}) - f^*(\cdot)\|^2_n \\
& \le C_1H^2(\mathbf{m})\|f^*\|^2_{W^L}m^{-1} + 4\lambda \nu(\theta^*) + 2\left ( \xi - \lambda \right ) \nu(\hat{\theta} - \theta^*) \\
& \le C_1H^2(\mathbf{m})\|f^*\|^2_{W^L}m^{-1} + 4\lambda \|f^*\|_{W^L}
\end{aligned}
\end{equation}
Now we can bound $\xi$. 
Applying a tail bound for quadratic forms of sub-Gaussian vectors~\citep{hsu2012tail} gives
\begin{align}
P(\|\epsilon\|_2^2 \ge \sigma^2_{\epsilon} + \sigma^2_{\epsilon}\sqrt{t} + \sigma^2_\epsilon t) \le e^{-t}
\end{align}
Letting $f(n)$ be as in~\eqref{fn_def} and choosing $t=4f(n) > 1$ 
for $n\ge 2$ yields
\begin{align}
\|\epsilon\|_2^2 < \sigma^2_\epsilon + \sigma^2_\epsilon\sqrt{t} + \sigma^2_\epsilon t < 3\sigma^2_\epsilon t < 16\sigma^2_\epsilon f(n) 
\end{align}
with probability at least $1-e^{-f(n)} > 1-O(n^{-C_2})$. 
Thus, for $\lambda \ge 2\xi$ to hold with the same probability, it suffices to set $\lambda = 8L_\sigma^{L-1}\sigma_\epsilon \sqrt{\frac{f(n)}{n}}$. 
To complete the proof, substituting the value of $\lambda$ into~\eqref{empirical_e} gives
\begin{align}\label{coarse_overpar_emp}
\|g(\cdot ; \hat{\theta}) - f^*(\cdot)\|^2_n \le C_1H^2(\mathbf{m})\|f^*\|^2_{W^L} + 16L_\sigma^{L-1}(\sigma^2_\epsilon + \|f^*\|^2_{W^L})\sqrt{\frac{f(n)}{n}} 
\end{align}
where we have used the inequality $2\sigma_\epsilon \|f^*\|_{W^L} \le \sigma^2_\epsilon + \|f^*\|^2_{W^L}$.

To prove~\eqref{emp_exp}, one way is  using method totally analogous to the one in~\citep{huiyuan2023Nonasymptotic} (Eq.(14) of Theorem 2). We refer the readers to the proof therein for details. 

We give a second proof, based on the estimation of Gaussian complexity which looks more concise. The motivation is that~\eqref{T3_coarse_bound} is closely related to Gaussian complexity. 

For any set ${T}$, we denote $\mathcal{G}(T)$ and $\mathcal{R}(T)$ to be the Gaussian complexity and Rademacher complexity of $T$. It is well known that 
\begin{align}\label{g_r_equivalence}
{\mathcal{G}(T) \over 2\sqrt{\log n}} \le \mathcal{R}(T) \le \sqrt{2\over \pi}\mathcal{G}(T)
\end{align}
By~\eqref{coarse_overpar_emp}, Lemma~\ref{gen_lem1} and~\eqref{g_r_equivalence}, and let $\sigma_i$ be Rademacher random variables and $\epsilon_i$ Gaussian random variables $N(0,1)$, we have
\begin{equation}
\begin{aligned}
E_{\epsilon}T_3 & = \sigma_\epsilon E_{\epsilon} {1 \over n} |\sum_{i=1}^n \epsilon_i(g(x_i;\hat{\theta}) - g(x_i;\theta^*))| \\
    & =  \sigma_\epsilon E_{\epsilon} {1 \over n} |\sum_{i=1}^n \epsilon_i(g(x_i;\hat{\theta} - \theta^*))| \\
    & =  \sigma_\epsilon E_{\epsilon} {1 \over n} \sup_{f\in \{g(\cdot ;\hat{\theta} - \theta^*),-g(\cdot ;\hat{\theta} - \theta^*)\}}\sum_{i=1}^n \epsilon_i(f(x_i)) \\
    & \le {2\sigma_\epsilon \over n}\sqrt{\log n} E_{\sigma}  \sup_{f\in \{g(\cdot ;\hat{\theta} - \theta^*),-g(\cdot ;\hat{\theta} - \theta^*)\}}\sum_{i=1}^n \sigma_if(x_i)   \\
    & = {2\sigma_\epsilon \over n}\sqrt{\log n} E_{\sigma} |\sum_{i=1}^n \sigma_i (g(x_i;\hat{\theta} - \theta^*)) | \\
    & \le {2\sigma_\epsilon \over n}\sqrt{\log n} E_{\sigma} \sup_{h \in \mathcal{F}(\mathbf{m}+\mathbf{m},\nu(\hat{\theta}\ominus\theta^*))}|\sum_{i=1}^n \sigma_i h(x_i) | \\
    & \le {2^Lc\sigma_\epsilon \over n}\sqrt{\log n} L_\sigma^{L-1}\nu(\hat{\theta}\ominus\theta^*)\sqrt{dn} \\
    & = 2^Lc\sigma_\epsilon L_{\sigma}^{L-1}\nu(\hat{\theta}\ominus\theta^*)\sqrt{dlogn \over n} \\
    & \le 2^Lc\sigma_\epsilon L_{\sigma}^{L-1}\nu(\hat{\theta}\ominus\theta^*)\sqrt{df(n) \over n}
\end{aligned}
\end{equation}
Summing together with~\eqref{T_1} and~\eqref{T_2} and noticing the value of $\lambda$ we chosen, we complete.
\end{proof}

\subsection{Proofs on results in section~\ref{sec6}}

\begin{proof}[Proof of Lemma~\ref{gen_cor0}] Let $b=\max_{x\in [-1,1]}\sigma(x)$. By the argument in the proof of Proposition~\ref{gen_lem0}~(e.g. 5.24 in~\citep{wainwright2019highdimensional}), we arrive at
\begin{align}\label{dubley}
\mathbb{E}_\rho \left [ \sup_{f\in \mathcal{F}} \left | \frac{1}{n}\sum_{k=1}^n\rho_kf(x_k)\right | \right ] \le \frac{24}{\sqrt{n}}\int_0^{2b}\sqrt{\log\mathcal{N}(t;\mathcal{F},\|\cdot \|_{\mathbb{P}_n})}dt
\end{align}
As $\|\cdot \|_{\mathbb{P}_n} \le \| \cdot \|_{\infty}$ , it remains to bound the metric entropy with respect to supremum norm.

Let $g(\cdot ;u_1),g(\cdot ;u_2)\in \mathcal{F}$, we have
\begin{align}
& |g(x;u_1)-g(x;u_2)| \\
& =  |\sigma (u_1x) - \sigma (u_2x) | \\
&  \le  L_\sigma |u_1x - u_2x | \\
& \le  L_\sigma \|u_1-u_2\|_2\|x\|_2 \\
& \le  L_\sigma  \|u_1-u_2\|_2
\end{align}
Therefore, in order to cover $\mathcal{F}$ it  needs to cover $\mathbf{B}^d$ with respect to $l^2$ norm. The volume argument in Lemma 5.7 of~\citep{wainwright2019highdimensional} yields
\begin{align}
\mathcal{N}(\delta, \mathbf{B}^d,\|\cdot\|_2) \le (1+2\delta^{-1})^d
\end{align}
resulting in
\begin{align}
\mathcal{N}(\delta, \mathcal{F}, \|\cdot\|_{\infty}) \le (1 + 2L_\sigma\delta^{-1})^d 
\end{align}
Substituting the above metric entropy estimation to the right hand side of~\eqref{dubley}, we get that there exists a universal constant $c$ depending on $\sigma$ and $L_\sigma$ satisfying
\begin{align}
\mathbb{E}_\rho \left [ \sup_{f\in \mathcal{F}} \left | \frac{1}{n}\sum_{k=1}^n\rho_kf(x_k)\right | \right ] \le 24\sqrt{\frac{1}{n}}\int_0^{2b}\sqrt{\log(1+2L_\sigma t^{-1})}dt = c\sqrt{\frac{d}{n}}
\end{align}
where we have used the finiteness of the integral. In particular, if $\sigma$ is identity function, we have $c$ a universal constant.
\end{proof}
\begin{proof}[Proof of Lemma~\ref{gen_lem1}] Let's first do the cases of $L=2$ and $L=3$ for clarification, then one can generalize it to general $L$ without any difficulty.

When $L=2$, this was already done in~\citep{huiyuan2023Nonasymptotic}. For completeness, we present the proof here.

By definition, $\mathbf{m} = (m_1)$
\begin{align}
\mathbb{E}_\rho & \sup_{f\in \mathcal{F}(m_1,F)} \left | \sum_{k=1}^n \rho_k f(x_k) \right | \\
                 & = \mathbb{E}_\rho \sup_{\nu(\theta) \le F}\left |\sum_{i=1}^n\rho_i \sum_{k=1}^{m_1} a_k\sigma(w^T_kx_i) \right | \\
                 & = \mathbb{E}_\rho \sup_{\nu(\theta)\le F} \left |\sum_{k=1}^{m_1} a_k \|w_k\|_2\sum_{i=1}^n \rho_i  \sigma (u_i^Tx_i) \right |, \quad   \|u_i\|_2=1\\                 
                 & = \mathbb{E}_\rho \sup_{\nu(\theta)\le F} \sup_{\|u\|_2=1}\left |\sum_{k=1}^{m_1} a_k \|w_k\|_2\sum_{i=1}^n \rho_i  \sigma (u^Tx_i) \right | \\
                 & \le F\mathbb{E}_\rho \sup_{\|u\|_2=1}\left |\sum_{i=1}^n \rho_i \sigma(u^Tx_i) \right | \\
                 & \le 2cL_\sigma F \sqrt{dn}
\end{align}
where the third  equality holds because there is a maximizer $u$ of $\sum_{i=1}^n \rho_i  \sigma (u^Tx_i)$  over unit sphere, and then one can change the sign of all $a_k$ appropriately so that they become all positive or negative. This will not change $\mu(\theta)$. 
And the last inequality  follows from inequality (4) of Theorem 12 in~\citep{2001Rademacher} and Lemma~\ref{gen_cor0}.

When $L=3$, $\mathbf{m} = (m_1,m_2)$, we do manipulation
\begin{align}
\mathbb{E}_\rho & \sup_{f\in \mathcal{F}(\mathbf{m},F)} \left | \sum_{i=1}^n \rho_i \sum_{k_2=1}^{m_2}w^3_{k_2}\sigma (w^{2}_{k_2k_1}\sum_{k_1=1}^{m_1} \sigma ((w^1_{k_1})^Tx_i) \right | \\
& = \mathbb{E}_\rho  \sup_{\nu(\theta)\le F} \left | \sum^n_{i=1}\rho_i \sum_{k_2=1}^{m_2}w^3_{k_2}\sigma (\sum_{k_1=1}^{m_1} w^{2}_{k_2,k_1}\|w^1_{k_1}\|_2\sigma ((u_{k_1})^Tx_i) \right |, \quad \|u_{k_1}\|_2 = 1 \\
& = \mathbb{E}_\rho  \sup_{\nu(\theta)\le F} \left | \sum_{i=1}^n \rho_i \sum_{k_2=1}^{m_2} w^3_{k_2}\|w^2_{k_2,k_1}\|_1\|w^1_{k_1}\|_2 \sigma (\sum_{k_1=1}^{m_1} v_{k_2,k_1} \sigma (u_{k_1}^Tx_i)) \right |, \quad \|v_{k_2,k_1}\|_1 = 1 \\
& = \mathbb{E}_\rho  \sup_{\nu(\theta)\le F} \left | \sum_{k_2=1}^{m_2} w^3_{k_2}\|w^2_{k2,k1}\|_1\|w^1_{k_1}\|_2 \sum_{i=1}^n \rho_i \sigma (\sum_{k_1=1}^{m_1} v_{k_2,k_1} \sigma (u_{k_1}x_i)) \right | \\
\end{align}
One can then deduce that $v_{k_2,k_1}$ does not depend on $k_2$ by using the same argument as for $L=2$ case for the deduction that $u_k$ does not depend on $k$. So we abuse the notation a little bit to write $v_{k_2,k_1}$ as $v_{k_1}$  and continue from the last line of the above expression
\begin{align}
 & \mathbb{E}_\rho   \sup_{\nu(\theta)\le F} \left | \sum_{k_2=1}^{m_2} w^3_{k_2}\|w^2_{k2,k1}\|_1\|w^1_{k_1}\|_2 \sum_{i=1}^n \rho_i \sigma (\sum_{k_1=1}^{m_1} v_{k_2,k_1} \sigma (u_{k_1}x_i)) \right | \\ 
 & = \mathbb{E}_\rho  \sup_{\nu(\theta)\le F} \left | \sum_{k_2=1}^{m_2} w^3_{k_2}\|w^2_{k2,k1}\|_1\|w^1_{k_1}\|_2 \sup_{\|v_{k_1}\|_1=1}\sum_{i=1}^n  \rho_i \sigma (\sum_{k_1=1}^{m_1} v_{k_1} \sigma (u_{k_1}x_i)) \right | \\
& \le F \mathbb{E}_\rho \sup_{\nu(\theta)\le F}  \sup_{\|v_{k_1}\|_1=1} \left | \sum_{i=1}^n \rho_i \sigma (\sum_{k_1=1}^{m_1} v_{k_1} \sigma (u_{k_1}x_i)) \right | \\
& \le 2L_\sigma F \mathbb{E}_\rho  \sup_{\nu(\theta)\le F} \sup_{\|v_{k_1}\|_1=1} \left | \sum_{i=1}^n \rho_i (\sum_{k_1=1}^{m_1} v_{k_1} \sigma (u_{k_1}x_i)) \right | \\
& = 2L_\sigma F \mathbb{E}_\rho \sup_{\nu(\theta)\le F} \sup_{\|v_{k_1}\|_1=1} \left | \sum_{k_1=1}^{m_1} v_{k_1} \sum_{i=1}^n \rho_i \sigma (u_{k_1}x_i)  \right |
\end{align}
Again, one can adjust the sign of $v_{k_1}$ to make $u_{k_1}$ independent of $k_1$. So
\begin{align}
& = 2L_\sigma F \mathbb{E}_\rho  \sup_{\nu(\theta)\le F} \sup_{\|v_{k_1}\|_1=1} \left | \sum_{k_1=1}^{m_1} v_{k_1} \sum_{i=1}^n \rho_i \sigma (u_{k_1}x_i) \right |\\
& = 2L_\sigma F \mathbb{E}_\rho  \sup_{\|v_{k_1}\|_1=1}  \sup_{\|u\|_2=1} \left | \sum_{k_1=1}^{m_1} v_{k_1} (\sum_{i=1}^n \rho_i \sigma (ux_i)) \right | \\
& = 2L_\sigma F \mathbb{E}_\rho  \sup_{\|v_{k_1}\|_1=1}  \sup_{\|u\|_2=1} \left | \sum_{k_1=1}^{m_1} v_{k_1} \right | \left | \sum_{i=1}^n \rho_i \sigma (ux_i) \right | \\
& \le 2L_\sigma F \mathbb{E}_\rho   \sup_{\|u\|_2=1} \left | \sum_{i=1}^n \rho_i \sigma (ux_i) \right |  \label{88} \\ 
& \le 4cL_\sigma^2F\sqrt{dn}   \label{89}
\end{align}

The preceding argument can be easily generalized to general $L\ge 2$ cases. We omit the proof here because it only involves lengthy and tedious symbols, but the idea is  completely straightforward.
\end{proof}

To prove Theorem~\ref{overpara_gen}, we need a preliminary result. This Lemma is a \textit{uniform functional concentration inequality} that bounds the maximum of the $L^2$ norm of a family of functions in terms of the empirical $L^2$ norm. This key result will be used to link empirical errors and generalization errors.

 \begin{lemma}\label{f_concentration}
Assume that Assumptions~\ref{assump1}, \ref{assump2}, and \ref{assump3} hold. Let $\mathbf{m} = (m_1,\dots,m_{L-1})$, $F^*(\mathbf{m},1) = \{f - f^* \mid f\in \mathcal{F}(\mathbf{m},1), f^* \text{ is fixed with } \|f^*\|_{W^L} \le 1\}$, and let
\begin{align}
Z_n = \sup_{f\in \mathcal{F}^*(\mathbf{m},1)} \left| \|f\|_n^2 - \|f\|_2^2 \right|.
\end{align}
Then $\mathbb{E} Z_n \le C_\mathcal{F} n^{-1/2}$ for some constant $C_\mathcal{F} > 0$ (may depend on $L$ and $L_\sigma$). Furthermore, if $n \ge C_\mathcal{F}^2$, then
\begin{align}
P\left( Z_n \ge \frac{C_\mathcal{F}}{\sqrt{n}} + t \right) \le \exp\left( -\frac{n}{32} \min\left( \frac{t^2}{12e}, t \right) \right).
\end{align}
\end{lemma}

\begin{proof}[Proof of Lemma~\ref{f_concentration}] By a standard symmetrization argument,
\begin{align}
\mathbb{E}Z_n \le 2 \mathbb{E}_{\rho,x} \sup_{f\in \mathcal{F}^*(\mathbf{m},1)} \left | {1 \over n}\sum_{i=1}^n \rho_if^2(x_i)  \right |
\end{align}
where $\rho_i$ are independent Rademacher variables. Since $\phi(x)=x^2$ is 2-Lipschitz continuous for $x\in[-1,1]$ and it is easy to see that $\sup_{f\in \mathcal{F}^*(\mathbf{m},1)}\sup_{x\in \mathbf{B}^d}|f(x)| \le 2$, by Lemma 26.9 of~\citep{shalev2014understanding} we have
\begin{align}
\mathbb{E}Z_n \le 2 \mathbb{E}_{\rho,x}\sup_{f\in \mathcal{F}^*(\mathbf{m},1)}\left |{1\over n}\sum_{i=1}^n \rho_i \phi (f(x_i))   \right | \le 16 \mathbb{E}_{\rho,x}\sup_{f\in \mathcal{F}^*(\mathbf{m},1)}\left |{1 \over n}\sum_{i=1}^n \rho_i f(x_i)  \right |
\end{align}
Let $\Tilde{f}\in \mathcal{F}(\Tilde{\mathbf{m}},1)$ be the $L$-layer neural network that best approximates $f^*$ under the $L^2(\mathbf{B}^d)$-norm in Theorem~\ref{approx1_2}, where $\Tilde{m} \ge n$ elementwise. Thus, $\|\Tilde{f} - f^*\|_{L^2(\mathbf{B}^d)} \le C_1/\sqrt{n}$ for some constant $C_1>0$ depending on $f^*$. By decomposing $f=f - \Tilde{f} + \Tilde{f}$ and noting that $f - \Tilde{f} \in \mathcal{F}(\mathbf{m}+\Tilde{\mathbf{m}},2)$, we obtain
\begin{align}
\mathbb{E}Z_n & \le 16 \mathbb{E}_{\rho,x}\sup_{f\in \mathcal{F}(\mathbf{m},1)}\left |{1 \over n}\sum_{i=1}^n \rho_i(f(x_i)-\Tilde{f}(x_i))   \right | + 16\mathbb{E}_{\rho,x}\left | {1\over n}\sum_{i=1}^n \rho_i(\Tilde{f}(x_i) - f^*(x_i))  \right | \\
& \le 16 \mathbb{E}_{\rho,x}\sup_{f\in \mathcal{F}(\mathbf{m}+\Tilde{\mathbf{m}},2)} \left | {1\over n}\sum_{i=1}^n\rho_i f(x_i)  \right | + {16\over n}\sum_{i=1}^n \sqrt{\mathbb{E}_\rho  \rho_i^2}\|\Tilde{f} - f^*\|_2 \\
& \le 16\mathbb{E}_{\rho,x}\sup_{f\in \mathcal{F}(\mathbf{m}+\Tilde{\mathbf{m}},2)} \left |{1\over n}\sum_{i=1}^n\rho_if(x_i)  \right | + {16C_1\over \sqrt{n}} \le {32c2^{L-1}L_\sigma^{L-1}\sqrt{d} + 16C_1 \over \sqrt{n}} \\
& \equiv {C_\mathcal{F} \over \sqrt{n}}
\end{align}
where the last inequality follows from Lemma~\ref{gen_lem1}.
Define
\begin{align}
U = \sup_{x\in \mathbf{B}^d}\sup_{f\in \mathcal{F}^*(\mathbf{m},1)} |f(x)|^2, \xi^2 = \sup_{f\in \mathcal{F}^*(\mathbf{m},1)}\mathbb{E}|f(x)|^4, K_n = 2U\mathbb{E}Z_n + \xi^2
\end{align}
and note that for $\sqrt{n}\ge C_\mathcal{F}$,
\begin{align}
U \le 2sup_{x\in \mathbf{B}^d}\sup_{f\in \mathcal{F}(\mathbf{m},1)}(|f(x)|^2 + |f^*(x)|^2) \le 4, \xi^2\le U^2 \le 16, K_n\le {8C_\mathcal{F}\over \sqrt{n}} + 16 \le 24
\end{align}
By Talagrand's concentration inequality~\citep{wainwright2019highdimensional},
\begin{align}
P(Z_n - \mathbb{E}Z_n \ge t) \le 2exp\left (-{nt^2 \over 8eK_n+4Ut}  \right ) \le 2exp\left (-{nt^2 \over 192e+16t} \right )
\end{align}
Note that $-nt^2/(192e+16t) \le -nt/32$ if $t\ge 12e$, and $-nt^2/(192e+16t) \le -nt^2 /(384e)$ otherwise. We then conclude that
\begin{align}
P\left (Z_n \ge {C_\mathcal{F} \over \sqrt{n}} + t   \right ) \le exp\left \{ -{n \over 32}\min({t^2\over 12e},t) \right \}
\end{align}
\end{proof}
\begin{proof}[Proof of Theorem~\ref{overpara_gen}] Let $\hat{f}(\cdot) = g(\cdot ; \hat{\theta})$ and $\hat{\Delta} = \hat{f} - f^*$. By the proof in~\ref{overpara_emp} and, in particular, ~\eqref{empirical_e}, if we choose $\lambda = \max\{8L^{L-1}_\sigma, 2^LcL_\sigma^{L-1}\sqrt{d}\}\sqrt{\frac{f(n)}{n}}$ 
then, with probability at least $1-n^{-4}$,
\begin{align}
0 \le \|\hat{f} - f^*\|^2_n \le C_1H^2(\mathbf{m})\|f^*\|^2_{W^L} + 4\lambda \nu(\theta^*) - \lambda(\nu(\hat{\theta}) -\nu(\theta^*)) 
\end{align}
for some constant $C_1>0$. Since $\nu(\theta^*) \le \|f^*\|_{W^L}$, we further obtain
\begin{align}
\lambda \nu(\hat{\theta}) \le 5\lambda \nu(\theta^*) + C_1H^2(\mathbf{m})\|f^*\|^2_{W^L}
\end{align}
If
\begin{align}
H(\mathbf{m}) \le \sqrt{\frac{\max\{8L^{L-1}_\sigma, 2^{L}cL_\sigma^{L-1}\sqrt{d}\}}{C_1}}\sqrt[4]{\dfrac{f(n)}{n}}
\end{align}
then
\begin{align}
\nu(\hat{\theta}) \le 5\nu(\theta^*) + \|f^*\|_{W^L} \le 6\|f^*\|_{W^L}
\end{align}
By~\eqref{s_path_norm} and the homogeneity of ReLU, the path enhanced scaled variation norm of $\hat{f}/\nu(\hat{\theta})$ is exactly 1. Also, by definition, the $W^L$-norm of $f^*/(6\|f^*\|_{W^L})$ is smaller than 1. Thus, the event
\begin{align}
    {\hat{\Delta}\over 6\|f^*\|_{W^L}} = {\hat{f}\over 6\|f^*\|_{W^L}} - {f^*\over 6\|f^*\|_{W^L}} \in \mathcal{F}^*(\mathbf{m},1)
\end{align}
holds with probability at least $1-n^{-4}$.

Now, conditioning on the event $\{\hat{\Delta}/(6\|f^*\|_{W^L})\in \mathcal{F}^*(\mathbf{m},1)   \}$, applying Lemma~\ref{f_concentration} with $t=8\sqrt{6ef(n)/n} < 12e$ yields
\begin{align}
\|\hat{\Delta}\|^2_2 \le \|\hat{\Delta}\|^2_n + {36 \over \sqrt{n}}C_\mathcal{F}\|f^*\|^2_{W^L} + 288\|f^*\|^2_{W^L}\sqrt{{6ef(n)\over n}}
\end{align}
with probability at least $1-n^{-1}$. By~\ref{overpara_emp}, with probability at least $1-n^{-4}$ we have 
\begin{align}
\|\hat{\Delta}\|^2_n = \|\hat{f} - f^*\|^2_n \le C_3 & \left \{ H^2(\mathbf{m})\|f^*\|^2_{W^L} \right.\\
&\left. + \max\{16L^{L-1}_\sigma, 2^{L+1}cL_\sigma^{L-1}\sqrt{d}\}(\sigma^2_{\epsilon} + \|f^*\|^2_{W^L}) \sqrt{\frac{f(n)}{n}}  \right \}
\end{align}
for some constant $C_3>0$. Combining these pieces, we conclude that
\begin{align}
\|\hat{f} - f^*\|^2_2 \le C_4 & \left \{H^2(\mathbf{m}) \|f^*\|^2_{W^L} \right.\\
&\left. + \max\{16L^{L-1}_\sigma, 2^{L+1}cL_\sigma^{L-1}\sqrt{d}\}(\sigma^2_{\epsilon} + \|f^*\|^2_{W^L})\sqrt{\frac{f(n)}{n}}  \right \}
\end{align}
with probability at least $1-O(n^{-1})$ for some constant $C_4 > 0$.

The proof for~\eqref{gen_exp} is totally analogous to the one in~\citep{huiyuan2023Nonasymptotic}, to reduce the size of this paper we refer the readers to the proof therein for details.
\end{proof}
\subsection{Proof on results in section~\ref{sec7}}
\begin{proof}[Proof of Lemma~\ref{inf-metric-entropy}] Let's make an induction on $L$. 
$L=2$ is established in~\citep{huiyuan2023Nonasymptotic}. We take their proof for motivation and clarification of the ideas.

Let $\mathbf{m}={m}$, and let $g(\cdot ;\theta_1), g(\cdot ; \theta_2) \in \mathcal{F}(m,1)$ be two two-layer networks, such that $\theta_1 = (a^1_1,\dots,a^1_m, (w^1_1)^T, \dots, (w^1_m)^T)^T$ and $\theta_2 = (a^2_1,\dots,a^2_m,(w^2_1)^T,\dots,(w^2_m)^T)^T$. We can assume without loss of generality that $\|w^j_i\|_2 = 1$ for all $i,j$, in which case $g(x;\theta_j) \in \mathcal{F}(m,1)$ is equivalent to $\sum_{k=1}^m |a^j_k|\le 1$ for both $j$.  We then have
\begin{align}
& |g(\Tilde{x};\theta_1)-g(\Tilde{x};\theta_2)| \\
&  = |\sum_{k=1}^ma^1_k\sigma(\Tilde{x}^Tw^1_k) - \sum_{k=1}^ma^2_k\sigma(\Tilde{x}^Tw^2_k)| \\
& \le |\sum_{k=1}^m(a^1_k-a^2_k)\sigma(\Tilde{x}^Tw^1_k)| + |\sum_{k=1}^ma^2_k(\sigma(\Tilde{x}^Tw^1_k) - \sigma(\Tilde{x}^Tw^2_k))| \\
& \le \sqrt{2}L_\sigma\sum_{k=1}^m|a^1_k-a^2_k\|\Tilde{x}^Tw^1_k| + \sqrt{2}L_\sigma\sum_{k=1}^m|a^2_k|\max_{1\le k\le m}\|w^1_k-w^2_k\|_2 \\
& \le \sqrt{2}L_\sigma\sum_{k=1}^m|a^1_k - a^2_k| + \sqrt{2}L_\sigma \max_{1\le k \le m}\|w^1_k-w^2_k\|_2
\end{align}
Such coefficient is due to $\|\Tilde{x}\|_2=\sqrt{2}$. We denote the unit $l_1$-ball in $R^n$ by $B^n_1(1)$. To cover $\mathcal{F}(m,1)$ with respect to $\|\cdot\|_{\infty}$, we need only cover $B^m_1(1)$ with respect to $\|\cdot\|_1$ and $m$ many $\mathbf{B}^d$ with respect to $\|\cdot\|_2$ simultaneously. The volume argument in Lemma 5.7 of~\citep{wainwright2019highdimensional} yields
\begin{align}
\mathcal{N}(\delta, B^m_1,\|\cdot\|_1)\le (1+2\delta^{-1})^m, \mathcal{N}(\delta, \mathbf{B}^d,\|\cdot\|_2) \le (1+2\delta^{-1})^d
\end{align}
that results in 
\begin{align}
\log\mathcal{N}(\delta, \mathcal{F}(m,1),\|\cdot\|_{\infty}) \le (d+1)mlog(1+4\sqrt{2}L_\sigma\delta^{-1}).
\end{align}
When $L=3$, letting $\mathbf{m} = \{m_1,m_2\}$ and using the notation above, now let 
\begin{align}
\theta_1 & = (a^1_1,\dots,a^1_{m_2}, b^1_{1,1},\dots, b^1_{m_2,m_1}, (w^1_{1})^T,\dots, (w^1_{m_1})^T)^T \\
\theta_2 & = (a^2_1,\dots,a^2_{m_2}, b^2_{1,1},\dots, b^2_{m_2,m_1}, (w^2_{1})^T,\dots, (w^2_{m_1})^T)^T. 
\end{align}
By changing the scales, we can also assume without loss of generality that  $\|w^j_i\|_2 = 1$ for all $i,j$. By changing the scales further, we can assume $\sum_{i_2=1}^{m_1} |b^j_{i_1,i_2}|=1$ for all $i_1,j$. Thus $g(x;\theta_j) \in \mathcal{F}(\mathbf{m},1)$ is equivalent to $\sum_{i=1}^{m_2}|a^j_i| \le 1$ for both $j$.
We then have
\begin{align}
& |g(\Tilde{x};\theta_1)-g(\Tilde{x};\theta_2)| \\
& = |\sum_{k=1}^{m_2}a^1_k\sigma (\sum_{j=1}^{m_1}b^1_{k,j}\sigma (\Tilde{x}^Tw^1_j)) - \sum_{k=1}^{m_2}a^2_k\sigma (\sum_{j=1}^{m_1}b^2_{k,j}\sigma (\Tilde{x}^Tw^2_j))| \\
& \le   |\sum_{k=1}^{m_2}a^1_k\sigma (\sum_{j=1}^{m_1}b^1_{k,j}\sigma (\Tilde{x}^Tw^1_j)) - \sum_{k=1}^{m_2}a^2_k\sigma (\sum_{j=1}^{m_1}b^1_{k,j}\sigma (\Tilde{x}^Tw^1_j))| \\
& + |\sum_{k=1}^{m_2}a^2_k\sigma (\sum_{j=1}^{m_1}b^1_{k,j}\sigma (\Tilde{x}^Tw^1_j)) - \sum_{k=1}^{m_2}a^2_k\sigma (\sum_{j=1}^{m_1}b^2_{k,j}\sigma (\Tilde{x}^Tw^1_j))| \\
& + |\sum_{k=1}^{m_2}a^2_k\sigma (\sum_{j=1}^{m_1}b^2_{k,j}\sigma (\Tilde{x}^Tw^1_j)) - \sum_{k=1}^{m_2}a^2_k\sigma (\sum_{j=1}^{m_1}b^2_{k,j}\sigma (\Tilde{x}^Tw^2_j))|\\
& \le L_\sigma \sum_{k=1}^{m_2}|a^1_k-a^2_k| \sum_{j=1}^{m_1}|b^1_{k,j}||\sigma (\Tilde{x}^Tw^1_j)| \\
& + L_\sigma \sum_{k=1}^{m_2}|a^2_k|\sum_{j=1}^{m_1}|b^1_{k,j} - b^2_{k,j}||\sigma (\Tilde{x}^Tw^1_j)| \\
& + L_\sigma \sum_{k=1}^{m_2}|a^2_k| \sum_{j=1}^{m_1}|b^2_{k,j}| \|\Tilde{x}^Tw^1_j - \Tilde{x}^Tw^2_j\|_2 \\
& \le \sqrt{2}L_\sigma \sum_{k=1}^{m_2}|a^1_k-a^2_k| +  \sqrt{2}L_\sigma \sum_{k=1}^{m_2}\sum_{j=1}^{m_1}|b^1_{k,j} - b^2_{k,j}| +  \sqrt{2}L_\sigma \max_{1\le j \le m_1}\|w^1_j - w^2_j\|_2 \\
\end{align}
To cover $\mathcal{F}(\mathbf{m},1)$ with respect to $\|\cdot\|_{\infty}$, we need only cover $B^{m_2}_1(1)$ with respect to $\|\cdot\|_1$, $m_2$ number of $B^{m_1}_1(1)$ with respect to $\|\cdot\|_1$ and $m_1$ many $\mathbf{B}^d$ with respect to $\|\cdot\|_2$ simultaneously. Again, using volume argument we yield

\begin{align}
\log\mathcal{N}(\delta, \mathcal{F}(m,1),\|\cdot\|_{\infty}) \le (dm_1+m_1m_2+m_2)\log(1+4\sqrt{2}L_\sigma\delta^{-1}).
\end{align}
The above argument can be readily generalized  to general $L$ without much difficulty.  Therefore,  for $\mathbf{m} = (m_1,m_2,\dots, m_{L-1})$, we obtain 
\begin{align}
\log\mathcal{N}(\delta, \mathcal{F}(\mathbf{m},1),\|\cdot\|_{\infty}) \le (dm_1+m_1m_2+m_2m_3+\cdots+ m_{L-1})\log(1+4\sqrt{2}L_\sigma\delta^{-1}).
\end{align}
\end{proof}
\begin{proof}[Proof of Theorem~\ref{underpara_emp}] It remains to bound $\mathcal{N}_{\mathbf{m}+\mathbf{m}}(\delta_n) \equiv \mathcal{N}(\delta_n,\mathcal{F}(\mathbf{m}+\mathbf{m},1),\|\cdot\|_n)$. By Lemma~\eqref{inf-metric-entropy}, 
\begin{align}
\log\mathcal{N}_{\mathbf{m}+\mathbf{m}}(\delta_n) & \le \log \mathcal{N}(\delta_n, \mathcal{F}(\mathbf{m}+\mathbf{m},1),\|\cdot\|_\infty) \\
& \le (2dm_1+4m_1m_2+4m_2m_3+\cdots+ 2m_{L-1})\log(1+4\sqrt{2}L_\sigma\delta^{-1}) 
\end{align}
Then we choose $\delta_n = n^{-1}L_\sigma(2dm_1+4m_1m_2+4m_2m_3 + \cdots + 2m_{L-1})dlogn$, and take 
$\Tilde{p}=n^{2L_\sigma(2dm_1+4m_1m_2+4m_2m_3 + \cdots + 2m_{L-1})}$. And we go over the proof of Theorem S.1 in~\citep{huiyuan2023Nonasymptotic}
with these new expressions. Everything in the proof will hold and some constants in Theorem depend on $L$ and $L_\sigma$   now. To reduce the size of this paper, we omit the complete proof, the reader can look closely into supplementary materials in~\citep{huiyuan2023Nonasymptotic}.

Similar to the proof of the Theorem of the generalization error in the overparameterized regime~\ref{overpara_gen}, we bound $T_1,T_2$ and $T_3$. We recall that $g(x;\hat{\theta}\ominus\theta^*)$ is a $L$-layer network with widths at most $\mathbf{m}+\mathbf{m}$. 
We still take the following bounds for $T_1$ and $T_2$
\begin{align}
    T_1 &\le 2\lambda \nu(\theta^*) \label{underpar_T1} \\
    T_2 \label{underpar_T2} &\le C_1 H(\mathbf{m})^2\|f^*\|_{W^L}^2m^{-1} 
\end{align}
for some constant $C_1 > 0$. Define $\hat{\sigma_\epsilon}=\sqrt{n^{-1}\sum_{i=1}^n\epsilon_i^2}$. For $T_3$, since $g(x;\hat{\theta}\ominus\theta^*)/\nu(g(x;\hat{\theta}\ominus\theta^*))\in \mathcal{F}(\mathbf{m}+\mathbf{m},1)$, we obtain
\begin{align}
    \frac{T_3-\delta_n\nu(\Delta^*)\hat{\sigma_\epsilon}}{\|\Delta^*\|_n + \delta_n\nu(\Delta^*)} = \frac{n^{-1}|\sum_{i=1}^n\epsilon_i\Delta^*(x_i)/\nu(\Delta^*)|-\delta_n\hat{\sigma_\epsilon}}{\|\Delta^*/\nu(\Delta^*)\|_n+\delta_n} \le V_{\delta_n}(\epsilon)
\end{align}
Noting that $V_{\epsilon_n}(\epsilon)$ is a Lispchitz continuous function of independent Gaussian variables and applying Theorem 2.26 in~\cite{wainwright2019highdimensional} yields
\begin{align}
    \mathbb{P}(|V_{\delta_n}(\epsilon) - \mathbb{E}V_{\delta_n}(\epsilon)|\ge t) \le 2 exp\{-\frac{nt^2}{2} \}
\end{align}
Similar to Lemma S1 in~\cite{huiyuan2023Nonasymptotic}, we have $\mathbb{E}V_{\delta_n}(\epsilon) \le 2\sigma_\epsilon\sqrt{\log\mathcal{N}_{\mathbf{m}+\mathbf{m}}(\delta_n)/n}$. This is a straightforward generalization and we omit the proof. Choosing $t=2\sigma_\epsilon\sqrt{\log\Tilde{p}/n}$ for some $\Tilde{p}\ge \mathcal{N}_{2m}(\delta_n)$ to be specified later, we have, with probability at least $1-2\Tilde{p}^{-2\sigma_\epsilon^2}$,
\begin{align}
    V_{\delta_n}(\epsilon) < 4\sigma_\epsilon\sqrt{\frac{\log\Tilde{p}}{n}}
\end{align}
Similarly, $\mathbb{P}(\hat{\epsilon} \ge \sigma_\epsilon + t) \le \exp(-nt^2/2)$ as $n^{-1/2}\|\epsilon\|_2$ is also $n^{-1/2}$-Lipschitz continuous and $n^{-1/2}\mathbb{E}\|\epsilon\|_2\le \sqrt{n^{-1}\mathbb{E}\epsilon^T\epsilon} = \sigma_\epsilon$. Choosing $t=\sigma_\epsilon$, we have, with probability at least $1-\exp(-n\sigma_\epsilon^2/2)$,
\begin{align}
    \hat{\sigma_\epsilon} < 2\sigma_\epsilon
\end{align}
Combining these pieces gives
\begin{align}\label{underpar_T3}
    T_3 \le 4\sigma_\epsilon \sqrt{\frac{\log\Tilde{p}}{n}}(\|\Delta^*\|_n + \delta_n\nu(\Delta^*)) + 2\sigma_\epsilon\delta_n\nu(\Delta^*)
\end{align}
with probability at least $1-2\Tilde{p}^{-2\sigma_\epsilon^2}-\exp(-\sigma_\epsilon^2n/2)$. Furthermore, combining the estimation of $T_1$,$T_2$ and $T_3$~\eqref{underpar_T1},~\eqref{underpar_T2} and~\eqref{underpar_T3} yields
\begin{align}
    \frac{1}{2}\|g(\cdot ;\hat{\theta}) - f^*\|_n^2 & \le C_1H(\mathbf{m})^2 \|f^*\|_{W^L}m^{-1} + 4\sigma_\epsilon\sqrt{\frac{\log\Tilde{p}}{n}}\|\Delta^*\|_n \\
    & +\{(2\sqrt{\frac{\log\Tilde{p}}{n}}+1)2\sigma_\epsilon\delta_n - \lambda  \}\nu(\Delta^*) + 2\lambda \nu(\theta^*)
\end{align}
Choosing $\lambda \ge (2\sqrt{n^{-1}\log\Tilde{p}}+1)4\sigma_\epsilon\delta_n$, we have 
\begin{align}~\label{underpara_emp_proof_1}
    \frac{1}{2}\|g(\cdot ;\hat{\theta})-f^*\|_n^2 & \le C_1H(\mathbf{m})^2\|f^*\|_{W^L}^2m^{-1} + 2\lambda\nu(\theta^*) \\
    & +4\sigma_\epsilon\sqrt{\frac{\log\Tilde{p}}{n}}(\|g(\cdot ; \hat{\theta}) - f^* \|_n + \|g(\cdot ;\theta^*)-f^*\|_n
\end{align}
where we have used the triangle inequality to bound $\|\Delta^*\|_n$. Using the inequality $ab\le a^2 + b^2/4$, we obtain
\begin{align}~\label{underpara_emp_proof_2}
    4\sigma_\epsilon\sqrt{\frac{\log\Tilde{p}}{n}}\|g(\cdot ;\hat{\theta})-f^*\|_n \le 16\sigma_\epsilon^2\frac{\log\Tilde{p}}{n} + \frac{1}{4}\|g(\cdot;\hat{\theta}) - f^*\|_n^2 \\
    4\sigma_\epsilon\sqrt{\frac{\log\Tilde{p}}{n}}\|g(\cdot ;{\theta^*})-f^*\|_n \le 16\sigma_\epsilon^2\frac{\log\Tilde{p}}{n} + \frac{1}{4}\|g(\cdot;{\theta^*}) - f^*\|_n^2
\end{align}
Substituting~\eqref{underpara_emp_proof_2} into~\eqref{underpara_emp_proof_1} and noting that $\nu(\theta^*)\le \|f^*\|_{W^L}$ yields
\begin{align}
    \|g(\cdot ; \hat{\theta}) - f^* \|_n^2 \le 6C_1H(\mathbf{m})^2\|f^*\|_{W^L}^2m^{-1} + 128\sigma_\epsilon^2\frac{\log\Tilde{p}}{n} + 8\lambda\|f^*\|_{W^L}
\end{align}
with probability at least $1-2\Tilde{p}^{-2\sigma_\epsilon^2} - \exp(-\sigma_\epsilon^2n/2)$.
\end{proof}

It remains to bound $\mathcal{N}_{\mathbf{m}+\mathbf{m}}(\delta_n) \equiv \mathcal{N}(\delta_n,\mathcal{F}(\mathbf{m}+\mathbf{m},1), \|\cdot\|_n)$. Since a $\delta_n$-covering of $\mathcal{F}(\mathbf{m}+\mathbf{m},1)$ with respect to $\|\cdot\|_{\infty}$ is always a $\delta_n$-covering with respect to $\|\cdot\|_n$, by Lemma~\ref{inf-metric-entropy} we have
\begin{align}
    \log\mathcal{N}_{\mathbf{m}+\mathbf{m}}(\delta_n) & \le \log \mathcal{N}(\delta_n, \mathcal{F}(\mathbf{m}+\mathbf{m},1),\|\cdot\|_\infty) \\
& \le (2dm_1+4m_1m_2+4m_2m_3+\cdots+ 2m_{L-1})\log(1+4\sqrt{2}L_\sigma\delta^{-1}) 
\end{align}
Recall that $\delta_n =  n^{-1}L_\sigma(4m_1m_2+4m_2m_3 + \cdots + 2m_{L-1})dlogn < 1$, and we have
\begin{align}
    \log(1+4\sqrt{2}L_\sigma\delta^{-1}) \le \log(1+L_\sigma\frac{4\sqrt{2}n}{2dm_1+4m_1m_2+4m_2m_3 + \cdots + 2m_{L-1}}) \le 2log(nL_\sigma)
\end{align}
Now take $\Tilde{p} = (nL_\sigma)^{2(2dm_1+4m_1m_2+4m_2m_3+\cdots+ 2m_{L-1})}$, and by the assumption that $\delta_n \le 1$ we have
\begin{align}
    \sqrt{\frac{\log\Tilde{p}}{n}} \le \sqrt{\frac{2(2dm_1+4m_1m_2+4m_2m_3+\cdots+ 2m_{L-1})\log(nL_\sigma)}{n}} \le 2
\end{align}
when $n$ is sufficient large. In order for $\lambda \ge (2\sqrt{\log\Tilde{p}/n}+1)4\sigma_\epsilon\delta_n$ to hold, setting $\lambda = 20\sigma_\epsilon \max(\delta_n,H(\mathbf{m}))$ is sufficient. With this choice of $\lambda$, we conclude that~\ref{underpara_emp} holds with probability at least $1-2\Tilde{p}^{-2\sigma_\epsilon^2} -\exp(-\sigma_\epsilon^2n/2)$. This completes the proof of~\ref{underpara_emp} by noting that
\begin{align}
    -2\sigma_\epsilon^2log\Tilde{p} = -4\sigma_\epsilon^2((2dm_1+4m_1m_2+4m_2m_3+\cdots+ 2m_{L-1}))\log(nL_\sigma) \le -4\sigma_\epsilon^2log(nL_\sigma)
\end{align}
and $\exp(-\sigma_\epsilon^2(nL_\sigma)/2) = o(n^{-C_2})$ for any constant $C_2 > 0$.

\begin{lemma}[Restatement of Lemma~\ref{underpara_gen_lem1}]
For any $0 < \gamma < 1$, define $\mathcal{B}_\mathcal{F}(\gamma) = \{ f\in \mathcal{F}^*(\mathbf{m},1) : \|f\|_2 < \gamma \}$, the $L^2(\mu)$-ball in $\mathcal{F}^*(\mathbf{m},1)$ of radius smaller than $\gamma$. Let $Z_n(\gamma) = \sup_{f\in \mathcal{B}_\mathcal{F}(\gamma)} \left| \|f\|_n^2 - \|f\|_2^2 \right|$. Then, for any $\gamma$ satisfying
\begin{align}
\sqrt{ \frac{2\log(18L_\sigma)(d m_1 + m_1 m_2 + m_2 m_3 + \cdots + m_{L-1}) \log n}{n} } \le \gamma \le 1,
\end{align}
we have
\begin{align}
\mathbb{E} Z_n(\gamma) \le 272 \gamma \sqrt{ \frac{\log(18L_\sigma)(d m_1 + m_1 m_2 + m_2 m_3 + \cdots + m_{L-1}) \log n}{n} }.
\end{align}
\end{lemma}

The proof is to reduce to the estimation of the Local Rademacher complexity, and again, the metric entropy argument is also an important ingredient. The full proof is more or less the same as in~\citet{huiyuan2023Nonasymptotic}, with the exception that we have new expressions for the lower bound assumption on $\gamma$. This is caused by the calculation of the number of $1/n$-coverings of $\mathcal{B}_{\mathcal{F}}(\gamma)$.

\begin{proof}[Proof of Lemma~\ref{underpara_gen_lem1}]
It remains to bound the $(1/n)$-covering of $\mathcal{B}_{\mathcal{F}}(\gamma)$ with respect to $L_{\infty}(\mathbf{B}^d)$-norm. By the deductions in~\citep{huiyuan2023Nonasymptotic}, we get
\begin{align}
\log M & \le \log \mathcal{N}(1/(2n)^{3/2},\mathcal{F}(\mathbf{m},1),\|\cdot\|_{\infty}) \\
& \le (dm_1+m_1m_2+m_2m_3+\cdots+ m_{L-1})\log(1+4\sqrt{2}L_\sigma (2n)^{3/2}) \\
& \le (dm_1+m_1m_2+m_2m_3+\cdots+ m_{L-1})\log(18L_\sigma (n)^{3/2}) \\
& \le  (dm_1+m_1m_2+m_2m_3+\cdots+ m_{L-1})\log(18L_\sigma (n)^{3/2}) \\
& \le  (dm_1+m_1m_2+m_2m_3+\cdots+ m_{L-1})\log(18L_\sigma) +  (dm_1+m_1m_2+m_2m_3+\cdots+ m_{L-1})3/2log(n) \\
& \le 4(dm_1+m_1m_2+m_2m_3+\cdots+ m_{L-1})\log(18L_\sigma)\log n 
\end{align}
Then we go over the proof of Lemma S.2 in~\citep{huiyuan2023Nonasymptotic} and we complete. Some constants in the Theorem depend on $L$ and  $L_\sigma$  now. For self-contained purpose, we provide the complete proof. 

First note that $\sup_{x\in \mathbf{B}^d}\sup_{f\in \mathcal{F}^*(\mathbf{m},1)} |f(x)| \le 2$. By a standard symmetrization argument, we have
\begin{align}
    \mathbb{E}Z_n(\gamma) \le 16\mathbb{E}_{\rho,x}\sup_{f\in \mathcal{B}_\mathcal{F}(\gamma)} \frac{1}{n}|\sum_{i=1}^n\rho_if(x_i)|
\end{align}
where $\rho_i$ are independent Rademacher variables. Let $\{g_j\}_{J=1}^M$ be a minimal $(1/n)$-covering of $\mathcal{B}_\mathcal{F}(\gamma)$ with respect to the $L^\infty(\mathbf{B}^d)$-norm. For a given $f\in \mathcal{B}_\mathcal{F}(\gamma)$, let $g_{j^*}$ be the function closest to $f$. By the triangle inequality, we obtain
\begin{align}
|\sum_{i=1}^n\rho_if(x_i)| &\le |\sum_{i=1}^n\rho_i(f(x_i)-g_{j^*}(x_i))| + \max_{1\le j \le M} |\sum_{i=1}^n\rho_ig_j(x_i)| \\
& \le 1 + \max_{1\le j \le M}|\sum_{i=1}^n\rho_i\frac{g_j(x_i)}{\|g_j\|_n}|\sqrt{\max_{1\le j \le M}\|g_j\|_n^2} \\
& \equiv 1 + I_1\sqrt{I_2}
\end{align}
Since $\rho_i$ are sub-Gaussian with $\mathbb{E}e^{\gamma \rho_i}\le e^{\gamma^2/2}$ for all $\gamma$, it follows from Lemma S.7 in~\cite{huiyuan2023Nonasymptotic} that
\begin{align}
    \mathbb{E}I_1 \le \sqrt{2nlog(2M)}
\end{align}
Moreover, since $g_j\in \mathcal{B}_\mathcal{F}(\gamma)$, we have $\max_j \|g_j\|_2\le \gamma$, and thus
\begin{align}
    I_2 \le \gamma^2 \max_j|\|g_j\|_n^2 - \|g_j\|_2^2|
\end{align}
Note that $\max_j \sup_x |g_j(x)| \le 2$ and $Var(|g_j(x)|^2) \le \mathbb{E}|g_J(x)|^4 \le 4\gamma^2$. Apply Bernstein's inequality~\cite{} and the union bound
\begin{align}
    \mathbb{P}(\max_J |\|g_j\|_n^2 - \|g_j\|_2^2| \ge t\gamma) \le 2Mexp(-\frac{nt^2}{16t/(3\gamma)+8}) \le 2Mexp(-\frac{nt\gamma}{6})
\end{align}
for $t\ge 12\gamma$. Using the identity $\mathbb{E}X = \int_0^{\infty}\mathbb{P}(X\ge t)dt$ for nonnegative $X$ gives, for $M\ge 4$,
\begin{align}
    \mathbb{E}\max_j|\|g_j\|_n^2 - \|g_j\|_2^2|/\gamma & \le \int_0^{12\gamma}1dt + \int_{12\gamma}^{\infty}2Mexp(-\frac{nt\gamma}{6}dt \\
    & 12\gamma + \frac{12M}{n\gamma}e^{-2n\gamma^2} \le 15\gamma
\end{align}
since $\gamma \ge \sqrt{logM/(2n)}$, which is due to the assumption $\gamma \ge \sqrt{2(dm_1+m_1m_2+m_2m_3+\cdots+ m_{L-1})\log(18L_\sigma)\log n/n}$ and the fact that $logM \le 4(dm_1+m_1m_2+m_2m_3+\cdots+ m_{L-1})\log(18L_\sigma)\log n$ to be shown later. Combining these pieces, by Jensen's inequality we have
\begin{align}\label{proof_underpara_EZ}
    \frac{1}{n}\mathbb{E}_{\rho,x}(I_1\sqrt{I_2}) = \frac{1}{n}\mathbb{E}_x{\mathbb{E}_\rho(I_1 | (x_i)_i)\sqrt{I_2}}\le 8\gamma \sqrt{\frac{logM}{n}}
\end{align}
It remains to find a $(1/n)$-covering of $\mathcal{B}_\mathcal{F}(\gamma)$ with respect to the $L_{\infty}(\mathbf{B}^d)$-norm. Consider a $(1/n^{3/2})$-covering of $\mathcal{F}^*(\mathbf{m},1)$ with respect to the $L_{\infty}(\mathbf{B}^d)$-norm, which we denote by $\{f_j\}_{j=1}^{M^{\prime}}$. In the following, we prove that $\{f_j/\max(\|f_j\|_2/\gamma,1) \}_{j=1}^{M^{\prime}}$ is a $(2/n)$-covering of $\mathbb{B}_\mathcal{F}(\gamma)$.

Since $\mathcal{B}_\mathcal{F}(\gamma) \subset \mathcal{F}^*(\mathbf{m},1)$, for any $f\in \mathcal{B}_\mathcal{F}(\gamma)$ there exists some $g_j \subset \{f_j\}_{j=1}^{M^{\prime}}$ such that $\|f_{g_j}\|_{\infty} \le 1/n^{3/2}$, and hence
\begin{align}\label{proof_underpara_EZ_2}
    |\|g_j\|_2 - \|f\|_2| \le \|g_j-f\|_2 \le \frac{1}{n^{3/2}}
\end{align}
If $\|g_j\|_2 \le \gamma$, then $g_j$ also belongs to$ \{ f_j / \max(\|f_j\|_2/\gamma, 1)\}_{j=1}^{M^{\prime}}$. If $\|g_j\|_2 > \gamma$, then by the triangle inequality,~\eqref{proof_underpara_EZ_2}, and the assumption that $\gamma\ge 1/\sqrt{n}$ we have 
\begin{align}
    |f-\frac{\gamma g_j}{\|g_j\|_2}| \le |f-g_j| + \frac{|\|g_j\|_2-\gamma|}{\|g_j\|_2}|g_j| \le \frac{1}{n^{3/2}} + \frac{n^{-3/2}}{1/\sqrt{n}} \le \frac{2}{n}
\end{align}
where we have used the fact that $0 \le \|g_j\|_2 - \gamma \le \|g_j\|_2 - \|f\|_2$. 
Substituting the above calculation of metric entropy in the beginning of this proof into~\eqref{proof_underpara_EZ} yields
\begin{align}
    \frac{1}{n}\mathbb{E}_{\rho,x}(I_1\sqrt{I_2}) \le 16\gamma \sqrt{\frac{(dm_1+m_1m_2+m_2m_3+\cdots + m_{L-1})\log(18L_\sigma)\log n}{n}}
\end{align}
and
\begin{align}
    \mathbb{E}Z_n(\gamma) & \le 16 ( \frac{1}{n}+16\gamma \sqrt{\frac{(dm_1+m_1m_2+m_2m_3+\cdots + m_{L-1})\log(18L_\sigma)\log n}{n}}) \\
    & \le 272\gamma\sqrt{\frac{(dm_1+m_1m_2+m_2m_3+\cdots + m_{L-1})\log(18L_\sigma)\log n}{n}}
\end{align}
where we have used the fact that $1/n \le \gamma \sqrt{(dm_1+m_1m_2+m_2m_3+\cdots + m_{L-1})\log(18L_\sigma)\log n/n}$. 
By the calculation of metric entropy in the beginning of this proof, we complete.

\end{proof}

\begin{proof}[Proof of Theorem~\ref{underpara_gen}]
 Modifying appropriately according to Lemma~\ref{inf-metric-entropy} and~\ref{underpara_gen_lem1}, the proof follows completely the same steps as the proof of Theorem 4 in~\citep{huiyuan2023Nonasymptotic}. Some constants now depend on $L$ and $ L_\sigma$. Again, the readers can refer to their proofs. 

 As we mentioned before, the readers should read Chapter 14 of~\cite{wainwright2019highdimensional} so that they can quickly understand the proof of this Theorem.

\end{proof}

\subsection{Proof on results in section~\ref{sec8}}

\begin{proof}[Proof of Theorem~\ref{para_gen}]
Both Theorems~\ref{overpara_gen} and~\ref{underpara_gen} hold on the range where their
hypotheses overlap, since $\lambda_2\ge\lambda_1$ there; taking the smaller of the two
right-hand sides gives the stated bound.
\end{proof}
\begin{proof}[Proof of Theorem~\ref{para_lower_bound}]
By the third property of generalized Barron spaces~\ref{property_g_barron} we know that $W^2 \subset W^L$, it is shown in~\citep{huiyuan2023Nonasymptotic} that $\inf_{\hat{f}}\sup_{f^*\in W^2}\|\hat{f} - f^*\|^2_2 \ge \frac{C}{\sqrt{n\log n}}$, where $\hat{f}$ is any estimator, thereby immediately implying the conclusion. 

We remark that the truth of the third property of generalized Barron spaces~\ref{property_g_barron} needs the special characteristic of ReLU activations (positive homogeneity), which does not hold for general Lipschitz activations.
\end{proof}
\subsection{Proofs on results in section~\ref{sec_lb}}

We first record two consequences of homogeneity that are used repeatedly. Recall from
Definition~\ref{L_layer_nn} that $g(x;\theta)$ is linear in the last weight matrix $a^L$, and
that $\nu(\theta)=\sum_{k}|W_k|\|w^1_k\|_2$ with $W=a^L\prod_{i=2}^{L-1}w^i$ is likewise
linear in $a^L$. For $t>0$ write $\theta^{(t)}$ for the parameter obtained from $\theta$ by
replacing $a^L$ with $ta^L$; then
\begin{align}\label{hom_eq}
g(\cdot;\theta^{(t)}) = t\,g(\cdot;\theta), \qquad \nu(\theta^{(t)}) = t\,\nu(\theta).
\end{align}

\begin{lemma}\label{embed_lem}
Assume $\max_i\|x_i\|_2\le1$. Then for every $\theta$,
$\|g(\cdot;\theta)\|_n\le L_\sigma^{L-1}\nu(\theta)$ and
$|\langle \epsilon, g(\cdot;\theta)\rangle_n|\le \xi\,\nu(\theta)$, with $\xi$ as
in~\eqref{xi_def}.
\end{lemma}

\begin{proof}
Since $\sigma$ is $L_\sigma$-Lipschitz with $\sigma(0)=0$ we have $|\sigma(z)|\le L_\sigma|z|$,
and peeling the $L-1$ activations in Definition~\ref{L_layer_nn} gives, for $\|x\|_2\le1$,
\begin{align}
|g(x;\theta)| \;\le\; L_\sigma^{L-1}\sum_{k=1}^{m_1}|W_k|\,|\langle w^1_k,x\rangle|
\;\le\; L_\sigma^{L-1}\sum_{k=1}^{m_1}|W_k|\,\|w^1_k\|_2
\;=\; L_\sigma^{L-1}\nu(\theta),
\end{align}
whence the first claim. For the second, by the Cauchy--Schwarz inequality,
\begin{align}
\left|\langle \epsilon, g(\cdot;\theta)\rangle_n\right|
= \left|\frac1n\sum_{i=1}^n\epsilon_i\,g(x_i;\theta)\right|
\le \frac{1}{\sqrt n}\|\epsilon\|_2\,\|g(\cdot;\theta)\|_n
\le \frac{L_\sigma^{L-1}}{\sqrt n}\|\epsilon\|_2\,\nu(\theta)
= \xi\,\nu(\theta).
\end{align}
\end{proof}

\begin{proof}[Proof of Lemma~\ref{emp_norm_conc}]
That $\|f^*\|_\infty\le L_\sigma^{L-1}\|f^*\|_{W^L}$ follows from the computation in the proof
of Lemma~\ref{embed_lem} applied to an arbitrary representation of $f^*$ and taking the
infimum over representations. Put $Z_i:=f^*(x_i)^2$, so that the $Z_i$ are i.i.d.\ with
$0\le Z_i\le B^2$, $\mathbb{E}Z_i=\|f^*\|_2^2=:\kappa^2$ and
$\mathrm{Var}(Z_i)\le\mathbb{E}Z_i^2\le B^2\kappa^2$. Bernstein's inequality gives, for $t>0$,
\begin{align}
\mathbb{P}\left(\kappa^2-\|f^*\|_n^2\ge t\right)
\;\le\; \exp\left(-\frac{nt^2}{2B^2\kappa^2+\tfrac23B^2t}\right),
\end{align}
and the choice $t=\kappa^2/2$ yields exponent
$\frac{n\kappa^4/4}{2B^2\kappa^2+B^2\kappa^2/3}=\frac{3n\kappa^2}{28B^2}$. On the complementary
event $\|f^*\|_n^2\ge\kappa^2/2$. The second form follows from $B\le L_\sigma^{L-1}\|f^*\|_{W^L}$
and the definition~\eqref{rho_def} of $\rho$, and $\|f^*\|_n\le B$ is immediate.
\end{proof}

\begin{proof}[Proof of Theorem~\ref{shrinkage_floor}]
Write $\hat g:=g(\cdot;\hat\theta)$ and $\Delta:=\hat g-f^*$. If $\hat g= 0$ then
$\|\Delta\|_n=\|f^*\|_n$ and there is nothing to prove, so assume $\nu(\hat\theta)>0$.
Restrict the objective to the ray $t\mapsto\hat\theta^{(t)}$ of~\eqref{hom_eq}:
\begin{align}
\varphi(t) \;:=\; J_n(\hat\theta^{(t)};\lambda)
\;=\; \frac{1}{2}\left\|y-t\hat g\right\|_n^2 + \lambda t\,\nu(\hat\theta),
\qquad t>0 .
\end{align}
This is a quadratic polynomial in $t$, hence smooth, and $t=1$ is a local minimum of
$\varphi$ because $t\mapsto\hat\theta^{(t)}$ is a curve through $\hat\theta$ and $\hat\theta$
is a local minimizer of $J_n(\cdot;\lambda)$. Therefore $\varphi'(1)=0$, that is,
\begin{align}\label{euler_eq}
\langle y-\hat g,\ \hat g\rangle_n \;=\; \lambda\,\nu(\hat\theta).
\end{align}
We emphasize that~\eqref{euler_eq} is obtained without differentiating $\nu$ with respect to
$\theta$; it is Euler's relation for the homogeneous pair $(g,\nu)$, read off from a
univariate polynomial. Substituting $y-\hat g=\epsilon-\Delta$ into~\eqref{euler_eq} and
applying Lemma~\ref{embed_lem},
\begin{align}
\left|\langle\Delta,\hat g\rangle_n\right|
\;=\;\left|\langle\epsilon,\hat g\rangle_n-\lambda\nu(\hat\theta)\right|
\;\ge\;\lambda\nu(\hat\theta)-\xi\nu(\hat\theta)
\;=\;(\lambda-\xi)\nu(\hat\theta).
\end{align}
On the other hand $|\langle\Delta,\hat g\rangle_n|\le\|\Delta\|_n\|\hat g\|_n\le
L_\sigma^{L-1}\nu(\hat\theta)\|\Delta\|_n$ by Cauchy--Schwarz and Lemma~\ref{embed_lem}.
Combining the two displays and dividing by $L_\sigma^{L-1}\nu(\hat\theta)>0$ gives the claim.
\end{proof}

\begin{proof}[Proof of Theorem~\ref{budget_floor}]
Comparing the value of the objective at $\hat\theta$ with its value at $\theta=0$, and using
$g(\cdot;0)=0$ and $\nu(0)=0$, we obtain
$\frac12\|y-\hat g\|_n^2+\lambda\nu(\hat\theta)\le\frac12\|y\|_n^2$, hence
$\lambda\nu(\hat\theta)\le\frac{1}{2n}\sum_iy_i^2$, which is the first claim. Since the
$W^L$ norm is an infimum over all representations of a function and $\nu(\hat\theta)$ is the
cost of one particular representation of $\hat g$, we have
$\|\hat g\|_{W^L}\le\nu(\hat\theta)\le Q_\lambda$, so $\hat g\in\mathcal{B}_{Q_\lambda}$. The
set $\mathcal{B}_{Q}$ is convex and closed in $L^2(\mu)$, being the closed convex symmetric
hull of a dictionary that is compact in $L^2(\mu)$ on $\mathbf{B}^d$, scaled by $Q$; the
displayed identity is then the support-function representation of the distance from a point
to a closed convex set.

For~\eqref{budget_floor_exp}, take $h=f^*/\kappa$, so $\|h\|_2=1$, $\langle f^*,h\rangle_2=\kappa$
and $\|h\|_{*}=\|f^*\|_{*}/\kappa=\mu(f^*)\kappa/M_*$, giving
$\operatorname{dist}_2(f^*,\mathcal{B}_{Q_\lambda})\ge\kappa-Q_\lambda\mu(f^*)\kappa/M_*$.
That $\mu(f^*)\ge1$ follows from
$\kappa^2=\langle f^*,f^*\rangle_2\le\|f^*\|_{W^L}\|f^*\|_{*}=M_*\|f^*\|_{*}$. Finally
$\mathbb{E}\sum_iy_i^2/n=\|f^*\|^2_2+\sigma_\epsilon^2$ by Assumption~\ref{assump3}.
\end{proof}

\begin{proof}[Proof of Proposition~\ref{path_norm_bounds}]
Part (i) is Lemma~\ref{embed_lem} together with the triangle inequality
$\|g(\cdot;\hat\theta)\|_n\ge\|f^*\|_n-\|g(\cdot;\hat\theta)-f^*\|_n$.

For part (ii), write $\hat g:=g(\cdot;\hat\theta)$ and substitute $y=f^*+\epsilon$
into~\eqref{euler_eq}:
\begin{align}
\lambda\nu(\hat\theta)=\langle y-\hat g,\hat g\rangle_n
=\langle f^*,\hat g\rangle_n+\langle\epsilon,\hat g\rangle_n-\|\hat g\|_n^2 .
\end{align}
By Lemma~\ref{embed_lem}, $\langle\epsilon,\hat g\rangle_n\le\xi\nu(\hat\theta)$, whence
\begin{align}
(\lambda-\xi)\nu(\hat\theta)\;\le\;\langle f^*,\hat g\rangle_n-\|\hat g\|_n^2
\;\le\;\|f^*\|_n\|\hat g\|_n-\|\hat g\|_n^2\;\le\;\frac{\|f^*\|_n^2}{4},
\end{align}
the last step being the maximum of $a\mapsto\|f^*\|_na-a^2$ over $a\ge0$, attained at
$a=\|f^*\|_n/2$. Dividing by $\lambda-\xi>0$ gives the claim.
\end{proof}

\begin{proof}[Proof of Proposition~\ref{nu_floor}]
Fix $\theta_1$ with $\nu(\theta_1)\le1$ and write $g_1:=g(\cdot;\theta_1)$. Along the
output-layer ray of~\eqref{hom_eq} we have $g(\cdot;\theta_1^{(t)})=tg_1$ and
$\nu(\theta_1^{(t)})=t$, so
\begin{align}
J_n\big(\theta_1^{(t)};\lambda\big)
=\tfrac12\|y\|_n^2-t\langle y,g_1\rangle_n+\tfrac{t^2}{2}\|g_1\|_n^2+\lambda t ,
\end{align}
a quadratic in $t$ minimized over $t\ge0$ at
$t^*=\big(\langle y,g_1\rangle_n-\lambda\big)_+\big/\|g_1\|_n^2$, with value
$\tfrac12\|y\|_n^2-\big(\langle y,g_1\rangle_n-\lambda\big)_+^2\big/\big(2\|g_1\|_n^2\big)$.
Since $\hat\theta$ is a global minimizer, $J_n(\hat\theta;\lambda)\le\tfrac12\|y\|_n^2-G$ with
\begin{align}
G:=\sup_{\nu(\theta_1)\le1}\frac{\big(\langle y,g_1\rangle_n-\lambda\big)_+^2}{2\|g_1\|_n^2}
\;\ge\;\frac{\big(\|y\|_{*,n}-\lambda\big)_+^2}{2L_\sigma^{2(L-1)}},
\end{align}
the last step by taking $\theta_1$ to nearly attain the supremum in~\eqref{dual_emp} and using
$\|g_1\|_n\le L_\sigma^{L-1}$ from Lemma~\ref{embed_lem}. Expanding
$J_n(\hat\theta;\lambda)=\tfrac12\|y\|_n^2-\langle y,\hat g\rangle_n+\tfrac12\|\hat g\|_n^2+\lambda\nu(\hat\theta)$
and cancelling $\tfrac12\|y\|_n^2$ gives
\begin{align}
\langle y,\hat g\rangle_n\;\ge\;G+\tfrac12\|\hat g\|_n^2+\lambda\nu(\hat\theta)\;\ge\;G ,
\end{align}
while $\langle y,\hat g\rangle_n\le\|y\|_{*,n}\nu(\hat\theta)$ by~\eqref{dual_emp} and
homogeneity. Dividing by $\|y\|_{*,n}$ yields~\eqref{nu_floor_eq}.

For~\eqref{nu_floor_eq2}, note $\big|\|y\|_{*,n}-\|f^*\|_{*,n}\big|\le\|\epsilon\|_{*,n}\le\xi$,
so the hypothesis $\lambda+\xi\le\frac12\|f^*\|_{*,n}$ gives
$\|y\|_{*,n}-\lambda\ge\|f^*\|_{*,n}-\xi-\lambda\ge\frac12\|f^*\|_{*,n}$ and
$\|y\|_{*,n}\le\|f^*\|_{*,n}+\xi\le\frac32\|f^*\|_{*,n}$; substituting into~\eqref{nu_floor_eq}
gives $\nu(\hat\theta)\ge\frac{\|f^*\|_{*,n}^2/4}{3L_\sigma^{2(L-1)}\|f^*\|_{*,n}}$. The second
inequality in~\eqref{nu_floor_eq2} is $\|f^*\|_n^2=\langle f^*,f^*\rangle_n\le\|f^*\|_{W^L}\|f^*\|_{*,n}$,
which holds because the supremum in~\eqref{dual_emp} over finite networks with $\nu(\theta)\le1$
equals the supremum over the unit ball of $W^L$, the former being dense in the latter and
pointwise evaluation being continuous on $W^L$.
\end{proof}

\begin{proof}[Proof of Proposition~\ref{approx_floor}]
$g(\cdot;\hat\theta)\in\mathcal{F}(\mathbf{m},\infty)$ by construction, and $A_\infty$ is by
definition the distance from $f^*$ to that set.
\end{proof}

\begin{proof}[Proof of Corollary~\ref{four_phases}]
Both entries of~\eqref{total_floor} are monotone in their own right. The first,
$A_\infty(\mathbf{m},f^*)$, is nonincreasing in the width by definition and tends to zero at
the rate of Theorem~\ref{approx1_2}; it is strictly positive at every width precisely when
$f^*$ lies in no $\mathcal{F}(\mathbf{m},\infty)$. For the second, on $\delta_n<1$ the
schedule~\eqref{schedule} gives $\lambda=\lambda_2=C_1\sigma_\epsilon\max\{\delta_n,H(\mathbf{m})^2\}$,
in which $H(\mathbf{m})^2$ decreases and $\delta_n$ increases with the width; hence $\lambda_2$
decreases, attains a minimum where the two agree, and then increases, and the second entry of
\eqref{total_floor}, being nondecreasing in $\lambda$, does the same. This gives (i), and (iii)
with the value $\min\{\|f^*\|_n,(C_1\sigma_\epsilon-\xi)/L_\sigma^{L-1}\}$ at $\delta_n\to1^-$,
since $\lambda_2\to C_1\sigma_\epsilon$ there. Part (ii) follows because $\mathrm{Fl}$ is the
maximum of a nonincreasing and an eventually increasing function, so its minimum is attained at
their crossing and is at least $A_\infty$ there. For (iv), Theorem~\ref{underpara_gen} requires
$\delta_n<1$, beyond which only $\lambda_1$ is prescribed; substituting $\lambda=\lambda_1$
into~\eqref{total_floor} gives the stated value, which decreases with $A_\infty$.
\end{proof}

\begin{proof}[Proof of Theorem~\ref{double_descent}]
All three displays are evaluated under the schedule~\eqref{schedule}, so each theorem invoked
is applied at a width where its hypotheses hold.

For~\eqref{dd_valley}, the width $\mathbf{m}_0$ lies in the underparameterized regime and
$\lambda=\lambda_2$ there, so Theorem~\ref{underpara_gen} applies; by the definition of
$\mathbf{m}_0$ its two terms are of the same order, whence the bound
$2C(\sigma_\epsilon^2+\|f^*\|^2_{W^L})\delta_n(\mathbf{m}_0)$, valid with probability
$1-O(n^{-C_2})$.

For~\eqref{dd_peak}, $\mathbf{m}_\sharp$ also lies in the underparameterized regime, where
$\lambda=\lambda_2\to C_1\sigma_\epsilon$ as $\delta_n\to1^-$. Theorem~\ref{shrinkage_floor}
then gives
$\|g(\cdot;\hat\theta)-f^*\|_n\ge\min\{\|f^*\|_n,(C_1\sigma_\epsilon-\xi)/L_\sigma^{L-1}\}$,
and Lemma~\ref{emp_norm_conc} replaces $\|f^*\|_n$ by $\|f^*\|_2/\sqrt2$ on an event of
probability $1-\exp(-3n\rho(f^*)^2/(28L_\sigma^{2(L-1)}))$. The bound on $\xi$ established in
the proof of Theorem~\ref{overpara_emp} holds with probability $1-e^{-f(n)}$.

For~\eqref{dd_tail}, at widths with $\delta_n\ge1$ the schedule sets $\lambda=\lambda_1$, which
is exactly the hypothesis of Theorem~\ref{overpara_emp}; its conclusion is the stated display.
A union bound over the three events gives the stated probability. In part~(a) the two upper
bounds are the empirical versions, Theorems~\ref{underpara_emp} and~\ref{overpara_emp}, so all
three displays are statements about $\|\cdot\|_n$.

For part~(b), \eqref{dd_valley2} and~\eqref{dd_tail2} are Theorems~\ref{underpara_gen}
and~\ref{overpara_gen}, applied at the same widths and under the same schedule. For
\eqref{dd_peak2} we use Theorem~\ref{budget_floor} rather than Theorem~\ref{shrinkage_floor},
for the reason given in Remarks~\ref{rmk_transfer} and~\ref{rmk_which_floor}: at
$\mathbf{m}_\sharp$ the schedule gives $\lambda\to C_1\sigma_\epsilon$, so
$Q_\lambda=\|y\|_n^2/(2\lambda)\to(\|f^*\|_2^2+\sigma_\epsilon^2)/(2C_1\sigma_\epsilon)$ by the
law of large numbers, and~\eqref{budget_floor_exp} gives
$\|g(\cdot;\hat\theta)-f^*\|_2\ge\|f^*\|_2(1-\mu(f^*)Q_\lambda/\|f^*\|_{W^L})_+$, which is the
stated $\mathrm{Fl}^{(2)}_\sharp$. It is strictly positive exactly under the hypothesis
$\mu(f^*)(\|f^*\|_2^2+\sigma_\epsilon^2)<2C_1\sigma_\epsilon\|f^*\|_{W^L}$, which by
Remark~\ref{rmk_rho_mu} is implied by $C_1>L_\sigma^{L-1}\mu(f^*)$ together with the stated
interval for $M_*/\sigma_\epsilon$.

Finally, $\delta_n(\mathbf{m}_0)\to0$ because $\mathbf{m}_0$ is determined by the balance
$H(\mathbf{m}_0)^2\|f^*\|^2_{W^L}\asymp(\sigma_\epsilon^2+\|f^*\|^2_{W^L})\delta_n(\mathbf{m}_0)$
with $H(\mathbf{m})^2\to0$ as the width grows, and $\sqrt{f(n)/n}\to0$ by inspection, whereas
$\mathrm{Fl}_\sharp$ depends only on $\|f^*\|_2$, $\sigma_\epsilon$, $C_1$ and $L_\sigma$. The two
ratios therefore diverge.
\end{proof}

\begin{proof}[Proof of Theorem~\ref{no_fixed_lambda}]
By Theorem~\ref{shrinkage_floor}, requirement (ii) forces
$\min\{\|f^*\|_n,(\lambda-\xi)_+/L_\sigma^{L-1}\}^2$ to be at most the plateau
$C\Lambda_{L,d}(\sigma_\epsilon^2+\|f^*\|^2_{W^L})\sqrt{f(n)/n}$. Since the latter tends to
zero while $\|f^*\|_n$ does not, the minimum must be attained by the second branch, so
\begin{align}
\lambda \;\le\; \xi + L_\sigma^{L-1}\sqrt{C\Lambda_{L,d}(\sigma_\epsilon^2+\|f^*\|^2_{W^L})}\left(\frac{f(n)}{n}\right)^{1/4}
\;\asymp\; L_\sigma^{L-1}\sqrt{C\Lambda_{L,d}(\sigma_\epsilon^2+\|f^*\|^2_{W^L})}\left(\frac{f(n)}{n}\right)^{1/4},
\end{align}
because $\xi=O(\sqrt{f(n)/n})$ is of smaller order. Comparing with
$\lambda\ge C_1\sigma_\epsilon\delta$ from (i) yields the stated threshold on $\delta$;
substituting $\delta_n=n^{-1}L_\sigma(dm_1+m_1m_2+\cdots+m_{L-1})\log n$ gives the statement
about the number of weights.
\end{proof}

\begin{proof}[Proof of Theorem~\ref{slow_rate}]
This is~\eqref{empirical_e} with the term $2(\xi-\lambda)\nu(\hat\theta-\theta^*)$ retained
rather than discarded: it is nonpositive as soon as $\lambda\ge\xi$, and
$\nu(\theta^*)\le\|f^*\|_{W^L}$ by Theorem~\ref{approx3}.
\end{proof}

\subsection{Proof on results in section~\ref{sec9}}

We first recall some terminologies presented in the main text. Let $\Tilde{L}(g(x;\hat{\theta}),f^*(x)):= E_{y}{L}(g(x;\hat{\theta}),y)$, the expectation over $y$ conditioned on $x$. We have suggested using $\mathbb{D}(g(\cdot ;\hat{\theta}), f^*(\cdot)):=\mathbb{E}_x\Tilde{L}(g(x;\hat{\theta}),f^*(x))-\mathbb{E}_x\Tilde{L}(f^*(x),f^*(x))$ as the measure of the difference between $g(\cdot;\hat{\theta})$ and $f^*$ (without using data $x,y$). For a given training data $(x_i,y_i), i=1,2, \dots,n$, we similarly have the empirical version $\Tilde{L}_n(g(x;\hat{\theta}),f^*)-\Tilde{L}_n(f^*,f^*)$. 
The second term $\Tilde{L}(f^*,f^*)$ plays a role as a normalization factor, so that $\mathbb{D}(f^*,f^*)=0$ which is required. 
\begin{proof}[Proof of Theorem~\ref{general_loss_overpara_emp}, inequality~\eqref{emp_general_loss}]
\begin{align}
& \frac{1}{n}\sum_{i=1}^n L(g(x_i;\hat{\theta}),y_i) + \lambda\mu(\hat{\theta})  \le \frac{1}{n}\sum_{i=1}^n L(g(x_i;{\theta}^*),y_i) + \lambda\mu({\theta}^*) \\
   & \frac{1}{n}\sum_{i=1}^n \Tilde{L}(g(x_i;\hat{\theta}),f^*(x_i)) + \frac{1}{n}\sum_{i=1}^n L(g(x_i;\hat{\theta}),y_i) - \mathbb{E}_{y_i} L(g(x_i;\hat{\theta}),y_i) + \lambda\mu(\hat{\theta}) \\
   & \le
\frac{1}{n}\sum_{i=1}^n \Tilde{L}(g(x_i;{\theta}^*),f^*(x_i)) + \frac{1}{n}\sum_{i=1}^n L(g(x_i;{\theta}^*),y_i) - \mathbb{E}_{y_i} L(g(x_i;{\theta}^*),y_i) + \lambda\mu({\theta}^*) 
\end{align}
Subtracting $\Tilde{L}(f^*,f^*)$ on both side and rearranging terms, one gets
\begin{align}\label{empirical_e_gen_loss}
    \frac{1}{n}\sum_{i=1}^n \Tilde{L}(g(x_i;\hat{\theta}),f^*(x_i)) - \Tilde{L}(f^*(x_i),f^*(x_i)) & \le \frac{1}{n}\sum_{i=1}^n |\Tilde{L}(g(x_i;{\theta}^*),f^*(x_i)) - \Tilde{L}(f^*(x_i),f^*(x_i))|  \\ 
    & + \frac{1}{n}|\sum_{i=1}^n L(g(x_i;{\theta}^*),y_i) 
     - \mathbb{E}_{y_i} L(g(x_i;{\theta}^*),y_i) \\
     & - (\sum_{i=1}^n L(g(x_i;\hat{\theta}),y_i) - \mathbb{E}_{y_i} L(g(x_i;\hat{\theta}),y_i))| \\
     & + \lambda(\mu({\theta}^*) - \mu(\hat{\theta})) \\
     & \equiv T_1 + T_2 + T_3
\end{align}
$T_1$ can be bounded via the Lipschitzness of $\Tilde{L}$ and $L^2$ difference between $g(x_i;\theta^*)$ and $f^*$. $T_2$ can be jointly bounded with part of $T_3$ via Hoeffeding concentration inequality by the Lipschitzness of $L$. Part of $T_3$ is bounded by approximation Theorem~\ref{approx1_2}. 

\begin{align}
T_1 & =\frac{1}{n}\sum_{i=1}^n |\Tilde{L}(g(x_i;{\theta}^*),f^*(x_i)) - \Tilde{L}(f^*(x_i),f^*(x_i))| \\
    & \le L_1^{\prime} \frac{1}{n}\sum_{i=1}^n |g(x_i;{\theta}^*) - f^*(x_i)| \\
    & \le \frac{L_1^{\prime}}{\sqrt{n}}\sqrt{\sum_{i=1}^n |g(x_i;{\theta}^*) - f^*(x_i)|^2} \\
    & = \frac{L_1^{\prime}}{\sqrt{n}}\|g(\cdot ;\theta^*) - f^*(\cdot)\|_{L^2(\mathrm{P}_n)} \\
    &   \le C_1 \frac{H(\mathbf{m}){\|f\|_{W^L}}}{\sqrt{n} } 
\end{align}
for some constant $C_1$ depending on $L'_1$ further.

$T_3$ is also rewritten as $
T_3 = \lambda(\nu(\theta^*) - \nu(\hat{\theta})) = 2\lambda \nu(\theta^*)  - \lambda \nu(\theta^* \ominus \hat{\theta})$ with the first term bounded by $2\lambda\|f^*\|_{W^L}$. 

$T_2$ can be rewritten as 
\begin{align}
     & \frac{1}{n}|\sum_{i=1}^n L(g(x_i;{\theta}^*),y_i) 
     - \mathbb{E}_{y_i} L(g(x_i;{\theta}^*),y_i) - (\sum_{i=1}^n L(g(x_i;\hat{\theta}),y_i) - \mathbb{E}_{y_i} L(g(x_i;\hat{\theta}),y_i))| \\
    & = \frac{1}{n} |\sum_{i=1}^n L(g(x_i;{\theta}^*),y_i) - L(g(x_i;\hat{\theta}),y_i)
     - (\mathbb{E}_{y_i} L(g(x_i;{\theta}^*),y_i) -  \mathbb{E}_{y_i} L(g(x_i;\hat{\theta}),y_i))|
\end{align}
As before, $L(g(x_i;{\theta}^*),y_i) - L(g(x_i;\hat{\theta}),y_i)$ can also be considered as the concatenation of two networks $L(g(x_i;{\theta}^*),y_i)$ and with one more output layer on top of them for doing subtraction, and  $L(g(x_i;\hat{\theta}),y_i)$ evaluated on $x_i$ and $y_i$. The loss $L(\cdot,\cdot)$ is interpreted as  an activation function composed with the value of the output node. So we may continuously abbreviate write $L(g(x_i;{\theta}^*-\hat{\theta}),y_i) := L(g(x_i;{\theta}^*),y_i) - L(g(x_i;\hat{\theta}),y_i)$ with network parameters $\theta^*-\hat{\theta}$. The Lipschitzness of $L$ tells us $|L(g(x_i;{\theta}^*),y_i) - L(g(x_i;\hat{\theta}),y_i)| \le L_1|g(x_i;{\theta}^*) - g(x_i;\hat{\theta})|=L_1|g(x_i;{\theta}^*-\hat{\theta})|\le L_1\|g(\cdot ; {\theta}^*-\hat{\theta})\|_{W^L}$ which is bounded uniformly. 

Then, for any positive real $e$, conditioned on the event $\|g(\cdot ; {\theta}^*-\hat{\theta})\|_{W^L} = e$, $L(g(x_i;{\theta}^*-\hat{\theta}),y_i)$  is a bounded random variable, so it is subgaussian. 
Therefore,
\begin{align}
    & \mathbb{P} [\frac{1}{n}|\sum_{i=1}^n L(g(x_i;{\theta}^*- \hat{\theta}),y_i) - \mathbb{E}_{y_i}L(g(x_i;{\theta}^*-\hat{\theta}),y_i) | \ge \lambda \mu(\theta^* \ominus \hat{\theta}) |\|g(\cdot ; {\theta}^*-\hat{\theta})\|_{W^L} = e] \\
    & \le 2exp\left \{-\frac{n^2\lambda^2\mu^2({\theta}^*-\hat{\theta})}{2\sum_{i=1}^nL_1^2g^2(x_i;{\theta}^*-\hat{\theta})}\right \} \\
    & \le 2exp\left \{-\frac{n^2\lambda^2\mu^2({\theta}^*-\hat{\theta})}{2\sum_{i=1}^nL_1^2(L_\sigma^{L-1})^2\mu^2({\theta}^*-\hat{\theta})}\right \} \\
    & = 2exp\left \{-\frac{n\lambda^2}{2L_1^2(L_\sigma^{L-1})^2}\right \}
\end{align}
The bound is independent of $e$. Therefore, we can average over $e$ to get an unconditional probability bound with the same form. In order for the right hand size of the above inequality less than $n^{-4}$, one can compute that $\lambda \ge L_1L_\sigma^{L-1}\sqrt{{2log(2n^4)\over n}}$. So we set $\lambda = \max\{6L_1L^{L-1}_\sigma, (L_1+L_1^{\prime})2^{L-1}cL_\sigma^{L-1}\sqrt{d}\}\sigma_{\epsilon}\sqrt{f(n)/n}$ so that $T_2 \le \lambda \nu(\hat{\theta}\ominus\theta^*)$ holds with probability at least $1 - n^{-4}$ (we take such value for $\lambda$ because we need to be consistent with the expectation value estimation below).  To complete the proof, substituting the value of $\lambda$ into~\eqref{empirical_e_gen_loss} gives
\begin{align}\label{fine_overpar_emp}
\frac{1}{n}\sum_{i=1}^n \Tilde{L}(g(x_i;\hat{\theta}),f^*(x_i)) - \Tilde{L}(f^*,f^*) \le C_1H(\mathbf{m})\|f^*\|_{W^L}/\sqrt{n} + \max\{12L^{L-1}_\sigma, (L_1+L_1^{\prime})2^{L}cL_\sigma^{L-1}\sqrt{d}\}\|f^*\|_{W^L}\sqrt{\frac{f(n)}{n}}
\end{align}
So,
\begin{align}
\frac{1}{n}\sum_{i=1}^n \Tilde{L}(g(x_i;\hat{\theta}),f^*(x_i))\le   \Tilde{L}(f^*,f^*)  + C_1H(\mathbf{m})\|f^*\|_{W^L}/\sqrt{n} + \max\{12L^{L-1}_\sigma, (L_1+L_1^{\prime})2^{L}cL_\sigma^{L-1}\sqrt{d}\}\|f^*\|_{W^L}\sqrt{\frac{f(n)}{n}}
\end{align}
\end{proof}
To understand the difference between this expression and the one for MSE loss, we note that MSE loss is not an Lipschitz function. In general, if $L(f,g)$ defines an distance between $f$ and $g$, we may similarly obtain  an empirical error bound as the case of MSE loss by setting an approximation result corresponding to this distance instead of $L^2$ distance. But if no, our Lipschitzness assumption is a stronger assumption, leading to a stronger empirical error bound. 

To prove the second statement~\eqref{emp_exp_general_loss}, we make a further assumption on the distribution of target $y$ conditional on $x$ which is general enough.
\begin{assumption}~\label{gen_loss_y_dis}
    $y$ is sub-gaussianal conditional on $x$.
\end{assumption}
This has already included our MSE case as $y$ is Gaussian conditional on $x$ in that case. A standard fact on the sub-gaussian is the following
\begin{proposition}~\label{subgaussin_prop}
If $h(y)$ is a Lipschitz function of $y$ with Lipschitz constant $L_y$, and $y$ is sub-gaussian with parameter $\sigma$, i.e. $\mathbb{E}e^{\lambda y}\le e^{\lambda^2\sigma^2}$, then $h(y)$ is sub-gaussian with parameter $L_y\sigma$.    
\end{proposition}

\begin{proof}[Proof of Theorem~\ref{general_loss_overpara_emp}, inequality~\eqref{emp_exp_general_loss}]
Let $Z(g):=\frac{1}{n}\sum_{i=1}^n L(g(x_i;{\theta}^*\ominus\hat{\theta}),y_i) - \mathbb{E}_{y_i}L(g(x_i;{\theta}^*\ominus\hat{\theta}),y_i)$, so it is a zero-mean random process indexed by $g$. Then $|Z(g_1)-Z(g_2)| \le \frac{1}{n}\sum_{i=1}^nL_1|g_1(x_i;\theta^*\ominus\hat{\theta}) - g_2(x_i;\theta^*\ominus\hat{\theta})|+\frac{1}{n}\sum_{i=1}^nL_1^{\prime}\sum_{i=1}^n|g_1(x_i;\theta^*\ominus\hat{\theta}) - g_2(x_i;\theta^*\ominus\hat{\theta})|$. So $Z(g_1)-Z(g_2)$ is bounded by a distance measure between $g_1$ and $g_2$, thereby is a sub-gaussian with parameter $\frac{1}{n}(L_1+L_1^{\prime})|g_1(x_i;\theta^*\ominus\hat{\theta}) - g_2(x_i;\theta^*\ominus\hat{\theta})|$. 
This parameter is a rescaled $l^1$ distance on the space of $g$. Let us define
\begin{align}
    \|g_1 - g_2\|_{\mathbb{P}_n} := \frac{1}{n}|g_1(x_i;\theta^*\ominus\hat{\theta}) - g_2(x_i;\theta^*\ominus\hat{\theta})|
\end{align}
We do 
\begin{align}
& \mathbb{E}T_2  \\ 
&= \mathbb{E} [\frac{1}{n}|\sum_{i=1}^n L(g(x_i;{\theta}^*\ominus\hat{\theta}),y_i) - \mathbb{E}_{y_i}L(g(x_i;{\theta}^*\ominus\hat{\theta}),y_i) | 
\end{align}

Given an arbitrary function family $\mathcal{F}$ of $g$ bounded by $b$, by our assumption of the loss function $L$, Proposition~\ref{subgaussin_prop} and Dudley entropy integral results in Chapter 5 of~\cite{wainwright2019highdimensional}, a simple change of scale between $L(g,y)$ and $g$ by the Lipschitzeness of $L$ shows that
\begin{align}
    \mathbb{E}_{y_i} \left [ \sup_{g\in \mathcal{F}} \left | \frac{1}{n}\sum_{k=1}^nL(g(x_i;{\theta}^*\ominus\hat{\theta}),y_i) - \mathbb{E}_{y_i}L(g(x_i;{\theta}^*\ominus\hat{\theta}),y_i)\right | \right ] \le (L_1+L_1^{\prime})\frac{24}{\sqrt{n}}\int_0^{2b}\sqrt{logN(t;\mathcal{F},\|\cdot \|_{\mathbb{P}_n})}dt
\end{align}
where we use $\sup_{f,g\in \mathcal{F}}\|f-g\|_{\mathbb{P}_n} \le 2b$.
So it reduces to the estimation of the covering number. As $\|\cdot \|_{\mathbb{P}_n} \le \| \cdot \|_{\infty}$ , it remains to bound the metric entropy with respect to supremum norm. Thus, it reduces to the situation of $L$ being MSE loss. By the proof of Proposition~\ref{gen_lem0} and its corollaries~\ref{gen_cor0},~\ref{gen_lem1}, we deduce that 
\begin{align}
     \mathbb{E}_{y_i} \left [ \sup_{g\in \mathcal{F}} \left | \frac{1}{n}\sum_{k=1}^nL(g(x_i;{\theta}^*\ominus\hat{\theta}),y_i) - \mathbb{E}_{y_i}L(g(x_i;{\theta}^*\ominus\hat{\theta}),y_i)\right | \right ] \le (L_1+L_1^{\prime})2^{L-1}cL_\sigma^{L-1}F\sqrt{\frac{d}{n}}
\end{align}
This bound $O(\sqrt{\frac{1}{n}})$ is much better than~\eqref{emp_exp} $O(\sqrt{\frac{\log n}{n}})$ for MSE loss, again due to its Lipschitzness of the loss function.

The proof of Lemma~\ref{gen_lem1} (or Remark~\ref{rmk_of_gen_lem1}) in fact gives the upper bound of the maximum value of $g$ by $2^LL_\sigma^{L-1}F$. By the metric entropy calculation in the proof of Lemma~\ref{gen_cor0} with its decreasing property with respect to $t$ under the integral, we complete.
\end{proof}

\begin{proof}[Proof of Lemma~\ref{general_loss_f_concentration}]
By a standard symmetrization argument,
\begin{align}
\mathbb{E}Z_n \le 2 \mathbb{E}_{\rho,x} \sup_{f\in \mathcal{F}(\mathbf{m},1)} \left | {1 \over n}\sum_{i=1}^n \rho_i\tilde{\mathcal{L}}(f(x_i),f^*(x_i))  \right | 
\end{align}
where $\rho_i$ are independent Rademacher variables. We have
\begin{align}
\mathbb{E}Z_n & \le 2 \mathbb{E}_{\rho,x}\sup_{f\in \mathcal{F}(\mathbf{m},1)} \left |{1\over n}\sum_{i=1}^n \rho_i \tilde{\mathcal{L}}(f(x_i), f^*(x_i))   \right | \\ & =  2 \mathbb{E}_{\rho,x}\sup_{f\in \mathcal{F^*}(\mathbf{m},1)} \left |{1\over n}\sum_{i=1}^n \rho_i \tilde{\mathcal{L}} (f(x_i)+f^*(x_i), f^*(x_i))   \right | \\
& \le 4L_1^{\prime} \mathbb{E}_{\rho,x}\sup_{f\in \mathcal{F}(\mathbf{m},1)}\left |{1 \over n}\sum_{i=1}^n \rho_i (f(x_i) - f^*(x_i))  \right |
\end{align}
Let $\Tilde{f}\in \mathcal{F}(\Tilde{\mathbf{m}},1)$ be the $L$-layer neural network that best approximates $f^*$ under the $L^2(\mathbf{B}^d)$-norm in Theorem~\ref{approx1_2}, where $\Tilde{m} \ge n$ elementwise. Thus, $\|\Tilde{f} - f^*\|_{L^2(\mathbf{B}^d)} \le C_1/\sqrt{n}$ for some constant $C_1>0$ depending on $f^*$. By decomposing $f=f - \Tilde{f} + \Tilde{f}$ and noting that $f - \Tilde{f} \in \mathcal{F}(\mathbf{m}+\Tilde{\mathbf{m}},2)$, we obtain
\begin{align}
\mathbb{E}Z_n & \le 4L_1^{\prime} \mathbb{E}_{\rho,x}\sup_{f\in \mathcal{F}(\mathbf{m},1)}\left |{1 \over n}\sum_{i=1}^n \rho_i(f(x_i)-\Tilde{f}(x_i))   \right | + 4L_1^{\prime} \mathbb{E}_{\rho,x}\left | {1\over n}\sum_{i=1}^n \rho_i(\Tilde{f}(x_i) - f^*(x_i))  \right | \\
& \le 4L_1^{\prime} \mathbb{E}_{\rho,x}\sup_{f\in \mathcal{F}(\mathbf{m}+\Tilde{\mathbf{m}},2)} \left | {1\over n}\sum_{i=1}^n\rho_i f(x_i)  \right | + {4L_1^{\prime}\over n}\sum_{i=1}^n \sqrt{\mathbb{E}_\rho  \rho_i^2}\|\Tilde{f} - f^*\|_2 \\
& \le 4L_1^{\prime}\mathbb{E}_{\rho,x}\sup_{f\in \mathcal{F}(\mathbf{m}+\Tilde{\mathbf{m}},2)} \left |{1\over n}\sum_{i=1}^n\rho_if(x_i)  \right | + {4L_1^{\prime}C_1\over \sqrt{n}} \le {8L_1^{\prime}c2^{L-1}L_\sigma^{L-1}\sqrt{d} + 4L_1^{\prime}C_1 \over \sqrt{n}} \equiv {C_\mathcal{F} \over \sqrt{n}}
\end{align}
where the last inequality follows from Lemma~\ref{gen_lem1}.
Define
\begin{align}
U = \sup_{x\in \mathbf{B}^d}\sup_{f\in \mathcal{F}(\mathbf{m},1)} \mathcal{L}(f,f^*), \xi^2 = \sup_{f\in \mathcal{F}(\mathbf{m},1)}\mathbb{E}\mathcal{L}^2(f,f^*), K_n = 2U\mathbb{E}Z_n + \xi^2
\end{align}
and note that for $\sqrt{n}\ge C_\mathcal{F}$,
\begin{align}
U & \le \sup_{x\in \mathbf{B}^d}\sup_{f\in \mathcal{F}(\mathbf{m},1)}|\mathcal{L}(f,f^*)| \\
& \le \sup_{x\in \mathbf{B}^d}\sup_{f\in \mathcal{F}(\mathbf{m},1)}L_1^{\prime}|f-f^*| \\
& \le L_1^{\prime}\sup_{x\in \mathbf{B}^d}\sup_{f\in \mathcal{F}(\mathbf{m},1)}(|f| + |f^*|) \\
& \le 2L_1^{\prime} \\
\xi^2 & \le U^2 \le 4L_1^{\prime 2}, \\
 K_n & \le {8C_\mathcal{F}\over \sqrt{n}} + 4L_1^{\prime 2} \\
& \le 8 + 4L_1^{\prime 2}
\end{align}
the second inequality owes to $\mathcal{L}(f^*,f^*) = 0$.
By Talagrand's concentration inequality~\citep{wainwright2019highdimensional},
\begin{align}
P(Z_n - \mathbb{E}Z_n \ge t) \le 2exp\left (-{nt^2 \over 8eK_n+4Ut}  \right ) \le 2exp\left (-{nt^2 \over 32(2+L_1^{\prime 2})e+8L_1^{\prime}t} \right )
\end{align}
Note that $-nt^2/(32(2+L_1^{\prime 2})e+8L_1^{\prime}t) \le -nt/(16L_1^{\prime})$ if $t\ge 4e(2+L_1^{\prime 2})/L_1^{\prime}$, and $ -nt^2 /(64(2+L_1^{\prime 2})e)$ otherwise. We then conclude that
\begin{align}
P\left (Z_n \ge {C_\mathcal{F} \over \sqrt{n}} + t   \right ) \le exp\left \{ -{n \over (16L_1^{\prime})}\min({t^2\over 4e(2+L_1^{\prime 2})/L_1^{\prime}},t) \right \}
\end{align}
\end{proof}

\begin{proof}[Proof of Theorem~\ref{general_loss_overpara_gen}]
The argument is exactly similar to the argument in the proof of~\ref{overpara_gen} (just notices the changes in some expressions accordingly). Also, one thing that will change is that for any $f \in \mathcal{F}(\mathbf{m}, F)$, we have $f/\nu(f) \in \mathcal{F}(\mathbf{m},1)$, and $\nu(f/\nu(f)) = 1 $. This can be checked by changing the output layer weights (the last hidden layer), which does not rely on the homogeneity of activation functions like ReLU.
\end{proof}

\begin{proof}[Proof of Lemma~\ref{general_loss_underpara_gen_lem1}]
By a standard symmetrization argument,

\begin{align}
\mathbb{E}Z_n & \le 2 \mathbb{E}_{\rho,x} \sup_{f\in \mathcal{F}(m,1), \|f-f^*\|_2\le \gamma} \left | {1 \over n}\sum_{i=1}^n \rho_i\tilde{\mathcal{L}}(f(x_i),f^*(x_i))  \right | \\ 
& \le 4L_1^{\prime} \mathbb{E}_{\rho,x} \sup_{f\in \mathcal{B}_F(\gamma)} \left | {1 \over n}\sum_{i=1}^n \rho_if(x_i)  \right |
\end{align}
The above argument is what we done as in the proof of Lemma~\ref{general_loss_f_concentration}.  Then this reduces to Lemma~\ref{underpara_gen_lem1}.

\end{proof}
\begin{proof}[Proof of Theorem~\ref{general_loss_para_gen}]
On the range where $\delta_n<1$ both Theorem~\ref{general_loss_overpara_gen} and
Theorem~\ref{general_loss_underpara_gen} apply, since $\lambda_2\ge\lambda_1$ there, and one
keeps whichever of the two right-hand sides is smaller; where $\delta_n\ge1$ only
Theorem~\ref{general_loss_overpara_gen} is available and the schedule sets $\lambda=\lambda_1$.
Combining the two gives the stated bound, with the minimum recording which regime is in force.
\end{proof}

\begin{proof}[Proof of Theorem~\ref{general_loss_underpara_emp}]
This is similar to the proof of Theorem~\ref{underpara_emp}. The only difference is that we have now $L+1$-layer neural networks 
because the composition of $\mathcal{L}(\cdot ,f^*)$  with $g$ is equivalent to setting activation function to be $\mathcal{L}_{n}(\cdot , f^*)$ for the last output node and adding one more hidden layer with a single weight setting to 1.  
\end{proof}
\begin{proof}[Proof of Theorem~\ref{general_loss_underpara_gen}]
This result is the analogy to Theorem~\ref{underpara_gen}. To emphasize the importance of Theorem~\ref{indispensable result} for our purpose, we copy it here for the reader's convenience. 

\begin{theorem}\label{indispensable_result_cp}
Given the uniform 1-bounded function class $\mathcal{F}(\mathbf{m},1)$, and it is clear that it is star shaped around the ground truth $f^*$, i.e. $cf\in \mathcal{F}(\mathbf{m},1)$ for any $c\in [0,1]$ and $f\in \mathcal{F}(\mathbf{m},1)$ near $f^*$. Let 
\begin{align}
r_n = \sqrt{ \frac{(2L_\sigma)^{L-1} (2L_{1} + 2|\mathcal{L}_{n}(0,f^*)|) (2d m_1 + 4m_1 m_2 + 4m_2 m_3 + \cdots + 2m_{L-1}) \log n}{n} }
\end{align}
Then
\begin{enumerate}
    \item \label{general_loss_aux_thm_1_cp} Recall that $\Tilde{\mathcal{L}}(f,f^*)$ is $L_1^{\prime}-$Lipschitz with respect to the first argument $f$, then
    \begin{align}\label{event_0}
    \sup_{f\in \mathcal{F}(\mathbf{m},1)} \frac{|\int_{\mathbf{B}^d}((\Tilde{\mathcal{L}}(f,f^*) - \Tilde{\mathcal{L}}(f^*,f^*))d\mu - (\Tilde{\mathcal{L}}_n(f,f^*)-\Tilde{\mathcal{L}}_n(f^*,f^*)))|}{\|f-f^*\|_2 + r_n} \le 10L_1^{\prime}r_n
    \end{align}
    with probability at least $1 - c_1e^{-c_2nr_n^2}$.
    \item \label{general_loss_aux_thm_2_cp} Furthermore, assume that $\Tilde{\mathcal{L}}(f,y)$ is $\gamma$-strongly convex for the first argument $f$ for each $y$, then we have 
    \begin{align}
    \|\hat{f}-f^*\|_2 \le c_2r_n + c_3 
    \end{align}
    and then,
    \begin{align}\label{event_1_cp}
    \sup_{f\in \mathcal{F}(\mathbf{m},1)} {|\int_{\mathbf{B}^d}((\Tilde{\mathcal{L}}(f,f^*) - \Tilde{\mathcal{L}}(f^*,f^*))d\mu - (\Tilde{\mathcal{L}}_n(f,f^*)-\Tilde{\mathcal{L}}_n(f^*,f^*)))|} \le 
    c_2r_n^2 + c_3r_n
    \end{align}
\end{enumerate}
with the same probability as~\ref{general_loss_aux_thm_1_cp}, for some constants $c_2,c_3$. 
\end{theorem}
We first introduce the concept of Local Rademacher Complexity, which is suitable for function classes with bounded norms.
\begin{definition}(Local Rademacher Complexity)
Let $\mathcal{F}$ be a given family of functions on $\mathbf{B}^d$ and a given $r>0$. The local Rademacher complexity, denoted by $R_n(r;\mathcal{F})$, is
\begin{align}
R_n(r;\mathcal{F}) := \mathbb{E}_{x,\epsilon} \sup_{\|f\|_2 \le r} \left| \frac{1}{n} \sum_{i=1}^n \epsilon_i (f(x_i) - f^*(x_i)) \right|.
\end{align}
For a given dataset $x_1^n := (x_1,\dots,x_n)$, the empirical local Rademacher complexity with respect to $x_1^n$ is
\begin{align}
R_n(r;\mathcal{F}, x_1^n) := \mathbb{E}_{\epsilon} \sup_{\|f\|_2 \le r} \left| \frac{1}{n} \sum_{i=1}^n \epsilon_i (f(x_i) - f^*(x_i)) \right|.
\end{align}
\end{definition}
Define a random variable family
\begin{align}
    Z_n(r) = \sup_{\|f-f^*\|_2\le r}|\int_{\mathbf{B}^d}(\Tilde{\mathcal{L}}(f,f^*) - \Tilde{\mathcal{L}}(f^*,f^*))d\mu - (\Tilde{\mathcal{L}}_n(f,f^*)-\Tilde{\mathcal{L}}_n(f^*,f^*))|
\end{align}
Then, we have the following fact 
controlling the $Z_n(r)$'s tail probability
\begin{lemma}(Lemma 14.21 of~\cite{wainwright2019highdimensional})~\label{gen_loss_Z_lemma}
    For every $r\ge r_n$, $Z_n(r)$ satisfies the tail probability bound
    \begin{align}
        \mathbb{P}[Z_n(r)\ge 8Lrr_n + u] \le c_1exp (-\frac{c_2nu^2}{(L_1^{\prime})^2r^2+L_1^{\prime}u} )
    \end{align}
\end{lemma}
By the Lipschitzness of $\Tilde{\mathcal{L}}$ and the boundedness of $f$, we have $|\Tilde{\mathcal{L}}(f,f^*)-\Tilde{\mathcal{L}}(f^*,f^*)|_{\infty}\le L_1^{\prime}\|f-f^*\|_{\infty}\le 2L$. Moreover, we have
\begin{align}
    Var(\Tilde{\mathcal{L}}(f,f^*) - \Tilde{\mathcal{L}}(f^*,f^*)) & \le \mathbb{P}[(\Tilde{\mathcal{L}}(f,f^*) - \Tilde{\mathcal{L}}(f^*,f^*)^2)] \\
    & \le L^{\prime 2}_1\|f-f^*\|_2^2 \le L^{\prime 2}_1r^2
\end{align}
So, by the Talagrand concentration inequality, we have
\begin{align}\label{general_loss_tail_Z}
    \mathbb{P}[Z_n(r)\ge 2\mathbb{E}[Z_n(r)]+u] \le c_1exp\{-\frac{c_2nu^2}{L^{\prime 2}_1r^2+L_1^{\prime}u} \}
\end{align}
It remains to upper bound $\mathbb{E}[Z_n(r)]$. 
\begin{align}
    \mathbb{E}[Z_n(r)] & \le 2\mathbb{E}[\sup_{\|f-f^*\|_2\le r}|\frac{1}{n}\sum_{i=1}^n\sigma_i(\tilde{\mathcal{L}}(f(x_i),y_i)-\tilde{\mathcal{L}}(f^*(x_i),y_i))|] \\
    & = 4L_1^{\prime}\mathbb{E}[\sup_{\|f-f^*\|_2\le r}|\frac{1}{n}\sum_{i=1}^n\sigma_i(f(x_i)-f^*(x_i))|] \\
    & \le 4L_1^{\prime}rr_n
\end{align}
Reducing the first equality to~\ref{underpara_gen_lem1}, the last inequality dues to our choice of $r_n$ and the non-increasing of $r\rightarrow\frac{R_n(r;\mathcal{F^*})}{r}$. We complete by combining with the tail probability~\eqref{general_loss_tail_Z}.

The proof of~\ref{general_loss_aux_thm_1_cp} is similar to the proof of Theorem 14.20 in~\cite{wainwright2019highdimensional}. Define event $E_0=\{Z_n(r_n)\ge 9Lr_n^2\}$ and $E_1=\{\int_{\mathbf{B}^d}(\Tilde{\mathcal{L}}(f,f^*) - \Tilde{\mathcal{L}}(f^*,f^*))d\mu - (\Tilde{\mathcal{L}}_n(f,f^*)-\Tilde{\mathcal{L}}_n(f^*,f^*))|\ge 10L_1^{\prime}r_n\|f-f^*\|_2$ for some $f$ with $\|f-f^*\|_2\ge r_n$\}. Let $\mathcal{E}$ be the event set such that the inequality~\eqref{event_0} in~\ref{general_loss_aux_thm_1_cp} holds. Then $E_0^c(r_n)\cap E_1^c \subseteq \mathcal{E}$. Letting $u=Lr_n^2$ and using Lemma~\ref{gen_loss_Z_lemma} we can get $\mathbb{P}[E_0]\le c_1exp(-c_2nr_n^2)$. Using the "pilling" skill, we get that for all $r_n^2\ge \frac{c}{n}$ we have $\mathbb{P}[E_1]\le c_1exp(-c_2^{\prime}nr_n^2)$. Then we complete. For full details the readers can refer to~\cite{wainwright2019highdimensional} and the proof of Theorem 4 therein. 

To prove~\ref{general_loss_aux_thm_2_cp}, going through the proof of~\ref{general_loss_aux_thm_1_cp}, we notice that either $\|\hat{f}-f^*\|_2\le r_n$ or 
\begin{align}
    |\int_{\mathbf{B}^d}(\Tilde{\mathcal{L}}(f,f^*) - \Tilde{\mathcal{L}}(f^*,f^*))d\mu - (\Tilde{\mathcal{L}}_n(f,f^*)-\Tilde{\mathcal{L}}_n(f^*,f^*))|\le 10L_1^{\prime}r_n\|f-f^*\|_2
\end{align}
If the former is true, then we are done. Otherwise, 
\begin{align}
    |\int_{\mathbf{B}^d}(\Tilde{\mathcal{L}}(f,f^*) - \Tilde{\mathcal{L}}(f^*,f^*))d\mu| - |(\Tilde{\mathcal{L}}_n(f,f^*)-\Tilde{\mathcal{L}}_n(f^*,f^*))| 
    & \le |\int_{\mathbf{B}^d}(\Tilde{\mathcal{L}}(f,f^*) - \Tilde{\mathcal{L}}(f^*,f^*))d\mu - (\Tilde{\mathcal{L}}_n(f,f^*)-\Tilde{\mathcal{L}}_n(f^*,f^*))| \\
    & \le 10L_1^{\prime}r_n\|f-f^*\|_2
\end{align}
We temporarily use symbol $E$ to represent the empirical error bound $\Tilde{\mathcal{L}}_n(f,f^*)-\Tilde{\mathcal{L}}_n(f^*,f^*)$~\ref{general_loss_underpara_emp}, then 
\begin{align}
 \int_{\mathbf{B}^d}(\Tilde{\mathcal{L}}(f,f^*) - \Tilde{\mathcal{L}}(f^*,f^*))d\mu  \le  10L_1^{\prime}r_n\|\hat{f} - f^*\|_2 + E
\end{align}
Using the $\gamma$-strongly convexity we have
\begin{align}
    \frac{\gamma}{2}\|\hat{f}-f^*\|_2^2 & \le \int_{\mathbf{B}^d}(\Tilde{\mathcal{L}}(f,f^*) - \Tilde{\mathcal{L}}(f^*,f^*))d\mu \\
    & \le (10L_1^{\prime})r_n\|\hat{f}-f^*\|_2 + E 
\end{align}
So, there exist constants $c_2,c_3$ (depending on $\gamma$ and $L$) such that 
\begin{align}\label{l2_estimation_gen_loss}
    \|\hat{f}-f^*\|_2 \le c_2r_n + \sqrt{c_3r_n^2+E}
\end{align}
As $E = c_0 + c_1r_n^2$ for some $c_0, c_1$, there are some constants, still denoting $c_2,c_3$, such that 
\begin{align}\label{l2_estimation_gen_loss_final}
    \|\hat{f}-f^*\|_2 \le c_2r_n + c_3 
\end{align}
The second half of~\ref{indispensable_result_cp} follows immediately. Substituting it into~\ref{general_loss_aux_thm_1_cp} and we get 
\begin{align}
    |\int_{\mathbf{B}^d}(\Tilde{\mathcal{L}}(f,f^*) - \Tilde{\mathcal{L}}(f^*,f^*))d\mu| & \le ((c_2+1)r_n + 10L_1^{\prime}\sqrt{c_3r_n^2+E})r_n + E \\
    & \le 10L_1^{\prime}(c_2+1)r_n^2 + \sqrt{c_3}10L_1^{\prime}r_n^2 + 10L_1^{\prime}\sqrt{E}r_n + E \\
    & \le c_4r_n^2 + c_5r_n + E
\end{align}
for some constants $c_4,c_5$ (depending on $\gamma$ and $L$). Plugging the empirical loss bound~\ref{general_loss_underpara_emp} into the expression of $E$, we complete.

\end{proof}

\bibliographystyle{plainnat}
\vskip 0.2in
\bibliography{ref}

\end{document}